\documentclass{article}

\usepackage{microtype}
\usepackage{graphicx}
\usepackage{subcaption}
\usepackage{bbm}
\usepackage{booktabs} 
\usepackage[table]{xcolor}
\usepackage{makecell}

\usepackage{xspace}
\usepackage{adjustbox}
\usepackage{tikz}
\usetikzlibrary{arrows.meta,calc,positioning,shapes.geometric, decorations.pathmorphing, shapes, decorations.markings}
\usepackage{placeins}

\usepackage{hyperref}

\usepackage[preprint]{icml2026}

\usepackage{amsmath}
\usepackage{amssymb}
\usepackage{mathtools}
\usepackage{amsthm}
 \usepackage{enumitem} 

\newcommand{\name}{\textsc{MacroGuide}\xspace}

\usepackage[capitalize,noabbrev]{cleveref}

\theoremstyle{plain}
\newtheorem{theorem}{Theorem}[section]

\newtheorem{lemma}[theorem]{Lemma}

\theoremstyle{definition}

\theoremstyle{remark}

\usepackage[textsize=tiny]{todonotes}

\icmltitlerunning{\name: Topological Guidance for Macrocycle Generation}

\begin{document}

\twocolumn[
  \icmltitle{\name: Topological Guidance for Macrocycle Generation}



  \icmlsetsymbol{equal}{*}

  \begin{icmlauthorlist}
    \icmlauthor{Alicja Maksymiuk}{equal,ox,aithyra}
    \icmlauthor{Alexandre Duplessis}{equal,ens}
    \icmlauthor{Michael Bronstein}{ox,aithyra}
    \icmlauthor{Alexander Tong}{aithyra}
    \icmlauthor{Fernanda Duarte}{ox}
    \icmlauthor{{\.I}smail {\.I}lkan Ceylan}{tu,aithyra}
  \end{icmlauthorlist}

  \icmlaffiliation{ox}{University of Oxford}
  \icmlaffiliation{aithyra}{AITHYRA}
  \icmlaffiliation{ens}{ENS Ulm, PSL}
  \icmlaffiliation{tu}{TU Wien}
  
  \icmlcorrespondingauthor{Alicja Maksymiuk}{alicja.maksymiuk@cs.ox.ac.uk}

  \icmlkeywords{Machine Learning}

  \vskip 0.3in
]



\printAffiliationsAndNotice{\icmlEqualContribution}

\begin{abstract}
Macrocycles are ring-shaped molecules that offer a promising alternative to small-molecule drugs due to their enhanced selectivity and binding affinity against difficult targets.  Despite their chemical value, they remain underexplored in generative modeling, likely owing to their scarcity in public datasets and the challenges of enforcing topological constraints in standard deep generative models.
We introduce \textsc{MacroGuide}: Topological Guidance for Macrocycle Generation, a diffusion guidance mechanism that uses Persistent Homology to steer the sampling of pretrained molecular generative models toward the generation of macrocycles, in both unconditional and conditional (protein pocket) settings. 
At each denoising step, \textsc{MacroGuide} constructs a Vietoris-Rips complex from atomic positions and promotes ring formation by optimizing persistent homology features. 
Empirically, applying \textsc{MacroGuide} to pretrained diffusion models increases macrocycle generation rates from $1\%$ to $99\%$, while matching or exceeding state-of-the-art performance on key quality metrics such as chemical validity, diversity, and PoseBusters checks. 
\end{abstract}

\section{Introduction}
\label{sec:intro}

\newcommand{\namemath}{\text{\rmfamily\textsc{MacroGuide}}}
\newcommand{\fuzzyblob}[4]{%
  \begin{scope}[shift={(#1, #2)}, rotate=#3]
    \foreach \k/\op in {1.15/0.05, 1.00/0.12, 0.85/0.12, 0.70/0.12, 0.55/0.12, 0.40/0.12, 0.25/0.12}{
      \fill[#4, opacity=\op] (0,0) ellipse ({2.3*\k} and {1.3*\k});
    }
  \end{scope}
}

\colorlet{cyclicbg}{green!5}

\colorlet{greeny}{green!50!black}

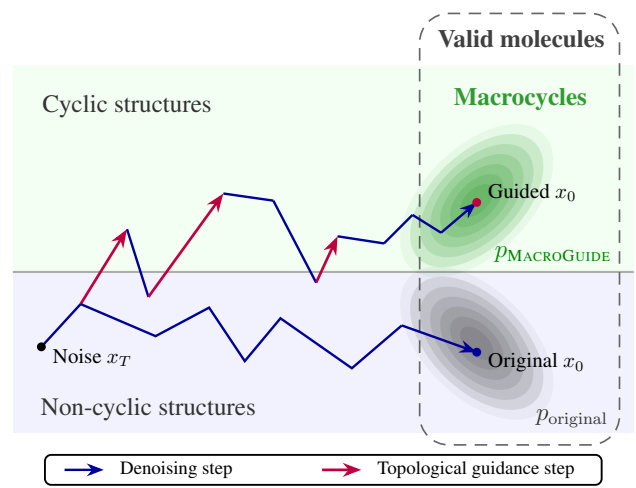
\begin{figure}[h!]

  \centering
  \begin{adjustbox}{max width=\columnwidth, margin=0pt}
  \begin{tikzpicture}[font=\sffamily,
  every node/.style={font=\sffamily},
      >=Stealth,
    set/.style={draw, fill opacity=0.2, thick},
    particle/.style={circle, fill=black, inner sep=1.5pt},
    grad_arrow/.style={->, thick},
    path_line/.style={draw, dashed, gray!60},
    x=0.60cm, y=0.60cm,
  ]

    \tikzset{
  traj/.style={line width=1.15pt, blue!60!black, },}
  \tikzset{
      guided/.style={line width=1.30pt, purple, ->},}
    \tikzset{
  legendbox/.style={
    draw=black,
    line width=0.9pt,
    fill=white,
    rounded corners=3pt,
    inner xsep=2pt,  
    inner ysep=6pt    
  }
}

    \def\xL{-2}
    \def\xR{15.4}

    \def\yTop{5.8}   
\def\yBot{-6.4}  
\def\yMid{0}

    \def\vmL{9.4}
\def\vmR{15.0}
\def\vmB{-4.55}
\def\vmT{7.25}
    \path[use as bounding box] (\xL,\yBot) rectangle (\xR,\vmT);

    \fill[cyclicbg]        (\xL,\yMid) rectangle (\xR,\yTop);
    \fill[blue!5!]   (\xL,\yMid) rectangle (\xR,-4.5);

    \draw[black!32, line width=1.pt] (\xL,\yMid) -- (\xR,\yMid);

    \node[black!80, font=\large]   at (1.2, 4.65) {Cyclic structures};

    \node[black!80, font=\large] at (1.8, -3.85) {Non-cyclic structures};

\draw[dashed, dash pattern=on 6pt off 4pt, line width=0.8pt,
      black!60, rounded corners=14pt]
  (\vmL,-4.85) rectangle (\vmR,\vmT);

\node[font=\bfseries\large, black!75, anchor=north west]
  at (\vmL+0.30,\vmT-0.25) {Valid molecules};

\node[font=\bfseries\large, greeny!80, anchor=north west]
at (\vmL+0.45+0.30,\vmT-2+0.1) {Macrocycles};

\fuzzyblob{11.0}{2.0}{45}{greeny}
\node[greeny, text opacity=1,
      rounded corners=1pt, inner sep=2pt,
      font=\large,
      anchor=east, ]
  at (14.90,.5) {$p_{\namemath}$};

\fuzzyblob{11.0}{-2.}{135}{black!80}
\node[black!80,
       text opacity=1, 
       font=\large,
      rounded corners=1pt, inner sep=2pt,
      anchor=east]
  at (14.80,-4.1) {\(p_{\mathrm{original}}\)};

\coordinate (start) at (-1.2, -2.1);
\coordinate (split) at (-0.1, -0.9);

\coordinate (target_top) at (11.0, 1.95);
\coordinate (target_bot) at (11.0, -2.25);

    \coordinate (t1)  at (1.2, 1.2);
    \coordinate (t2)  at (1.8, -0.7);
    \coordinate (t3)  at (3.9, 2.2);
    \coordinate (t4)  at (4.2, -0.5);
    \coordinate (t5)  at (5.5, 0.7);
    \coordinate (t6)  at (6.5, -0.3);
    \coordinate (t7)  at (7.1, 1.0);
    \coordinate (t8)  at (8.4, 0.8);
    \coordinate (t9)  at (9.2, 1.6);
    \coordinate (t10) at (10.0, 1.1);
    \coordinate (tnew) at (5.3, 2.);
    \coordinate (b1) at (2.0, -1.8);
    \coordinate (b2) at (3.5, -1.0);
    \coordinate (b3) at (4.5, -2.5);
    \coordinate (b4) at (5.5, -1.3);
    \coordinate (b5) at (7.5, -2.7);
    \coordinate (b6) at (8.9, -1.5);

    \draw[traj] (start) -- (split);

    \draw[guided] (split) -- (t1);
    \draw[traj]   (t1) -- (t2);
    \draw[guided] (t2) -- (t3);
    \draw[traj]   (t3) -- (tnew);
    \draw[traj]   (tnew) -- (t6);
    \draw[guided] (t6) -- (t7);
    \draw[traj]   (t7) -- (t8);
    \draw[traj]   (t8) -- (t9);
    \draw[traj]   (t9) -- (t10);
    \draw[traj, ->] (t10) -- (target_top);

    \draw[traj] (split) -- (b1) -- (b2) -- (b3) -- (b4) -- (b5) -- (b6);
    \draw[traj, ->] (b6) -- (target_bot);

particle/.style={circle, fill=black, inner sep=1.5pt},

\node[particle] at (start) {};
\node[
  anchor=west,
  font=\normalsize,
  xshift=2pt,
  yshift=-5pt
] at (start) {Noise $x_T$};
\node[particle, fill=blue!60!black] at (target_bot) {};
\node[
  anchor=west,
  font=\normalsize,
  xshift=2pt,
  yshift=-5pt
] at (target_bot) {Original $x_0$};
\node[particle, purple] at (target_top) {};
\def\legY{-5.55}
\node[
  anchor=west,
  font=\normalsize,
  xshift=2pt,
  yshift=+5pt
] at (target_top) {Guided $x_0$};

\def\legendSideGap{0.9}

\def\legendInnerPad{0.55}

\def\legendArrowLen{1.1}

\pgfmathsetmacro{\legL}{\xL + \legendSideGap}
\pgfmathsetmacro{\legR}{\xR - \legendSideGap}

\draw[black, line width=0.9pt, rounded corners=3pt, fill=white]
  (\legL, \legY-0.42) rectangle (\legR, \legY+0.42);

\draw[traj, ->] (\legL+\legendInnerPad, \legY) -- ++(\legendArrowLen,0);
\node[anchor=west, font=\small] at
  (\legL+\legendInnerPad+\legendArrowLen+0.25, \legY)
  {Denoising step};

\def\legendGap{0.25} 

\node[anchor=east, font=\small] (denlbl) at
  (\legR-\legendInnerPad, \legY)
  {Topological guidance step};

\coordinate (denArrowEnd)   at ($(denlbl.west)+(-\legendGap,0)$);
\coordinate (denArrowStart) at ($(denArrowEnd)+(-\legendArrowLen,0)$);
\draw[guided, ->] (denArrowStart) -- (denArrowEnd);

  \end{tikzpicture}
  \end{adjustbox}
  
  \caption{\textbf{Method overview.} \name drives the denoising trajectory towards macrocyclic structures using updates from a topological objective.}
  \label{fig:mainfig}
\end{figure}

Macrocycles -- cyclic molecules with a ring of 12 or more heavy atoms -- have attracted growing interest as drug candidates due to their improved \emph{selectivity} and \emph{binding affinity} against difficult targets 
\citep{garcia2023macrocycles, mallinson2012macrocycles, giordanetto2014macrocyclic}. 
The improvement in selectivity is largely attributed to the ring structure, which restricts molecular flexibility relative to linear analogs. This reduces off-target binding arising from conformational changes. Furthermore, macrocycles' larger size enables more protein-ligand interactions, allowing them to target challenging binding sites or even external protein surfaces \citep{yudin2015macrocycles}. The improvement in binding affinity comes from macrocycles' greater rigidity too, as it decreases their entropy in solvent and hence reduces entropic penalty upon binding. 

\textbf{Macrocycles in drug discovery.} \ \ 
Many macrocycles can fold into conformations that mask polar groups \citep{naylor2017cyclic}, which allows them to achieve oral bioavailability and to defy classical drug-likeness rules (e.g.  Rule of 5 by \citeauthor{lipinski1997experimental} \yrcite{lipinski1997experimental}). Macrocycles play a critical role in modern therapeutics with 
17 macrocyclic drugs approved by the FDA in the last five years alone \cite{Du25} (out of  
over 75 approved to date \cite{jiang2026macrocycle}). This includes clinical successes such as the immunosuppressant cyclosporine or the oncological drug lorlatinib (which outperforms its linear counterpart, crizotinib
\citep{ermert2017design,doi:10.1056/NEJMoa2027187}). 
Despite these developments and the great therapeutic promise of macrocycles, they remain underexplored in deep generative modeling. 

\textbf{Macrocycle generation.} \ 
Existing work on \mbox{macrocycle}\footnote{In this paper, \textit{macrocycles} refer to molecules with a cycle of at least~12 heavy atoms, although this term is sometimes used to describe a specific subset of such molecules, namely \textit{cyclic peptides}, which are built out of amino acids and peptide bonds.} generation has mostly focused on cyclic peptides (see \mbox{Appendix}~\ref{ap:cyclic_peptides} for a literature overview). These methods leverage peptide-specific traits, such as backbone homogeneity arising from a small set of amino-acid building blocks, which makes them inapplicable to structurally-diverse, arbitrary macrocycles. 

Non-peptidic macrocycle design has been explored in two prior works. Macformer, a SMILES-based transformer trained on linear-to-macrocyclic pairs, exploits the macrocyclization of pre-existing linear scaffolds to generate JAK2 inhibitors \citep{diao2023macrocyclization}. Macro-Hop performs macrocycle scaffold hopping via reinforcement learning and successfully produces PDGFR$\alpha$ inhibitors by generating structures that satisfy predefined constraints and 3D similarity to a reference molecule  \citep{liang2025designing}. However, both methods rely on handcrafted scaffolds or linear precursors, which severely limits their applicability in \textit{de novo} design, where neither the appropriate reference nor a linear equivalent is known \textit{a priori}.

\textbf{Problem statement and challenges.} \ \ 
Despite the rapid progress of diffusion-based models for small-molecule design \citep{wang2025diffusionmodelsmoleculessurvey, peng2023moldiff, hoogeboom2022equivariant, vignac2022digress}, to the best of our knowledge, \emph{there is currently no generative method explicitly designed to produce arbitrary macrocycles}. This gap can be largely attributed to two factors. First, public chemical datasets contain relatively few macrocycles, often much simpler than therapeutically-relevant ones \citep{garcia2023macrocycles}. Second, the existence of a large ring is a global topological property, whereas most generative models focus on approximating local chemical validity. As a result, existing unconstrained generative algorithms rarely produce macrocycles (see \cref{tab:all_models_unconditional}).

\textbf{Contributions.} \ \ 
We introduce Topological Guidance for Macrocycle Generation (\name), a diffusion guidance mechanism that steers pretrained molecular generative models -- originally designed for unconstrained molecule generation -- towards the generation of macrocycles, with or without conditioning (Figures~\ref{fig:mainfig} and \ref{fig:generated_macrocycles}). It works by computing gradients of the persistent homology features of a Vietoris-Rips complex at each step of the denoising process.

\begin{itemize}[topsep=0pt, partopsep=0pt, itemsep=0pt, parsep=0pt, leftmargin=*]
    \item Conceptually, \name is (i)~\textit{training-free}, generating macrocyclic topologies that may be rare or completely absent from the training data of the base model, without a need for retraining or finetuning; (ii)~\textit{lightweight}, introducing minimal computational overhead; (iii)~\textit{general}, as it can be plugged into various diffusion-based models; and (iv)~\textit{flexible}, allowing the user to specify both the number of rings and their sizes.
   \item Experimentally, we evaluate \name in both unconditional and conditional (protein pocket) settings; demonstrating a $99\%$ rate of macrocycle generation, compared to a baseline of $0\%$-$5\%$. Importantly, our method matches or exceeds state-of-the-art models across key quality metrics such as chemical validity, diversity, PoseBusters checks \citep{buttenschoen2024posebusters}, and pharmacophore satisfaction.
\end{itemize}

This work introduces \emph{the first \textit{de novo} arbitrary macrocycle generation method}, addressing a critical gap in deep generative modeling for drug discovery.

\begin{figure}[tb]
    \centering
\begin{minipage}{\linewidth}
    \centering
    \begin{minipage}{0.33\linewidth}
        \centering
        \includegraphics[width=\linewidth, trim=12pt 15pt 5pt 5pt, clip]{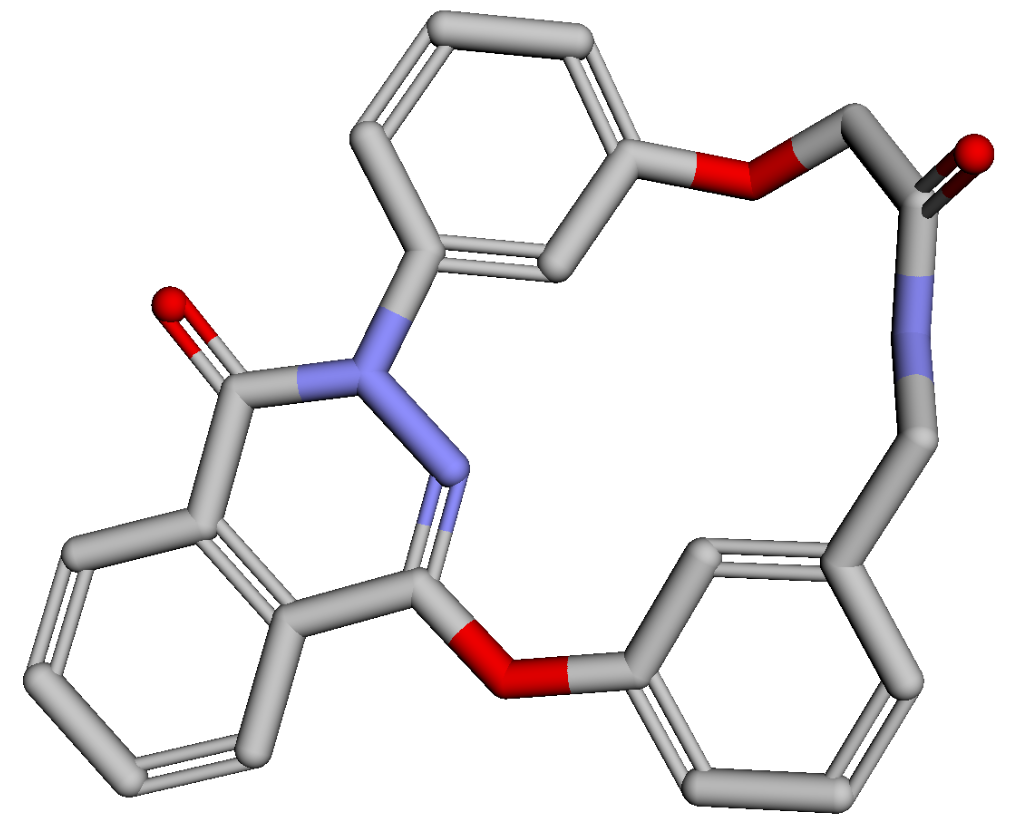}
    \end{minipage}\hfill
    \begin{minipage}{0.33\linewidth}
        \centering
        \includegraphics[width=\linewidth, trim=15pt 15pt 5pt 7pt, clip]{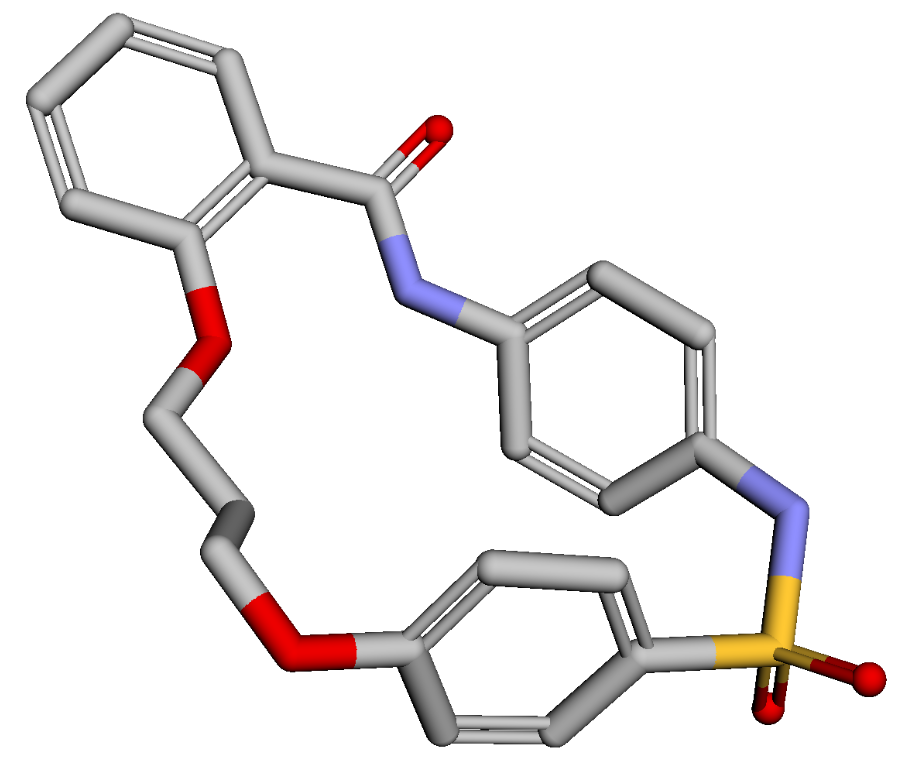}
    \end{minipage}\hfill
    \begin{minipage}{0.33\linewidth}
        \centering
        \includegraphics[width=\linewidth, trim=10pt -5pt 12pt 8pt, clip]{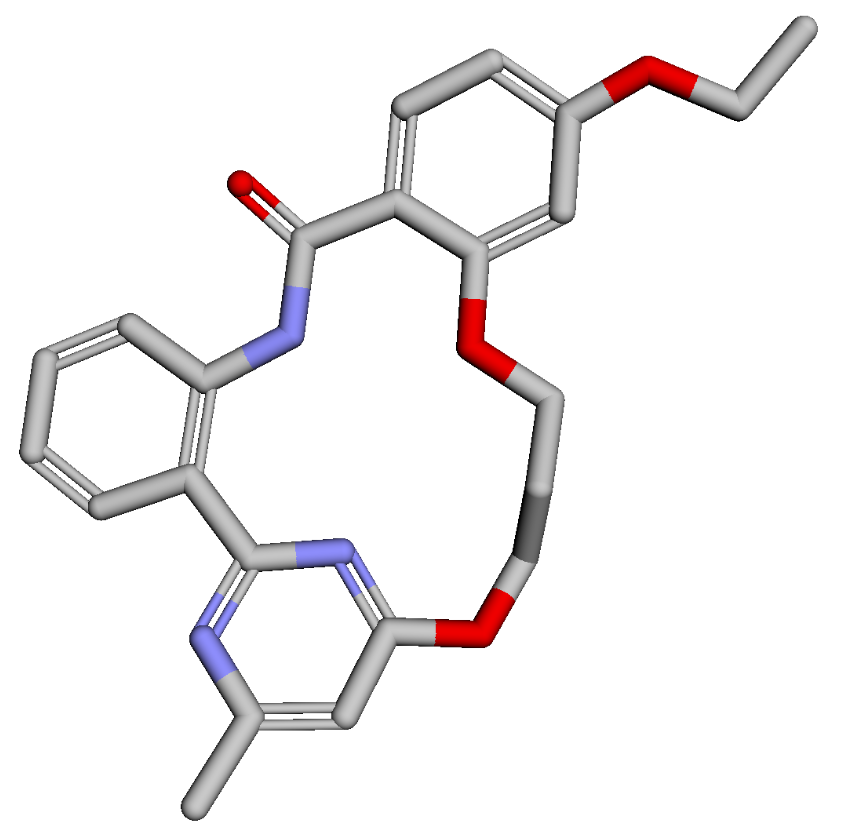}
    \end{minipage}
\end{minipage}

    \begin{minipage}{\linewidth}
        \centering
        \setlength{\fboxsep}{0pt}%
        \fbox{%
            \begin{minipage}{0.49\linewidth}%
                \centering
                \includegraphics[width=\linewidth, trim={19pt 0 0 0}, clip]{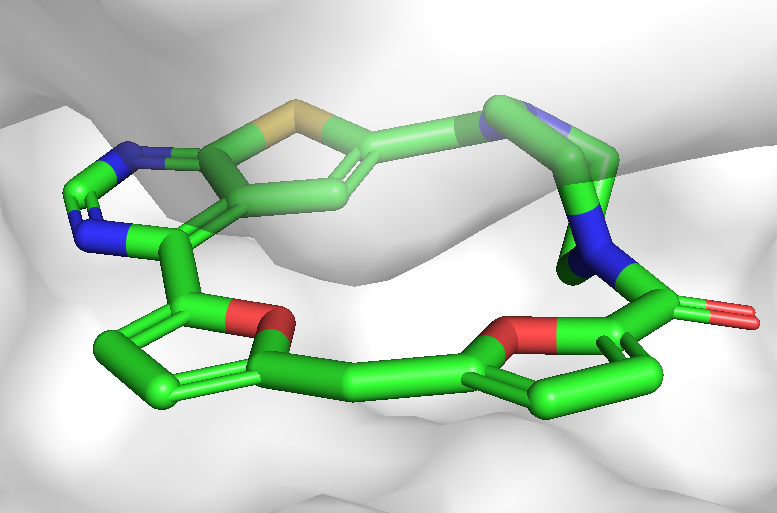}%
            \end{minipage}}
            \fbox{%
            \begin{minipage}{0.49\linewidth}%
                \centering
                \includegraphics[width=\linewidth, trim={0 0 0 15pt}, clip]{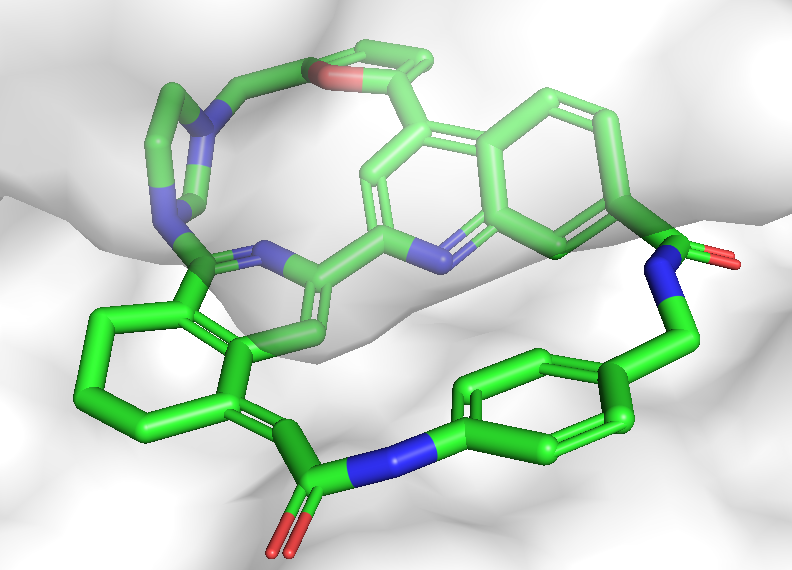}%
            \end{minipage}%
        }%
    \end{minipage}

    \caption{\textbf{Examples of generated macrocycles.} \textit{\underline{Top:}} Unconditional generation. \textit{\underline{Bottom:}} Protein conditioning. \textit{\underline{Bottom right:}} This molecule was specifically optimized to be bicyclic (two rings). Molecule fragments appear transparent when hidden by parts of the protein pocket.}
    \label{fig:generated_macrocycles}
\end{figure}

\section{Preliminaries}
\subsection{Diffusion-based Molecule Generation}
Diffusion models \citep{sohl2015deep,ho2020denoising,song2019generative} are a class of generative models that produce data by reversing a gradual noising process. In the forward process, a clean sample $x_0$ is progressively perturbed according to a Markov chain yielding noisy samples $x_t$ that converge to an isotropic Gaussian as $t$ increases. At the core of these models is the learning of the \emph{score function}
\begin{equation*}
    \mathbf{s}_\theta(x_t,t) \approx \nabla_{x_t} \log p_t(x_t),
\end{equation*}
which estimates the gradient of the log-density of the perturbed data. It is then applied iteratively, starting from pure noise, to generate samples during the reverse process, also known as denoising \citep{song2019generative,song2020score}. 

Diffusion models have been applied to tasks such as \textit{de novo} molecule generation \citep{hoogeboom2022equivariant, vignac2022digress, peng2023moldiff, schneuing2024structure, ziv2025molsnapper, schneuing2025multidomain}, as their iterative denoising formulation allows models to capture global structural dependencies while refining fine-grained details. In the molecular setting, diffusion can either operate directly in continuous 3D coordinate space of atoms \citep{hoogeboom2022equivariant}, on molecular graphs represented as adjacency matrices \citep{vignac2022digress}, or in a natural combination of both where the atom positions, atom types and bonds are included in the denoising network \citep{peng2023moldiff, vignac2023midi}. In contrast to sequential (autoregressive) models, diffusion-based approaches can more easily generate molecules in their energetic minima without the need for further redocking, as they model the full joint distribution of atoms.

\subsection{Persistent Homology for Molecular Point Clouds} Persistent homology is a tool from topological data analysis (TDA) that characterizes the topology of data across multiple scales \citep{edelsbrunner2002topological, edelsbrunner2010computational}.
In the context of a molecular 3D point cloud, one constructs a family of \emph{simplicial complexes} that encode proximity relationships between points, forming a nested sequence known as a \emph{filtration}.
A common choice is the \emph{Vietoris-Rips complex} at scale $\varepsilon$:
\[
\mathrm{VR}_\varepsilon(X) = \{\sigma \subseteq X \mid d(x_i, x_j) \le \varepsilon; \forall\, x_i,x_j \in \sigma \},
\]
where $d$ is typically the Euclidean distance.
As $\varepsilon$ increases, the topology of $\mathrm{VR}_\varepsilon(X)$ evolves: connected components (captured by $H_0$) merge, loops and cycles ($H_1$) form and fill in, and in higher dimensions ($H_2$, etc.) voids appear and vanish.
For each component $p_i^{(D)}$ of dimension $D$, persistent homology records the \emph{birth} $b_i^{(D)}$ and \emph{death} $d_i^{(D)}$ scales, producing a \emph{persistence diagram} or \emph{barcode}.
In the context of molecules, $H_0$ features correspond to clusters of atoms, $H_1$ components often correspond to chemical rings or other cycle-like structures, and $H_2$ can capture internal cavities, relevant for instance in molecular cages.
These descriptors are invariant to rigid transformations, robust to small perturbations and differentiable, making them useful for integrating topology information into various machine learning models (see Appendix \ref{ap:other_tda_papers} for a literature overview).

\section{Methodology}
\subsection{Method Overview}\label{sec:methodoverview}
For each diffusion time step $t$, we define the state of a molecule as $M_t=\{X_t,A_t,B_t\}$, where $X_t \in \mathbb{R}^{N\times3}$ describes the positions of the $N$ atoms, $A_t\in [0,1]^{N\times a}$ is the probability distribution on their atomic types (with $a$ the number of atom types), and $B_t\in [0,1]^{N\times N \times b}$ is the distribution on the bond types (if available), with $b$ the number of allowed bond types, including the absence of a bond.

We update the score provided by the denoising architecture $s_\theta (X_t, t)$ with our guidance function $\mathcal{F}_{\mathrm{TDA}}$
\begin{equation}
\label{eq:diffusion}
    \tilde s_\theta (X_t, t) = s_\theta (X_t, t) - \lambda_t \nabla_{X_t}\mathcal{F}_{\mathrm{TDA}}(X_t)
\end{equation} where $\lambda_t$ allows for scheduling of the guidance. 

The parameters of the original architecture remain fixed and the network is \emph{not} retrained. The guidance term (\Cref{fig:molecule_guidance}) is based on the computation of a Vietoris-Rips complex, and is defined as
\begin{equation}
\label{eq:tda_loss_new}
\mathcal{F}_{\mathrm{TDA}}(X)
=
F^{H_1}_{\mathrm{death}}(X)
+
F^{H_1}_{\mathrm{birth}}(X)
+
F^{H_0}_{\mathrm{death}}(X)
\end{equation}

\begin{itemize}[topsep=0pt, partopsep=0pt, itemsep=3pt, parsep=0pt, leftmargin=*]
    \item \textbf{Cycle size -- $H_1$ death.} $F^{H_1}_{\mathrm{death}}$ encourages the size of the largest cycle to be in a given range. Specifically, we optimize $d_{i^*}^{(1)}$ (where $i^\star = \arg\max_i d_i^{(1)}$), the death time of the $H_1$ component (ring) that dies last, to lie in a target interval $[d_{\min}, d_{\max}]$:
\begin{flalign}
\label{eq:h1_death}
F^{H_1}_{\mathrm{death}}(X)
=
 &\Big(\mathrm{ReLU}\big(d_{\min}- d^{(1)}_{i^\star}(X)\big)\Big)^2 + \notag\\
 &\Big(\mathrm{ReLU} \big( d^{(1)}_{i^\star}(X)-d_{\max}\big)\Big)^2
\end{flalign}

This promotes the formation of a large cycle and enables control over its size. The link between the number of atoms in the cycle and the death time of its associated $H_1$ component is detailed in \Cref{sec:macrosizecontrol}.
\item \textbf{Cycle connectivity -- $H_1$ birth.} $F^{H_1}_{\mathrm{birth}}$ acts as a proxy for cycle connectivity, by ensuring each edge is not longer than the maximum allowed bond length $\ell^\star$. In practice
\begin{equation}
\label{eq:h1_birth}
    F^{H_1}_{\mathrm{birth}}(X) = \mathrm{ReLU} \big(b^{(1)}_{i^\star}(X) - \ell^\star\big)
\end{equation}
constrains the largest edge size in the cycle $b_{i^\star}$.
Since bond information is not explicitly modeled, \mbox{$F^{H_1}_{\mathrm{birth}}$} does not formally guarantee cycle connectivity, but we find it sufficient in practice. The absence of a squared penalty for this term, in contrast to the other guidance terms, empirically leads to improved performance (see Appendix~\ref{app:ablations_square}).

\item \textbf{Molecule connectivity -- $H_0$ death.} $F^{H_0}_{\mathrm{death}}$ promotes the existence of one single connected component by making sure all the distances between any two adjacent atoms (i.e. the death $d^{(0)}_{j}$ of an $H_0$ component) are of length less than the maximum allowed bond length $\ell^\star$:
\begin{equation}
\label{eq:h0}
    F^{H_0}_{\mathrm{death}}(X) = \sum\limits_{j=1}^{N_0} \Big(\mathrm{ReLU} \big(d^{(0)}_{j}(X) - \ell^\star\big)\Big)^2
\end{equation}
where $N_0$ is the number of finite-death $H_0$ components.
\end{itemize}

Detailed algorithms can be found in Appendix~\ref{app:algo}.
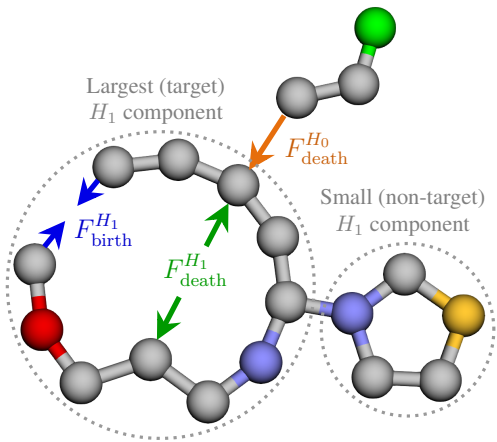
\begin{figure}[t]
  \centering
  \begin{tikzpicture}[scale=0.85, transform shape]
    \node[inner sep=0] (img) {%
      {\includegraphics[height=0.8\linewidth]{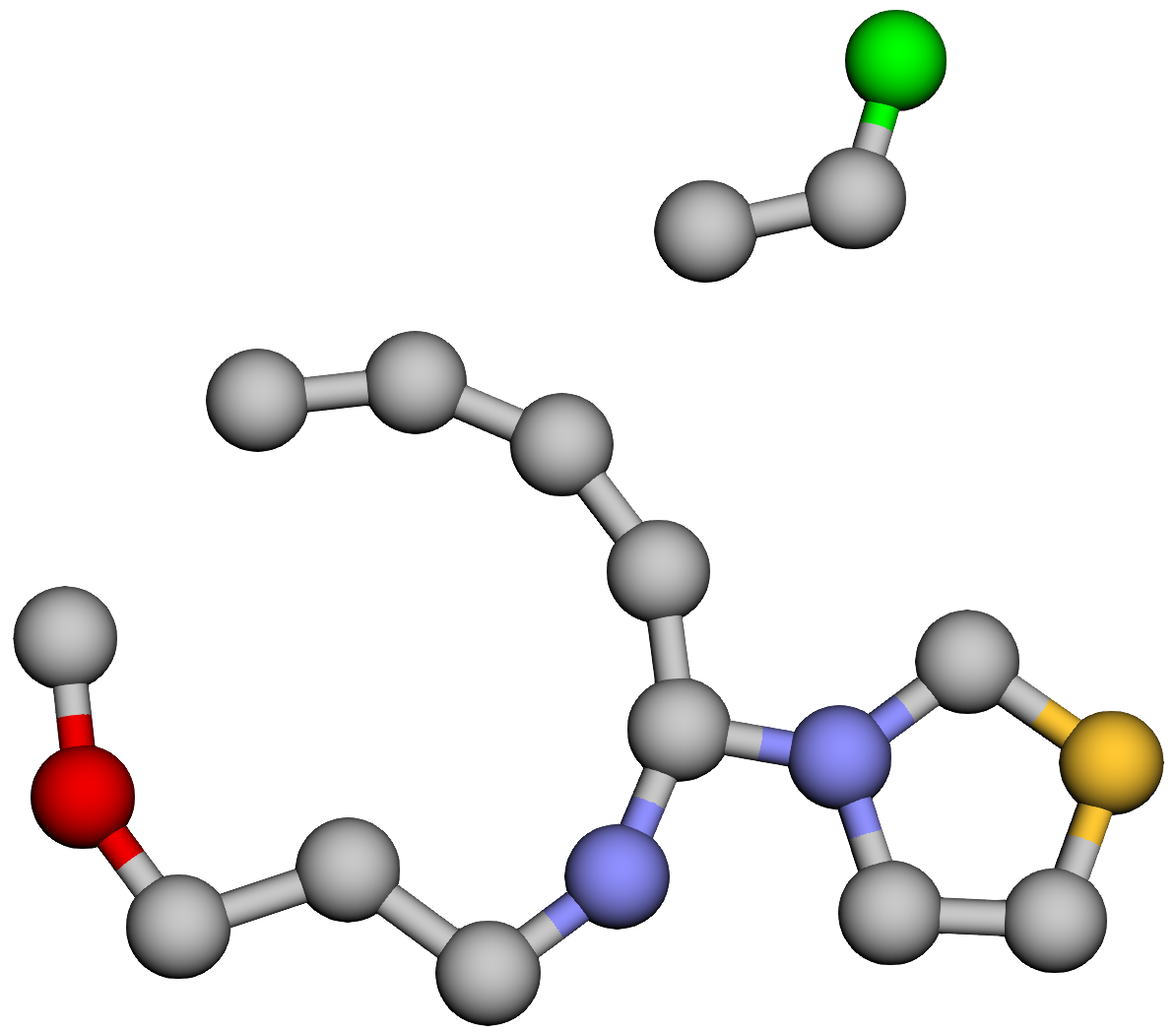}}%
    };

    \node[font=\sffamily\large, text=green!60!black, anchor=north west, xshift=2.3cm, yshift=-3.8cm]
      at (img.north west) {$F^{H_1}_{\mathrm{death}}$};

    \node[font=\sffamily\large, text={rgb,255:red,200; green,95; blue,10}
, anchor=north west, xshift=4.2cm, yshift=-1.9cm]
      at (img.north west) {$F^{H_0}_{\mathrm{death}}$};

\node[
    font=\sffamily\large,
    text=blue!80!black,
    anchor=south east,
    align=center,
    xshift=-5.45cm,
    yshift=2.7cm
] at (img.south east) {$F^{H_1}_{\mathrm{birth}}$};

        \node[
  text=black!50,
  align=center,
  anchor=south east,
  xshift=-4.02cm,
  yshift=4.6cm
] at (img.south east) {Largest (target)\\$H_1$ component};
  
  \node[
  text=black!50,
  align=center,
  anchor=south east,
  xshift=0.05cm,
  yshift=2.85cm
] at (img.south east) {Small (non-target)\\$H_1$ component};

\coordinate (A) at (1.0,0.0);
\coordinate (B) at (2.0,0.0);
\coordinate (C) at (3.0,0.0);
\coordinate (D) at (4.0,0.0);
    \draw[->, draw={rgb,255:red,230; green,110; blue,12},
    line width=2pt,
    -{Stealth[length=10pt, width=10pt]}]
        (0.55,1.55) -- (0,0.7);

  \draw[
    <-,
    blue!90!black,
    line width=2pt,
    -{Stealth[length=10pt, width=10pt]}
]
(-2.36,0.54) -- (-2.686,0.137);

\draw[
    <-,blue!90!black,
    line width=2pt,
    -{Stealth[length=10pt, width=10pt]}
]
(-3.2,-0.5) -- (-2.874,-0.097);
 
\draw[
    ->,
    green!60!black,
    line width=2pt,
    -{Stealth[length=10pt, width=10pt]}
]
(-1.090,-1.289) -- (-1.44,-1.96);

\draw[
    <-,
    green!60!black,
    line width=2pt,
    -{Stealth[length=10pt, width=10pt]}
]
(-0.670,-0.491) -- (-0.32,0.18);

    \draw[
        black!40,
        very thick,
        dotted
    ]
    (-1.38,-1.1) circle (2.4cm);

        \draw[
        black!40,
        very thick,
        dotted
    ]
    (2.48,-1.8) circle (1.35cm);

  \end{tikzpicture}

  \caption{\textbf{Topological guidance for diffusion-based macrocycle generation.}}
  \label{fig:molecule_guidance}
\end{figure}

\textbf{Addressing gradient sparsity.} \ \ 
Each topological feature is created and destroyed by specific simplices. As a consequence, the gradient signal comes only from the few points involved in critical birth and death events \citep{poulenard2018topological}. Although this is generally not problematic, as the critical points typically change across denoising steps \citep{carrière2021optimizingpersistenthomologybased}, it can render the  optimization of $H_0$ components unstable and potentially divergent. We address this issue by masking the gradient of the atom closest to the molecular centroid. Further details and theoretical guarantees are provided in Appendix \ref{ap:h0}. 

\subsection{Posterior Sampling}\label{sec:posteriorsampling}
Following \citet{guo2024gradientguidancediffusionmodels}, the score
of a conditional distribution admits the decomposition
\begin{equation*}
\label{eq:cond_score}
\nabla_x \log p_t(x\mid y)
=
\nabla_x \log p_t(x)
+
\nabla_x \log p_t(y\mid x).
\end{equation*}
In our setting, the guidance term can thus be interpreted as a conditional score term
\begin{equation}
\label{eq:tda_cond_score}
\nabla_x \log p_t(y\mid x)
=
-\lambda_t\,\nabla_x \mathcal{F}_{\mathrm{TDA}}(x),
\qquad \lambda>0,
\end{equation}
which corresponds to an energy-based conditional likelihood
\begin{equation*}
\label{eq:tda_likelihood}
p(y\mid x)\;\propto\;\exp\!\big(-\lambda_t\,\mathcal{F}_{\mathrm{TDA}}(x)\big).
\end{equation*}

In particular Bayes' Theorem gives
\begin{equation*}
\label{eq:tda_tilt}
p_t(x\mid y)\;\propto\;p_t(x)\exp\!\big(-\lambda_t\,\mathcal{F}_{\mathrm{TDA}}(x)\big),
\end{equation*}
and guidance can be interpreted as approximating sampling from this conditional distribution in the continuous-time limit with an exact score.

In practice, however, this interpretation only holds approximately. First, the learned score is imperfect, so the dynamics do not exactly track the target SDE. \citet{gao2025regrectifiedgradientguidance} analyze this discrepancy and propose rectified guidance formulations to reduce bias from score errors. Second, practical samplers discretize the SDE. \citet{guo2024gradientguidancediffusionmodels} show that discretization introduces systematic deviations in the effective distribution, and propose variance-preserving corrections. Nevertheless, we find this approximation sufficient for our task, as shown in Section \ref{sec:results}. 

\subsection{Macrocycle Size Control}\label{sec:macrosizecontrol}

\paragraph{Combinatorial and geometric sizes.}
The \emph{size} of a macrocycle is defined here as the number $n$ of heavy atoms it contains. However, we only have direct control over a geometric notion of size, which is the \emph{death} of the corresponding $H_1$ component in the Vietoris-Rips filtration. Therefore we need to establish an explicit relationship between these two definitions.

A convenient and analytically tractable way to do so is to consider an idealized but informative model: an equilateral polygon with $n$ vertices embedded in the
plane (representing atoms), and edges of length $\ell$ (representing bonds). However, regularity is not a realistic assumption as tetrahedral carbon arrangements impose smaller bond angles than the ones of a regular polygon ($\theta=109.5^\circ$, which can be observed in \Cref{fig:generated_macrocycles}). We thus study crown-shaped equilateral $n$-gons with interior angles $\theta$ (see \Cref{fig:crown_proof} of Appendix~\ref{ap:h0} for a visualization). The following theorem provides a simple relationship $n = f(d,\ell)$ in this framework (Proof in Appendix~\ref{app:proof}).

\begin{theorem}[Vietoris-Rips death time of a tetrahedral cycle] \label{thm:rips_death_tetra} Consider a regular cyclic conformation of n atoms ($n$ even) with bond length $\ell$ and bond angle $\theta$. The death time d of the dominant $H_1$ component in the Vietoris-Rips filtration is given by: \begin{equation} d = \ell \sqrt{2(1-\cos\theta)} \cdot \frac{\sin\left( \lceil \frac{n}{6} \rceil \frac{2\pi}{n} \right)}{\sin\left( \frac{2\pi}{n} \right)} \label{eq:closed_form}\end{equation} In the limit of large n, \begin{equation}\label{eq:linearapprox} n = \frac{4\pi}{\sqrt{6(1-\cos\theta)}}\frac{d}{\ell} + \mathcal{O}(1) \end{equation} \end{theorem}

\textbf{Accuracy of the linear approximation.} \ \ 
Solving Equation~\eqref{eq:closed_form} for $n$ does not admit an easy closed-form solution, making the linear approximation of Equation \ref{eq:linearapprox} very attractive from a practical point of view, assuming it is tight enough. Consider a macrocycle of size $n=14$ (crown shaped) with a typical bond length $\ell = 1.5$ \AA, and $\theta=109.5^\circ$. The linearization provided incurs a relative error of approximately $3\%$. This confirms that for drug-like macrocycles, the higher-order terms are negligible in practice. We further validate Equation \ref{eq:linearapprox} empirically in \Cref{sec:validationformula}.

\section{Experiments}
\label{sec:results}

\subsection{Main Results}

\begin{table*}[t]
\begin{center}
\begin{small}
\caption{\textbf{Performance of unconditional macrocycle generation.} Results obtained from 1000 molecules with 30 heavy atoms.}
\begin{tabular}{lccccccc}
\toprule
    Metrics ($\uparrow$; [0-1])  & \makecell[c]{MolDiff\\(no guid.)} &  \makecell[c]{MolDiff\\(finetuned)}&\makecell[c]{Naive\\(12 atoms)}& \makecell[c]{Naive\\(14 atoms)}& \makecell[c]{Naive\\(16 atoms)}&\makecell[c]{Torus noise\\initialization}& \makecell[c]{MolDiff+\name\\(ours)}\\ \hline
    Validity     & 0.991  &  \textbf{0.994}&0.937& 0.930& 0.918&0.993&0.989\\ 
    Connectivity  & 0.958  &  0.993&0.985& 0.990& 0.987&0.960&\textbf{0.999}\\
    Successfulness  & 0.949  &  0.987&0.923& 0.921& 0.906&0.953&\textbf{0.988}\\
\hline
    Out of successful:&    &  && & &&\\
        \rowcolor[gray]{.9}\hspace{1em}\textbf{Macrocycles}& 0.053  &  0.851&0.370& 0.364& 0.369&0.054&\textbf{0.997}\\

\bottomrule
    
\end{tabular}
\label{tab:all_models_unconditional}
\end{small}
\end{center}
\vskip -0.1in
\end{table*}

\begin{table*}[t]
\begin{center}
\begin{small}
\caption{\textbf{Performance of macrocycle generation with protein conditioning.}}

\begin{tabular}{lcccccc}
\toprule
    Metrics ($\uparrow$; [0-1])  & \makecell[c]{MolSnapper\\(no guid.)}&  \makecell[c]{MolSnapper\\(finetuned)}&\makecell[c]{Naive\\(12 atoms)}& \makecell[c]{Naive\\(14 atoms)}& \makecell[c]{Naive\\(16 atoms)}& \makecell[c]{MolSnapper+\name\\(ours)}\\ \hline
    Validity     & 0.858 &  0.795 &0.819& 0.791& 0.763&\textbf{0.925}\\ 
    Connectivity  & 1.000 &  1.000 &1.000 & 1.000 & 1.000&1.000\\
    Successfulness  & 0.858 &  0.795 &0.819& 0.791& 0.763&\textbf{0.925}\\
\hline
    Out of successful:&    &  && & &\\
        \rowcolor[gray]{.9}\hspace{1em}\textbf{Macrocycles}& 0.003 &  0.180 &0.013&0.000 & 0.003&\textbf{0.995}\\

\bottomrule
    
\end{tabular}
\label{tab:all_models_conditional}
\end{small}
\end{center}
\vskip -0.1in
\end{table*}

\subsubsection{Unconditional Generation} 
\label{sec:unconditional}

To demonstrate the ability of \name to generate macrocycles we apply it to MolDiff, a diffusion model that generates molecules by denoising a 3D point cloud of atoms, together with bonds \citep{peng2023moldiff}. MolDiff was pretrained on GEOM-Drug \citep{axelrod2022geom}, a dataset of small molecules with drug-like properties, where only 0.14\% of the training data are macrocycles.
In our guidance method, the simplicial complex is built using Pytorch Topological \citep{rieck2022torch_topological}. Further setup details are described in Appendix~\ref{app:exp}.

\textbf{Baselines.} \ \ 
As \name is the first arbitrary macrocycle generation method, we cannot rely on established baselines. Instead, we design the following custom baselines. First, we evaluate whether macrocycles can be generated by \textbf{finetuning} existing models on macrocyclic datasets (Appendix~\ref{app:finetuning}). Second, we construct a \textbf{naive} guidance mechanism that forces a chosen subset of atoms to form a ring, where adjacent atoms are kept close and opposite atoms are pushed apart (Appendix~\ref{app:naive}). We check the performance for subset sizes from 12 to 16 atoms. \mbox{Finally,} we test whether initializing the denoising process with \mbox{\textbf{torus-shaped noise}} resembling a macrocycle improves generation (Appendix~\ref{app:torus}).

\textbf{Metrics.} \ \ 
For consistency with previous approaches \citep{peng2023moldiff}, we reuse the following metric definitions: molecules are considered \textbf{valid} if they can be parsed by RDKit \citep{landrum2013rdkit}, \textbf{connected} if they have only one graph connected component, and \textbf{successful} if they meet both properties at the same time.
If the generated molecule is successful, we test if the molecule is a \textbf{macrocycle} by checking for the presence of a chordless cycle of size at least~12 using \texttt{Chem.GetSymmSSSR} from RDKit. This only includes cycles without edges between non-neighboring vertices, which excludes fused small rings where the outer boundary can have $\geq 12$ atoms.

\vspace{1cm}

\textbf{Parameter choice.} \ \ 
During the experiments, we set $\lambda_t = 1$,  $[d_{\min}, d_{\max}] = [4.45\ \text{\AA}, 5.05\ \text{\AA}]\,$, and $\ell^\star = 2\,\text{\AA}$, which is chosen based on chemical knowledge and is slightly higher than longest common bonds in organic chemistry.

\textbf{Results.} \ \ 
Table~\ref{tab:all_models_unconditional} demonstrates that \name increases the number of generated macrocycles by almost 20-fold compared to MolDiff, and outperforms all the other baselines. Finetuning does yield a respectable generation rate, although generating about 15\% fewer macrocycles than \name. Visualizations of generated molecules are provided in Appendix~\ref{app:vis_unc}.

\subsubsection{Protein Conditioning}
\label{sec:conditional}
\name can also be applied to molecular generation models conditioned on proteins which we demonstrate by guiding MolSnapper \citep{ziv2025molsnapper}. MolSnapper requires specifying the protein pocket placement and choosing a few atoms from a reference ligand to act as a pharmacophore. We follow the setup provided by the original work: we use the same protein pocket and choose reference atoms to guide generation within the protein. We reuse the same parameters as in Section~\ref{sec:unconditional}. Further setup details are described in Appendix~\ref{app:protein}.

\textbf{Baselines.} \ \ 
For comparison, we include the same baselines as in Section~\ref{sec:unconditional}, albeit without the torus noise initialization because of its poor performance in the unconditional setting and because of the non-triviality of choosing the right placement and angle of the initial macrocycle-like noise. 

\textbf{Results.} \ \ 
We report performance in Table~\ref{tab:all_models_conditional}. Similarly to the unconditional setting, the macrocycle generation rate increases from 0.3\% to nearly 100\%, representing a 300-fold increase. Importantly, \name is the only efficient method, as all of the baselines exhibit a dramatic performance drop. In particular, the performance of the finetuned model drops to 18\%. This collapse can be attributed to the protein constraints pushing the finetuned model outside of its training distribution, leaving no explicit mechanism to constrain the trajectory back towards macrocyclic \mbox{topologies}. Consequently, we do not consider finetuning to be a robust solution for this task and do not assess it further. Visualizations of generated molecules are provided in Appendix~\ref{app:vis_cond}.

\subsubsection{Assessing Macrocycle Quality}

To further assess the quality of the macrocycles generated with \name we evaluate the following additional criteria: \mbox{\textbf{novelty}}, indicating whether a molecule is absent from the training dataset of the original model; \textbf{uniqueness}, measuring whether it differs from all the other generated samples; and \textbf{diversity}, assessed via fingerprint similarity across all pairs of generated molecules \citep{peng2023moldiff}. Finally, to assess the chemical quality of the generated macrocycles we run \textbf{PoseBusters metrics} \cite{buttenschoen2024posebusters} and report the fraction of molecules that pass all tests, as well as the performance for individual checks.

\textbf{Conditional metrics.} \ \ 
For conditional generation we also calculate PoseBusters metrics related to the protein. Additionally, \textbf{pharmacophore satisfaction} is evaluated by matching each reference atom to the closest atom in the generated molecule. A match is valid if the atom types agree and the distance is within 1~\AA \  \citep{ziv2025molsnapper}. A~pharmacophore is considered satisfied if at least 80\% of the reference atoms are matched. Finally, we assess whether the generated molecules are likely to be orally available. Although there are no clearly established \textbf{Lipinski rules for macrocycles}, we follow \citet{garcia2023macrocycles} and \citet{viarengo2021defining} to propose the following rules: Molecular Weight (MW) $<$~1000 Da, number of Hydrogen Bond Donors (HBD) $\leq~$7 and lipophilicity between 2.4 and 6, and we report the fraction of macrocycles meeting all of the criteria. 

\begin{table}[t]
\begin{center}
\begin{small}
\caption{\textbf{Unconditional macrocycle quality metrics.} Values were obtained by generating molecules until at least 1000 macrocycles were found to ensure reliable and comparable estimates. }

\begin{tabular}{lcc}
\toprule
Metrics ($\uparrow$; [0-1])  & \makecell[c]{MolDiff\\(no guid.)}& \makecell[c]{+\name\\(ours)}\\ \hline
        Diversity&0.707&\textbf{0.771}\\
        Novelty& 1.000&1.000\\
        Uniqueness& 1.000&1.000\\
        \rowcolor{gray!20}\textbf{All PoseBusters tests}& 0.663&\textbf{0.805}\\ 
            \hspace{1em}Bond lengths    & 0.977&\textbf{0.990}\\ 
            \hspace{1em}Bond angles    & 0.949&\textbf{0.987}\\ 
            \hspace{1em}Internal steric clash    & 0.722&\textbf{0.844}\\ 
            \hspace{1em}Aromatic ring flatness& 0.980&\textbf{0.999}\\ 
            \hspace{1em}Non-ar. ring non-flatness& 0.981&\textbf{0.999}\\ 
            \hspace{1em}Double bond flatness& 0.965&\textbf{0.989}\\ 
            \hspace{1em}Internal energy    & 0.961&\textbf{0.984}\\ 
\bottomrule
    
\end{tabular}
\label{tab:results_unconditional_macro}
\end{small}
\end{center}
\vskip -0.1in
\end{table}

\begin{table}[t]
\begin{center}
\begin{small}
\caption{\textbf{Protein-conditioned macrocycle quality metrics.}}
\setlength{\tabcolsep}{2pt}
\begin{tabular}{l@{}cc}
\toprule
    Metrics ($\uparrow$; [0-1])  & \makecell[c]{MolSnapper\\(no guid.)}& \makecell[c]{+\name\\(ours)}\\ \hline

        Diversity&0.626&\textbf{0.712}\\
        Novelty& 1.000&1.000\\
        Uniqueness& 1.000&1.000\\
        \rowcolor{gray!20}\textbf{All PoseBusters tests}&  0.440&\textbf{0.575}\\
            \hspace{1.em}Ligand PoseBusters& 0.539&\textbf{0.626}\\
                \hspace{2em}Bond lengths    & 0.844&\textbf{0.860}\\ 
                \hspace{2em}Bond angles    & \textbf{0.913}&0.888\\ 
                \hspace{2em}Internal steric clash    & \textbf{0.862}&0.854\\ 
                \hspace{2em}Aromatic ring flatness& \textbf{0.996} &0.992\\ 
                \hspace{2em}Non-ar. ring non-flatness & 0.993&\textbf{0.999}\\ 
                \hspace{2em}Double bond flatness& 0.980&\textbf{0.993}\\ 
                \hspace{2em}Internal energy    & 0.818&\textbf{0.921}\\ 
            \hspace{1.em}Protein PoseBusters&  0.806&\textbf{0.911}\\
                \hspace{2em}Protein-ligand max. distance &  1.000&1.000\\
                \hspace{2em}Min. distance to protein&  0.806&\textbf{0.907}\\
                \hspace{2em}Volume overlap with protein&  1.000&1.000\\
        Pharmacophore satisfaction &  0.769&\textbf{0.789} \\
        Macrocycle Lipinski &  0.551&\textbf{0.638} \\
\bottomrule

\end{tabular}
\label{tab:results_protein_macro}
\end{small}
\end{center}
\vskip -0.1in
\end{table}
\vspace{1cm}
\textbf{Results.} \ \ 
Table~\ref{tab:results_unconditional_macro} and Table~\ref{tab:results_protein_macro} summarize the quality metrics of the generated macrocycles, with standard deviations reported in Appendix~\ref{app:stds}. \name substantially outperforms the base model for most of the tested metrics, both in the unconditional and conditional settings.
Additionally, our protein-conditioned macrocycles exhibit an important improvement in the internal energy test, which checks pose deviation from the relaxed molecule conformation.
This suggests that \name addresses one of the prominent challenges in protein-conditioned molecular design: generated molecules are often not in their energetic minimum, resulting in low binding affinities and changes in pose after redocking \citep{harris2023benchmarking}. 
Finally, we observe an increase in the rate of pharmacophore satisfaction and compliance with Lipinski rules for macrocycles.

Furthermore, we examine atom type and ring size distributions of generated molecules, and show that \name preserves molecule complexity during macrocycle generation (Appendix~\ref{app:atoms} and \ref{app:rings}). 
We also show that \name macrocycles tend to have fewer small strained rings, whose presence often poses~a~synthetic~\mbox{challenge}.

\textbf{Discussion.} \ \ 
Notably, \name does not explicitly guide towards better 3D structure or molecular properties. The observed improvement in generative metrics possibly stems from the fact that \name fixes the global shape of the molecule at early stages of the denoising process, which allows the denoiser to spend more time refining the local structure of the molecule (see \Cref{fig:featurestracking}). In contrast, without guidance the macrocycles represent a deviation from the training distribution, and likely assume cyclical shapes only later during the generation process. 

\subsection{Additional \name Applications}
\subsubsection{Bicyclic Molecule Generation}
Macrocycles can achieve better selectivity than small molecules, but as the ring grows, so does the conformational space of macrocycles, potentially allowing for off-target interactions. One way to generate larger binders without reducing selectivity is to generate bicyclic compounds \cite{heinis2009phage, rowland2025rational}. To generate bicycles, we modify the topological guidance to optimize the two largest $H_1$ components, and we keep the same desired death time range. In Table~\ref{tab:bicycles} we show 97\% and 90\% bicyclic performance rates in the unconditional and conditional setups respectively. An example of a generated bicyclic molecule is shown in Figure~\ref{fig:generated_macrocycles}.
\begin{table}[ht]
\begin{center}
\begin{small}
\setlength{\tabcolsep}{5.8pt}

\caption{\textbf{Performance of bicyclic molecule generation in unconditional and conditional settings.} A molecule is considered bicyclic if it has at least two chordless cycles of at least 12 atoms each. Results obtained from 1000 molecules with 45 heavy atoms.}

\begin{tabular}{lcc}
\toprule
    Metrics ($\uparrow$; [0-1])  & \makecell[c]{MolDiff\\+\name}& \makecell[c]{MolSnapper\\+\name}\\ \hline
    Validity     & 0.998&0.821\\ 
    Connectivity  & 0.931&1.000\\
    Successfulness  & 0.929&0.821\\
\hline
    Out of successful:&   &\\
        \rowcolor[gray]{.9}\hspace{1em}\textbf{Bicyclic molecules}& 0.970&0.895\\
 \bottomrule

\end{tabular}
\label{tab:bicycles}
\end{small}
\end{center}
\vskip -0.1in
\end{table}

\subsubsection{$H_0$ Death Guidance Improves Large Molecule Generation}
While originally designed as a stability safeguard for our guidance method, ablation studies reveal that $H_0$ death optimization improves large molecule generation beyond the macrocycle task. Notably, existing methods struggle with the generation of large molecular structures \citep{le2023navigatingdesignspaceequivariant} and often produce several disconnected components. 

We show that it is possible to preserve molecule connectivity as the number of atoms increases, by applying our $H_0$ death optimization (\Cref{eq:h0} with the gradient masking detailed in Appendix~\ref{ap:h0}). The results in \Cref{fig:large_molecules} show a drastic improvement in connectivity ($80\%$ versus $20\%$ for $n=110$). $H_0$ optimization therefore appears as an attractive sampling-time guidance method when dealing with large molecules.

\begin{figure}[htb]
    \centering
    \includegraphics[width=0.9\linewidth]{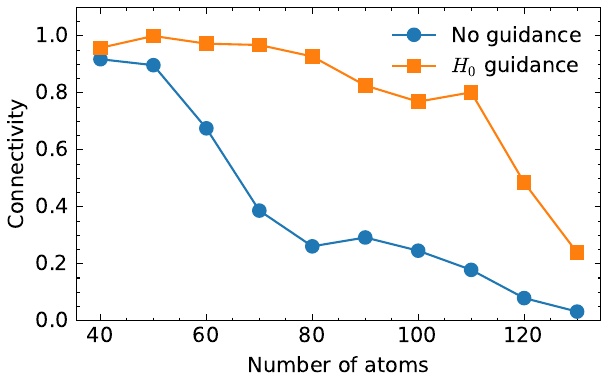}
    \caption{\textbf{Performance of MolDiff with increasing molecular size.} Adding $H_0$ guidance term improves performance for large molecule sizes. Results obtained from $200$ samples each.}
    \label{fig:large_molecules}
\end{figure}

\subsection{Further Analysis}

\subsubsection{Empirical Macrocycle Size Control}\label{sec:validationformula}
We empirically evaluate the results of \Cref{sec:macrosizecontrol} on macrocycle size control. \Cref{fig:size} confirms the validity of the provided control equation, although we observe a slight underestimation of the actual values by the theoretical predictions. One possible explanation is that elongated elliptic rings have a lower death time than circular rings with the same number of atoms (see \Cref{fig:ellipse} of Appendix~\ref{ap:h0} for a visualization).

\begin{figure}[htb]
    \centering
    \includegraphics[width=1\linewidth]{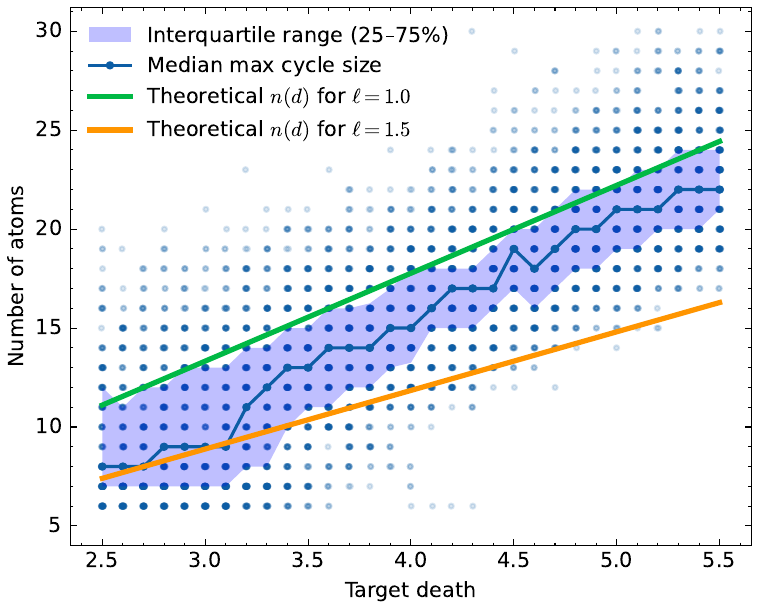}
    \caption{\textbf{Median max cycle size as a function of the target death.} The empirical results are compared to the theoretical formulas for $\ell=1.0$ and $\ell=1.5$, the minimum and maximum typical bond lengths, respectively. Results are computed for $200$ samples of 30 heavy atoms, with each target size $d^\star$ being constrained in the relaxed form of an interval $[d^\star-0.05, d^\star+0.05]$, sampled at 0.1 intervals.}
    \label{fig:size}
\end{figure}

\subsubsection{Runtime Comparison}
\label{sec:complexity}
Importantly, \name introduces only a very moderate computational overhead on top of the base denoising process.
\Cref{fig:complexity} provides a breakdown of the average time per denoising step for MolSnapper ($\approx 17 ms$) and \name ($\approx 21ms$).
The extra cost of our guidance method is dominated by the computation of the Vietoris-Rips complex, which grows quadratically with $n$. This is reasonable since the denoising process is quadratic too and \name represents only $20\%$ of the total cost. Nevertheless, we explored two different strategies to accelerate our guidance method. Both methods rely on selecting only a subset of the denoising steps to apply the guidance function.

\textbf{Guidance every $k$ steps.} \ \ 
Our first strategy consists of applying the topological guidance only once every $k$ denoising steps, instead of at every step. Results for $k=2,3,4,5$ are reported in Appendix~\ref{app:subsampling}. Importantly, we are able to divide the additional cost of \name by at least $5$ with less than $15\%$ degradation in macrocycle generation rates and molecular quality metrics relative to full guidance. Notably, it also improves the rate of pharmacophore satisfaction, compared both to MolSnapper and \name with $k=1$. This trade-off makes sparse guidance an attractive default choice in practice.

\textbf{Late-start guidance.} \ \ 
A second strategy we investigated consists of activating the topological guidance only during the later stages of the diffusion process. In practice, however, this approach consistently underperformed the previous strategy, which can be explained by the fact that noising processes tend to corrupt high-frequency information first \citep{falck2025fourierspaceperspectivediffusion}, which means the shape of a molecule is likely decided early during denoising. For completeness, we report the corresponding results in Appendix~\ref{app:starting_late}.

Consequently, applying guidance every $k$ steps appears to be a good practical choice when trying to sample a handful of molecules, or when using high values of $n$.

\begin{figure}
    \centering
    \includegraphics[width=1\linewidth]{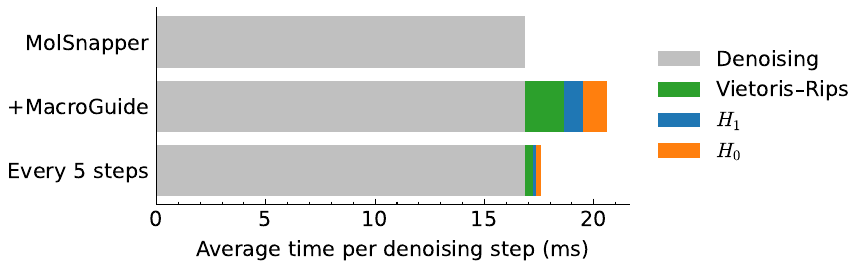}
    \caption{\textbf{Runtime comparison.} \name produces only a small computational overhead that can be further reduced by applying guidance only every $k$ steps.}
    \label{fig:complexity}
\end{figure}

\subsubsection{Stability Analysis}
The considerable sparsity of gradients may cast doubt on the stability of the method. That is why we provide insights about the guidance process by tracking optimized topological features and their gradients throughout denoising. Full results and discussion can be found in Appendix~\ref{app:tracking}. In particular \Cref{fig:featurestracking} demonstrates the stability of \name, after as few as $100$ denoising steps. More importantly, it provides an explanation to the success of our method: guidance shapes macrocyclic structure early, while late-stage denoising (last $200$ steps) performs fine-grained chemical validity optimization without interference.

\begin{figure}[t]
    \centering
    \includegraphics[width=1\linewidth]{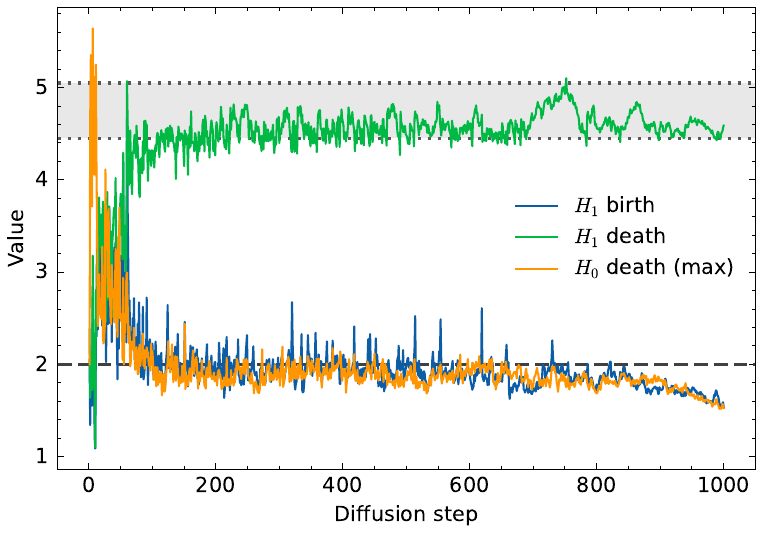}
    \caption{\textbf{Key topological features throughout denoising}. Maximum $H_0$ death (i.e. the longest distance to the closest neighbor over all atoms), and birth and death of the maximum-death $H_1$ component (largest cycle).}
    \label{fig:featurestracking}
\end{figure}

\subsection{Ablations} 

We further analyze the contribution of the different components of our guidance mechanism through a series of ablations. In Appendix~\ref{app:ablations}, we remove individual terms of the guidance function and show that all components are necessary in the unconditional setting. In the conditional setting, however, the $H_0$ death term (molecule connectivity) can be omitted, as the confined volume of the protein pocket naturally prevents atom disconnection. In Appendix~\ref{app:ablations_square}, we investigate how adding or removing the square of the topological terms influences performance. Finally, we study the sensitivity to guidance strength motivating the chosen
value (Appendix~\ref{app:strength}).

\section{Conclusion}

In this work, we presented \name, the first diffusion-based framework designed to generate arbitrary macrocycles without the need for model retraining or specialized datasets. By leveraging persistent homology to guide pretrained models, our method successfully bridges the gap between local chemical validity and global topological constraints. Specifically, our method drives the denoising path to a stable region of the set of cyclic structures, where validity optimization can be performed during later steps without interference from the guidance terms.

Experimentally, \name achieves a 99\% macrocycle generation rate in both unconditional and protein-conditioned settings. Furthermore, the resulting molecules demonstrate superior 3D structural quality, exceeding or matching state-of-the-art performance on chemical validity, diversity, and PoseBusters benchmarks.

By providing a lightweight and training-free mechanism, \name offers a robust solution for exploring the macrocyclic chemical space. We provide a detailed discussion of promising future directions in Appendix~\ref{sec:discussion}. We anticipate that this framework will integrate naturally with the expanding class of 3D point-cloud-based models, allowing topological guidance to contribute to continued improvements in generative drug discovery.

\section*{Acknowledgements}
The authors would like to thank Aleksy Kwiatkowski for insightful discussions.
AM is partially funded by AITHYRA and the UKRI Engineering and Physical Sciences Research Council (EPSRC) with grant code EP/S024093/1. This research is partially supported by the EPSRC Turing AI World-Leading Research Fellowship No. EP/X040062/1 and EPSRC AI Hub No. EP/Y028872/1.

\section*{Impact Statement}

This work aims to improve controllable molecular generative modeling for drug discovery. By enabling more structured exploration of chemical space, such methods have the potential to accelerate early-stage drug design, reduce the cost of candidate discovery, and support the development of novel therapeutics, including for diseases with limited or no existing treatments and for traditionally challenging or undruggable targets. Improved generative efficiency may also help lower barriers to early-stage molecular design, enabling broader access to computational drug discovery tools across research communities.

At the same time, as with any general-purpose molecular generation approach, there are potential risks. Techniques that facilitate the design of therapeutic molecules could also be misused to generate harmful or toxic compounds. This work focuses on methodological development rather than deployment, and practical applications require appropriate domain expertise, experimental validation, and regulatory oversight. Generative models should therefore be viewed as decision-support tools rather than autonomous systems for drug design.

Overall, we believe the potential benefits of these methods for drug discovery outweigh the associated risks when applied responsibly within established scientific and ethical frameworks.

\bibliography{example_paper}

@article{Du25,
  title={{FDA}-approved drugs featuring macrocycles or medium-sized rings},
  author={Du, Youlong and Semghouli, Anas and Wang, Qian and Mei, Haibo and Kiss, Lor{\'a}nd and Baecker, Daniel and Soloshonok, Vadim A and Han, Jianlin},
  journal={Archiv der Pharmazie},
  volume={358},
  number={1},
  pages={e2400890},
  year={2025},
  publisher={Wiley Online Library}
}

@inproceedings{
parktopology,
title={Topology-aware Graph Diffusion Model with Persistent Homology},
author={Joonhyuk Park and Donghyun Lee and Yujee Song and Guorong Wu and Won Hwa Kim},
booktitle={The Thirty-ninth Annual Conference on Neural Information Processing Systems},
year={2025},
}

@inproceedings{gupta2024topodiffusionnet,
 author = {Gupta, Saumya and Samaras, Dimitris and Chen, Chao},
 booktitle = {International Conference on Learning Representations},
 pages = {31699--31713},
 title = {Topo{D}iffusion{N}et: A Topology-aware Diffusion Model},
 volume = {2025},
 year = {2025}
}

@article{garcia2023macrocycles,
  title={Macrocycles in drug discovery - learning from the past for the future},
  author={Garcia Jimenez, Diego and Poongavanam, Vasanthanathan and Kihlberg, Jan},
  journal={Journal of Medicinal Chemistry},
  volume={66},
  number={8},
  pages={5377--5396},
  year={2023},
  publisher={ACS Publications}
}

@article{mallinson2012macrocycles,
  title={Macrocycles in new drug discovery},
  author={Mallinson, Jamie and Collins, Ian},
  journal={Future Medicinal Chemistry},
  volume={4},
  number={11},
  pages={1409--1438},
  year={2012},
  publisher={Taylor \& Francis}
}

@article{rettie2025accurate,
  title={Accurate de novo design of high-affinity protein-binding macrocycles using deep learning},
  author={Rettie, Stephen A and Juergens, David and Adebomi, Victor and Bueso, Yensi Flores and Zhao, Qinqin and Leveille, Alexandria N and Liu, Andi and Bera, Asim K and Wilms, Joana A and {\"U}ffing, Alina and others},
  journal={Nature Chemical Biology},
  pages={1--9},
  year={2025},
  publisher={Nature Publishing Group US New York}
}

@article{watson2023novo,
  title={De novo design of protein structure and function with {RF}diffusion},
  author={Watson, Joseph L and Juergens, David and Bennett, Nathaniel R and Trippe, Brian L and Yim, Jason and Eisenach, Helen E and Ahern, Woody and Borst, Andrew J and Ragotte, Robert J and Milles, Lukas F and others},
  journal={Nature},
  volume={620},
  number={7976},
  pages={1089--1100},
  year={2023},
  publisher={Nature Publishing Group UK London}
}

@article{grambow2023accurate,
  title={Accurate and Efficient Structural Ensemble Generation of Macrocyclic Peptides using Internal Coordinate Diffusion},
  author={Grambow, Colin A and Weir, Hayley and Diamant, Nathaniel L and Scalia, Gabriele and Biancalani, Tommaso and Chuang, Kangway V},
  journal={arXiv preprint arXiv:2305.19800},
  year={2023}
}

@inproceedings{schiff2022augmenting,
  title={Augmenting molecular deep generative models with topological data analysis representations},
  author={Schiff, Yair and Chenthamarakshan, Vijil and Hoffman, Samuel C and Ramamurthy, Karthikeyan Natesan and Das, Payel},
  booktitle={ICASSP 2022-2022 IEEE International Conference on Acoustics, Speech and Signal Processing (ICASSP)},
  pages={3783--3787},
  year={2022},
  organization={IEEE}
}

@inproceedings{vignac2022digress,
  title={Di{G}ress: Discrete denoising diffusion for graph generation},
  author={Clement Vignac and Igor Krawczuk and Antoine Siraudin and Bohan Wang and Volkan Cevher and Pascal Frossard},
  booktitle={The Eleventh International Conference on Learning Representations },
    year={2023}
}

@article{rettie2025cyclic,
  title={Cyclic peptide structure prediction and design using {A}lpha{F}old2},
  author={Rettie, Stephen A and Campbell, Katelyn V and Bera, Asim K and Kang, Alex and Kozlov, Simon and Bueso, Yensi Flores and De La Cruz, Joshmyn and Ahlrichs, Maggie and Cheng, Suna and Gerben, Stacey R and others},
  journal={Nature Communications},
  volume={16},
  number={1},
  pages={4730},
  year={2025},
  publisher={Nature Publishing Group UK London}
}

@article{jumper2021highly,
  title={Highly accurate protein structure prediction with {A}lpha{F}old},
  author={Jumper, John and Evans, Richard and Pritzel, Alexander and Green, Tim and Figurnov, Michael and Ronneberger, Olaf and Tunyasuvunakool, Kathryn and Bates, Russ and {\v{Z}}{\'\i}dek, Augustin and Potapenko, Anna and others},
  journal={nature},
  volume={596},
  number={7873},
  pages={583--589},
  year={2021},
  publisher={Nature Publishing Group UK London}
}

@InProceedings{hoogeboom2022equivariant,
  title = 	 {Equivariant Diffusion for Molecule Generation in 3{D}},
  author =       {Hoogeboom, Emiel and Satorras, V\'{\i}ctor Garcia and Vignac, Cl{\'e}ment and Welling, Max},
  booktitle = 	 {Proceedings of the 39th International Conference on Machine Learning},
  pages = 	 {8867--8887},
  year = 	 {2022},
  volume = 	 {162},
  publisher =    {PMLR},
}

@article{edelsbrunner2002topological,
  title={Topological persistence and simplification},
  
  author={Edelsbrunner, Herbert and Letscher, David and Zomorodian, Afra},
  journal={Discrete \& computational geometry},
  volume={28},
  pages={511--533},
  year={2002},
  publisher={Springer}
}

@article{huang2025hog,
  title={H{OG}-{D}iff: Higher-Order Guided Diffusion for Graph Generation},
  author={Huang, Yiming and Birdal, Tolga},
  journal={arXiv preprint arXiv:2502.04308},
  year={2025}
}

@inproceedings{
song2020score,
title={Score-Based Generative Modeling through Stochastic Differential Equations},
author={Yang Song and Jascha Sohl-Dickstein and Diederik P Kingma and Abhishek Kumar and Stefano Ermon and Ben Poole},
booktitle={International Conference on Learning Representations},
year={2021},

}

@InProceedings{peng2023moldiff,
  title = 	 {{M}ol{D}iff: Addressing the Atom-Bond Inconsistency Problem in 3{D} Molecule Diffusion Generation},
  author =       {Peng, Xingang and Guan, Jiaqi and Liu, Qiang and Ma, Jianzhu},
  booktitle = 	 {Proceedings of the 40th International Conference on Machine Learning},
  pages = 	 {27611--27629},
  year = 	 {2023},
  volume = 	 {202},
  publisher =    {PMLR},

}

@article{landrum2013rdkit,
  title={{RDK}it documentation},
  author={Landrum, Greg},
  journal={Release},
  volume={1},
  number={1-79},
  pages={4},
  year={2013}
}

@InProceedings{gabrielsson2020topology,
  title = 	 {A Topology Layer for Machine Learning},
  author =       {Gabrielsson, Rickard Br\"uel and Nelson, Bradley J. and Dwaraknath, Anjan and Skraba, Primoz},
  booktitle = 	 {Proceedings of the Twenty Third International Conference on Artificial Intelligence and Statistics},
  pages = 	 {1553--1563},
  year = 	 {2020},
  volume = 	 {108},
  publisher =    {PMLR},

}

@article{hu2024topology,
  title={Topology-aware latent diffusion for 3{D} shape generation},
  author={Hu, Jiangbei and Fei, Ben and Xu, Baixin and Hou, Fei and Yang, Weidong and Wang, Shengfa and Lei, Na and Qian, Chen and He, Ying},
  journal={arXiv preprint arXiv:2401.17603},
  year={2024}
}

@article{townsend2020representation,
  title={Representation of molecular structures with persistent homology for machine learning applications in chemistry},
  author={Townsend, Jacob and Micucci, Cassie Putman and Hymel, John H and Maroulas, Vasileios and Vogiatzis, Konstantinos D},
  journal={Nature communications},
  volume={11},
  number={1},
  pages={3230},
  year={2020},
  publisher={Nature Publishing Group UK London}
}

@article{giordanetto2014macrocyclic,
  title={Macrocyclic drugs and clinical candidates: what can medicinal chemists learn from their properties?},
  author={Giordanetto, Fabrizio and Kihlberg, Jan},
  journal={Journal of medicinal chemistry},
  volume={57},
  number={2},
  pages={278--295},
  year={2014},
  publisher={ACS Publications}
}

@article{diao2023macrocyclization,
  title={Macrocyclization of linear molecules by deep learning to facilitate macrocyclic drug candidates discovery},
  author={Diao, Yanyan and Liu, Dandan and Ge, Huan and Zhang, Rongrong and Jiang, Kexin and Bao, Runhui and Zhu, Xiaoqian and Bi, Hongjie and Liao, Wenjie and Chen, Ziqi and others},
  journal={Nature Communications},
  volume={14},
  number={1},
  pages={4552},
  year={2023},
  publisher={Nature Publishing Group UK London}
}

@article{liang2025designing,
  title={Designing Macrocyclic Kinase Inhibitors Using Macrocycle Scaffold Hopping with Reinforced Learning ({M}acro-{H}op)},
  author={Liang, Hong and Huang, Shengjie and Xu, Xinxin and Yin, Zhao and Hussain, Muzammal and Song, Xiaojuan and Yi, Jianqiao and He, Yingqi and Guo, Jing and Tu, Zhengchao and others},
  journal={Journal of Medicinal Chemistry},
  volume={68},
  number={6},
  pages={6698--6717},
  year={2025},
  publisher={ACS Publications}
}

@article{wang2025diffusionmodelsmoleculessurvey,
  publtype={informal},
  author={Liang Wang and Chao Song and Zhiyuan Liu and Yu Rong and Qiang Liu and Shu Wu and Liang Wang},
  title={Diffusion Models for Molecules: A Survey of Methods and Tasks},
  year={2025},
  cdate={1738368000000},
  journal={CoRR},
  volume={abs/2502.09511}
}

@inproceedings{carriere2020perslay,
  title={Pers{L}ay: A neural network layer for persistence diagrams and new graph topological signatures},
  author={Carri{\`e}re, Mathieu and Chazal, Fr{\'e}d{\'e}ric and Ike, Yuichi and Lacombe, Th{\'e}o and Royer, Martin and Umeda, Yuhei},
  booktitle={International Conference on Artificial Intelligence and Statistics},
  pages={2786--2796},
  year={2020},
  organization={PMLR}
}

@inproceedings{charlier2019phom,
  title={P{H}om-{G}e{M}: Persistent homology for generative models},
  author={Charlier, Jeremy and State, Radu and Hilger, Jean},
  booktitle={2019 6th Swiss Conference on Data Science (SDS)},
  pages={87--92},
  year={2019},
  organization={IEEE}
}

@inproceedings{
demir2023multiparameter,
title={Multiparameter Persistent Homology for Molecular Property Prediction},
author={Andac Demir and Bulent Kiziltan},
booktitle={ICLR 2023 - Machine Learning for Drug Discovery workshop},
year={2023},
}

@article{wang2025topologyguidancecontrollingoutput,
  title={Topology Guidance: Controlling the Outputs of Generative Models via Vector Field Topology},
  author={Wang, Xiaohan and Berger, Matthew},
  journal={arXiv preprint arXiv:2505.06804},
  year={2025}
}

@inproceedings{sohl2015deep,
  title={Deep Unsupervised Learning using Nonequilibrium Thermodynamics},
  author={Sohl-Dickstein, Jascha and Weiss, Eric and Maheswaranathan, Niru and Ganguli, Surya},
  booktitle={International Conference on Machine Learning},
  pages={2256--2265},
  year={2015},
  organization={PMLR}
}

@article{ho2020denoising,
  title={Denoising Diffusion Probabilistic Models},
  author={Ho, Jonathan and Jain, Ajay and Abbeel, Pieter},
  journal={Advances in Neural Information Processing Systems},
  volume={33},
  pages={6840--6851},
  year={2020}
}

@article{song2019generative,
  title={Generative Modeling by Estimating Gradients of the Data Distribution},
  author={Song, Yang and Ermon, Stefano},
  journal={Advances in Neural Information Processing Systems},
  volume={32},
  year={2019}
}

@inproceedings{
vignac2023midi,
title={Mi{D}i: Mixed Graph and 3{D} Denoising Diffusion for Molecule Generation},
author={Clement Vignac and Nagham Osman and Laura Toni and Pascal Frossard},
booktitle={ICLR 2023 - Machine Learning for Drug Discovery workshop},
year={2023}
}

@book{edelsbrunner2010computational,
  title={Computational topology: an introduction},
  author={Edelsbrunner, Herbert and Harer, John},
  year={2010},
  publisher={American Mathematical Soc.}
}

@misc{rieck2022torch_topological,
  title        = {torch\_topological: A Topological Machine Learning Framework for PyTorch},
  author       = {Bastian Rieck},
  note         = {Commit from December 1, 2022},
  year         = {2022}
}

@inproceedings{
jiang2025zeroshot,
title={Zero-Shot Cyclic Peptide Design via Composable Geometric Constraints},
author={Dapeng Jiang and Xiangzhe Kong and Jiaqi Han and Mingyu Li and Rui Jiao and Wenbing Huang and Stefano Ermon and Jianzhu Ma and Yang Liu},
booktitle={Forty-second International Conference on Machine Learning},
year={2025}
}

@inproceedings{
zhou2025designing,
title={Designing Cyclic Peptides via Harmonic {SDE} with Atom-Bond Modeling},
author={Xiangxin Zhou and Mingyu Li and Yi Xiao and Jiahan Li and Dongyu Xue and Zaixiang Zheng and Jianzhu Ma and Quanquan Gu},
booktitle={Forty-second International Conference on Machine Learning},
year={2025}
}

@article{yudin2015macrocycles,
  title={Macrocycles: lessons from the distant past, recent developments, and future directions},
  author={Yudin, Andrei K},
  journal={Chemical Science},
  volume={6},
  number={1},
  pages={30--49},
  year={2015},
  publisher={Royal Society of Chemistry}
}

@article{naylor2017cyclic,
  title={Cyclic peptide natural products chart the frontier of oral bioavailability in the pursuit of undruggable targets},
  author={Naylor, Matthew R and Bockus, Andrew T and Blanco, Maria-Jesus and Lokey, R Scott},
  journal={Current Opinion in Chemical Biology},
  volume={38},
  pages={141--147},
  year={2017},
  publisher={Elsevier}
}

@article{ermert2017design,
  title={Design, properties and recent application of macrocycles in medicinal chemistry},
  author={Ermert, Philipp},
  journal={Chimia},
  volume={71},
  number={10},
  pages={678--678},
  year={2017}
}

@inproceedings{poulenard2018topological,
  title={Topological function optimization for continuous shape matching},
  author={Poulenard, Adrien and Skraba, Primoz and Ovsjanikov, Maks},
  booktitle={Computer Graphics Forum},
  volume={37},
  pages={13--25},
  year={2018},
  organization={Wiley Online Library}
}

@inproceedings{carrière2021optimizingpersistenthomologybased,
  title={Optimizing persistent homology based functions},
  author={Carri{\`e}re, Mathieu and Chazal, Fr{\'e}d{\'e}ric and Glisse, Marc and Ike, Yuichi and Kannan, Hariprasad and Umeda, Yuhei},
  booktitle={International conference on machine learning},
  pages={1294--1303},
  year={2021},
  organization={PMLR}
}

@article{axelrod2022geom,
  title={{GEOM}, energy-annotated molecular conformations for property prediction and molecular generation},
  author={Axelrod, Simon and Gomez-Bombarelli, Rafael},
  journal={Scientific Data},
  volume={9},
  number={1},
  pages={185},
  year={2022},
  publisher={Nature Publishing Group UK London}
}

@article{ziv2025molsnapper,
  title={Mol{S}napper: conditioning diffusion for structure-based drug design},
  author={Ziv, Yael and Imrie, Fergus and Marsden, Brian and Deane, Charlotte M},
  journal={Journal of Chemical Information and Modeling},
  volume={65},
  number={9},
  pages={4263--4273},
  year={2025},
  publisher={ACS Publications}
}

@inproceedings{
gao2025regrectifiedgradientguidance,
title={{REG}: Rectified Gradient Guidance for Conditional Diffusion Models},
author={Zhengqi Gao and Kaiwen Zha and Tianyuan Zhang and Zihui Xue and Duane S Boning},
booktitle={Forty-second International Conference on Machine Learning},
year={2025}
}

@inproceedings{
guo2024gradientguidancediffusionmodels,
title={Gradient Guidance for Diffusion Models: An Optimization Perspective},
author={Yingqing Guo and Hui Yuan and Yukang Yang and Minshuo Chen and Mengdi Wang},
booktitle={The Thirty-eighth Annual Conference on Neural Information Processing Systems},
year={2024}
}

@article{Adamaszek_2017,
  title={The Vietoris--Rips complexes of a circle},
  author={Adamaszek, Micha{\l} and Adams, Henry},
  journal={Pacific Journal of Mathematics},
  volume={290},
  number={1},
  pages={1--40},
  year={2017},
  publisher={Mathematical Sciences Publishers}
}

@inproceedings{
harris2023benchmarking,
title={PoseCheck: Generative Models for 3D Structure-based Drug Design Produce Unrealistic Poses},
author={Charles Harris and Kieran Didi and Arian Jamasb and Chaitanya Joshi and Simon Mathis and Pietro Lio and Tom Blundell},
booktitle={NeurIPS 2023 Generative AI and Biology (GenBio) Workshop},
year={2023}
}

@article{doi:10.1056/NEJMoa2027187,
  title={First-line lorlatinib or crizotinib in advanced {ALK}-positive lung cancer},
  author={Shaw, Alice T and Bauer, Todd M and de Marinis, Filippo and Felip, Enriqueta and Goto, Yasushi and Liu, Geoffrey and Mazieres, Julien and Kim, Dong-Wan and Mok, Tony and Polli, Anna and others},
  journal={New England Journal of Medicine},
  volume={383},
  number={21},
  pages={2018--2029},
  year={2020},
  publisher={Mass Medical Soc}
}

@article{viarengo2021defining,
  title={Defining and navigating macrocycle chemical space},
  author={Viarengo-Baker, Lauren A and Brown, Lauren E and Rzepiela, Anna A and Whitty, Adrian},
  journal={Chemical science},
  volume={12},
  number={12},
  pages={4309--4328},
  year={2021},
  publisher={Royal Society of Chemistry}
}

@article{buttenschoen2024posebusters,
  title={PoseBusters: AI-based docking methods fail to generate physically valid poses or generalise to novel sequences},
  author={Buttenschoen, Martin and Morris, Garrett M and Deane, Charlotte M},
  journal={Chemical Science},
  volume={15},
  number={9},
  pages={3130--3139},
  year={2024},
  publisher={Royal Society of Chemistry}
}

@article{falck2025fourierspaceperspectivediffusion,
  title={A Fourier Space Perspective on Diffusion Models},
  author={Falck, Fabian and Pandeva, Teodora and Zahirnia, Kiarash and Lawrence, Rachel and Turner, Richard and Meeds, Edward and Zazo, Javier and Karmalkar, Sushrut},
  journal={arXiv preprint arXiv:2505.11278},
  year={2025}
}

@inproceedings{
le2023navigatingdesignspaceequivariant,
title={Navigating the Design Space of Equivariant Diffusion-Based Generative Models for De Novo 3D Molecule Generation},
author={Tuan Le and Julian Cremer and Frank Noe and Djork-Arn{\'e} Clevert and Kristof T Sch{\"u}tt},
booktitle={The Twelfth International Conference on Learning Representations},
year={2024}
}

@article{rowland2025rational,
  title={Rational design of cyclic peptides, with an emphasis on bicyclic peptides},
  author={Rowland, Catherine E and Bezerra, Gustavo Arruda and Skynner, Michael J},
  journal={Current Opinion in Structural Biology},
  volume={92},
  pages={103025},
  year={2025},
  publisher={Elsevier}
}

@article{heinis2009phage,
  title={Phage-encoded combinatorial chemical libraries based on bicyclic peptides},
  author={Heinis, Christian and Rutherford, Trevor and Freund, Stephan and Winter, Greg},
  journal={Nature Chemical Biology},
  volume={5},
  number={7},
  pages={502--507},
  year={2009},
  publisher={Nature Publishing Group US New York}
}

@article{beattie2003cyclin,
  title={Cyclin-dependent kinase 4 inhibitors as a treatment for cancer. Part 1: identification and optimisation of substituted 4, 6-bis anilino pyrimidines},
  author={Beattie, John F and Breault, Gloria A and Ellston, Rebecca PA and Green, Stephen and Jewsbury, Philip J and Midgley, Catherine J and Naven, Russell T and Minshull, Claire A and Pauptit, Richard A and Tucker, Julie A and others},
  journal={Bioorganic \& medicinal chemistry letters},
  volume={13},
  number={18},
  pages={2955--2960},
  year={2003},
  publisher={Elsevier}
}

@article{jiang2026macrocycle,
  title={Macrocycle-{DB}: a comprehensive database for macrocycle-based drug discovery},
  author={Jiang, Minchuan and Liu, Tianyue and Hussain, Muzammal and Luo, Yixin and Zheng, Rui and Hou, Tingjun and Lu, Xiaoyun and Zhou, Yang},
  journal={Nucleic Acids Research},
  volume={54},
  number={D1},
  pages={D1469--D1476},
  year={2026},
  publisher={Oxford University Press}
}

@article{schaufelberger2026generating,
  title={Generating Cyclic Conformers with Flow Matching in Cremer-Pople Coordinates},
  author={Schaufelberger, Luca and Hartgers, Aline and Jorner, Kjell},
  journal={arXiv preprint arXiv:2601.12859},
  year={2026}
}

@article{lipinski1997experimental,
  title={Experimental and computational approaches to estimate solubility and permeability in drug discovery and development settings},
  author={Lipinski, Christopher A and Lombardo, Franco and Dominy, Beryl W and Feeney, Paul J},
  journal={Advanced Drug Delivery Reviews},
  volume={23},
  number={1-3},
  pages={3--25},
  year={1997},
  publisher={Elsevier}
}

@inproceedings{
schneuing2025multidomain,
title={Multi-domain Distribution Learning for De Novo Drug Design},
author={Arne Schneuing and Ilia Igashov and Adrian W. Dobbelstein and Thomas Castiglione and Michael M. Bronstein and Bruno Correia},
booktitle={The Thirteenth International Conference on Learning Representations},
year={2025}
}

@article{schneuing2024structure,
  title={Structure-based drug design with equivariant diffusion models},
  author={Schneuing, Arne and Harris, Charles and Du, Yuanqi and Didi, Kieran and Jamasb, Arian and Igashov, Ilia and Du, Weitao and Gomes, Carla and Blundell, Tom L and Lio, Pietro and others},
  journal={Nature Computational Science},
  volume={4},
  number={12},
  pages={899--909},
  year={2024},
  publisher={Nature Publishing Group US New York}
}
\bibliographystyle{icml2026}

\newpage
\appendix
\onecolumn

\section{Additional Related Literature}
\subsection{Cyclic Peptide Generation}
\label{ap:cyclic_peptides}

Most existing approaches focus on a restricted subclass of macrocycles, called cyclic peptides, where all building blocks are amino acids linked by peptide bonds. AfCycDesign adapts the AlphaFold framework \citep{jumper2021highly} by encoding cyclicity through a modified positional encoding, enabling the generation and structural prediction of cyclic peptides \citep{rettie2025cyclic}. Similarly, RFpeptides builds on RFdiffusion \citep{watson2023novo} and introduces a cyclic relative positional encoding to generate macrocyclic peptide backbones \citep{rettie2025accurate}. Other works instead incorporate explicit geometric constraints or harmonic SDE formulations \cite{jiang2025zeroshot, zhou2025designing}. However, the design space of cyclic peptides is substantially smaller and more structured than that of general macrocycles, making the generation task comparatively easier. As a result, these methods do not readily extend to the broader and more diverse class of non-peptidic macrocycles.

Additionally, a closely related problem is the prediction of 3D macrocyclic conformations from a given 2D molecular graph. This task was addressed by RINGER, which introduced a diffusion-based transformer model to generate ensembles of 3D conformations for macrocyclic peptides \citep{grambow2023accurate}. More recently, PuckerFlow proposed a flow-matching approach operating in Cremer-Pople space to generate conformers of small cyclic molecules \cite{schaufelberger2026generating}.

\subsection{Applications of Persistent Homology in Machine Learning}
\label{ap:other_tda_papers}

\paragraph{General use of persistent homology in ML.}
Persistent homology (PH) has established itself as a critical framework in machine learning for characterizing data topology across multiple scales. Early integration strategies focused on creating differentiable, topology-aware architectures and losses. Notably, Topology Layer \citep{gabrielsson2020topology} and PersLay \citep{carriere2020perslay} allow deep neural networks to directly process persistence diagram features. PHom‑GeM \citep{charlier2019phom} instead introduced a topological loss function based on the bottleneck distance, ensuring that generated distributions faithfully reproduce the topology of the training set.

\paragraph{Applications to diffusion models.}
More recently, topological insights have been effectively adapted for diffusion-based generative frameworks across image, 3D, and graph domains. In the context of 3D and visual generation, Topology-Aware Latent Diffusion models \citep{hu2024topology} leverage PH to guide the formation of 3D shapes with specific topological characteristics. Similarly, \citet{wang2025topologyguidancecontrollingoutput} demonstrated that diffusion models can be steered using coordinate-based networks to satisfy rigorous topological constraints, while TopoDiffusionNet \citep{gupta2024topodiffusionnet} introduced a differentiable mechanism to regulate the number of objects (0-dim) and holes (1-dim) in generated image masks. In the graph domain, PH is increasingly used to capture complex connectivity; \citet{parktopology} developed a loss function based on persistence diagram matching for graph diffusion, and HOG-Diff \citep{huang2025hog} introduced a framework for generating higher-order structures via cell complex filtering and spectral diffusion.

\paragraph{Molecular applications.} In the realm of chemistry, persistent homology has proven effective for encoding molecular structure into ML-friendly representations. \citet{townsend2020representation} introduced persistence images, 2D embeddings of PH-derived features, as molecular fingerprints to screen vast chemical libraries like GDB‑9. Later, \citet{demir2023multiparameter} extended this with multiparameter PH fingerprints, incorporating atomic descriptors such as partial charges and bond types to enhance property prediction tasks.
Finally, TDA has been used to augment a VAE with topological molecule representations which resulted in an improvement over a range of molecular generation metrics \citep{schiff2022augmenting}.
\newpage
\section{Discussion and Future Work}\label{sec:discussion}

\paragraph{Leveraging base model advancements.} Since \name enhances generation metrics, it is well-positioned to benefit from rapid advancements in base model architectures. A promising avenue for future work is the integration of \name with recent, high-capacity models, potentially unlocking even greater performance.

\textbf{Bond-level information.} \ \ 
Our method relies on atomic positions to enforce macrocyclicity, without explicitly constraining the neighboring atoms in the cycle to form bonds. This design naturally aligns with point-cloud-based models such as EDM, which do not represent bonds explicitly \citep{hoogeboom2022equivariant}. For architectures that additionally model bond information (e.g., MolDiff), one could in principle incorporate bond-level constraints into the guidance. Our preliminary investigations (Appendix~\ref{app:bond}) suggest that the added complexity of this approach outweighs the potential benefits, though it remains an open option for future work.

\textbf{Explicit hydrogens.} \ \ 
We apply \name to models that do not explicitly model hydrogens during training, a common setting in recent works, as the hydrogens can be easily inferred during post-processing to avoid unnecessary time and memory overhead. However, if one chooses to work with a model that represents hydrogens explicitly, their monovalence can hinder cycle closure. In this case, a simple workaround is to exclude hydrogens from the Vietoris-Rips complex.

\textbf{Synthetic accessibility.} \ \ 
Macrocycles offer attractive pharmacological properties, but their synthesis typically involves an additional entropic cost associated with ring closure. As macrocycle-generating models move closer to practical deployment, incorporating synthetic considerations will become increasingly important. A natural and promising next step is to integrate synthesis-aware constraints and retrosynthetic planning into the generative process, enabling closer alignment between generated candidates and experimentally accessible chemistries.

\newpage
\section{Proof of Theorem \ref{thm:rips_death_tetra}}
\label{app:proof}

\begin{lemma}[Vietoris-Rips death time of a regular n-gon] \label{lemma:ngon} Consider a regular n-gon of side length $\ell$. The death time d of the dominant $H_1$ component in the Vietoris-Rips filtration is given by: \begin{equation*} d = \ell \frac{\sin\left( \lceil \frac{n}{3} \rceil \frac{\pi}{n} \right)}{\sin\left( \frac{\pi}{n} \right)}  \end{equation*}
\end{lemma}
\begin{proof}[Proof of Lemma \ref{lemma:ngon}]
Following the theory of Vietoris-Rips complexes of the circle established by \citet{Adamaszek_2017}, the homotopy type of $VR(X; r)$ for a finite subset $X \subset S^1$ is determined by its winding fraction $wf(X; r)$. For a regular $n$-gon with side length $\ell$, the complex is equivalent to the clique complex of the cyclic graph $C_n^k$. 

 \textbf{Threshold for death:} The $H_1$ component  persists as long as the winding fraction satisfies $0 < wf < 1/3$.

 \textbf{Geometric jump:} For a regular $n$-gon, the winding fraction is defined as $wf = k/n$, where $k$ is the maximum number of edges a single chord of length $d$ can span. Thus, death occurs at the smallest $d$ such that $k = \lceil n/3 \rceil$.

\textbf{Exact formula:} Let $R$ be the circumradius of the $n$-gon. The side length $\ell$ and the death distance $d$ are chords related by:
\begin{equation*}
\ell = 2R \sin\left(\frac{\pi}{n}\right), \quad d = 2R \sin\left(\lceil n/3 \rceil \frac{\pi}{n}\right)
\end{equation*}
Dividing these yields the exact relation:
\begin{equation*}
d = \ell \frac{\sin\left( \lceil \frac{n}{3} \rceil \frac{\pi}{n} \right)}{\sin\left( \frac{\pi}{n} \right)}
\end{equation*}
\end{proof}

\begin{proof}[Proof of \Cref{thm:rips_death_tetra}]

\textbf{Extension to tetrahedral conformation:} Consider $n$ atoms in a zigzag cycle. Because the bond angles are constrained, the atoms alternate between two concentric radii. Death edges can only be between two of the $n/2$ atoms of the interior circle. These atoms form a planar regular $N$-gon where $N=n/2$. That means we can apply Lemma~\ref{lemma:ngon}. The proof geometry is illustrated in \Cref{fig:crown_proof}.

\textbf{Effective side length:} The distance between two adjacent atoms in this sub-polygon is determined by the bond length $\ell$ and the bond angle $\theta$. More precisely the law of cosines gives
\begin{equation*} \ell' = \sqrt{\ell^2 + \ell^2 - 2\ell^2 \cos\theta} = \ell \sqrt{2(1-\cos\theta)} \end{equation*}

\textbf{Exact formula:} Substituting $N=n/2$ and $\ell'$ into the Lemma: \begin{equation*} d = \ell \sqrt{2(1-\cos\theta)} \cdot \frac{\sin\left( \lceil \frac{n}{6} \rceil \frac{2\pi}{n} \right)}{\sin\left( \frac{2\pi}{n} \right)} \end{equation*}

\textbf{Linearization:} As $n\rightarrow +\infty$, we get
\begin{equation*} d \approx \ell \sqrt{2(1-\cos\theta)} \frac{\sqrt{3}/2}{2\pi/n} = \frac{n \ell \sqrt{6(1-\cos\theta)}}{4\pi} \end{equation*} Solving for $n$ yields the linear form: \begin{equation*} n = \frac{4\pi}{\sqrt{6(1-\cos\theta)}}\frac{d}{\ell} + \mathcal{O}(1) \end{equation*}
\end{proof}

\begin{figure}[htbp]
    \centering
    \begin{tikzpicture}[scale=2.5, cap=round, join=round]

        \def\N{12}     
        \def\Rout{1.6} 
        \def\Rin{1.15} 
        
        \foreach \i in {1,...,\N} {
            \pgfmathsetmacro{\angle}{360/\N * (\i - 1) + 90} 
            \ifodd\i
                \coordinate (A\i) at (\angle:\Rout);
            \else
                \coordinate (A\i) at (\angle:\Rin);
            \fi
        }

        \draw[blue!40, dashed, line width=0.8pt] 
            (A2) -- (A4) -- (A6) -- (A8) -- (A10) -- (A12) -- cycle;
        
        \draw[gray!70, line width=2.5pt] 
            (A1) -- (A2) -- (A3) -- (A4) -- (A5) -- (A6) -- 
            (A7) -- (A8) -- (A9) -- (A10) -- (A11) -- (A12) -- cycle;

        \foreach \i in {1,...,\N} {
            \ifodd\i
                \node[circle, fill=red!70, inner sep=2.5pt, draw=black!70] at (A\i) {};
            \else
                \node[circle, fill=blue!70, inner sep=2.5pt, draw=black!70] at (A\i) {};
            \fi
        }

        \pgfmathanglebetweenpoints{\pgfpointanchor{A1}{center}}{\pgfpointanchor{A12}{center}}
        \let\angLeft\pgfmathresult
        
        \pgfmathanglebetweenpoints{\pgfpointanchor{A1}{center}}{\pgfpointanchor{A2}{center}}
        \let\angRight\pgfmathresult
        
        \draw[orange, thick] (A1) +(\angLeft:0.35) arc (\angLeft:\angRight:0.35);
        
        \node[orange, font=\footnotesize, font=\bfseries] at ($(A1) + ({(\angLeft+\angRight)/2}:0.5)$) {$\theta$};

        
        \draw[<->, >=stealth, thick, blue!80] ($(A12)!0.1!(A2)$) -- ($(A2)!0.1!(A12)$) 
            node[midway, below, font=\small\bfseries] {$\ell'$};

        \draw[<->, >=stealth, thick] ($(A2)!0.1!(A3)$) -- ($(A3)!0.1!(A2)$) 
            node[midway, sloped, above, font=\small\bfseries, text=black] {$\ell$};

        \draw[<->, >=stealth, very thick, purple] (A6) -- (A10) 
            node[midway, sloped, fill=white, inner sep=1.5pt, font=\bfseries] {$H_1$ death ($d$)};

    \end{tikzpicture}
    \caption{\textbf{Proof geometry: crown-shaped $n$-gon ($n=12$).} The inner polygon \textcolor{blue!80}{(dashed blue)} determines the filtration death time $d$. The effective side length $\ell'$ is derived from the bond length $\ell$ and the interior bond angle.}
    \label{fig:crown_proof}
\end{figure}
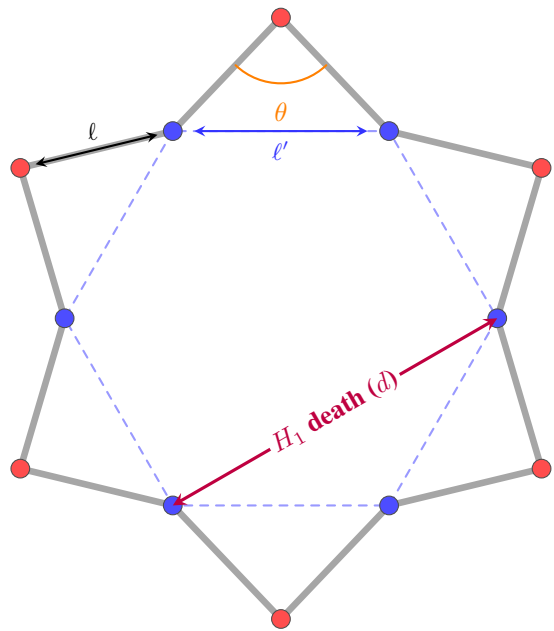

\begin{figure}[htbp]
    \centering
    
    \begin{subfigure}[c]{0.45\textwidth}
        \centering
        \begin{tikzpicture}[scale=1.5, font=\sffamily]
            \def\N{16}
            \def\R{1.5}
            \fill[green!10] (0,0) circle (\R);
            
            \draw[gray, line width=1.5pt] 
                ({90}:\R) 
                \foreach \i in {2,...,\N} { -- ({90 - 360/\N * (\i - 1)}:\R) } 
                -- cycle;

            \foreach \i in {1,...,\N} {
                \node[circle, fill=gray!80, inner sep=2pt, draw=black] (c\i) at ({90 - 360/\N * (\i - 1)}:\R) {};
            }

            \draw[<->, >=stealth, very thick, red] (c1) -- (c7) 
                node[midway, fill=green!10, text=red, inner sep=1.5pt, font=\bfseries\small] {$d_{circle}$};
        \end{tikzpicture}
        \caption{Circular Ring}
        \label{fig:circle}
    \end{subfigure}
    \hfill
    \begin{subfigure}[c]{0.45\textwidth}
        \centering
        \begin{tikzpicture}[scale=1.5, font=\sffamily]
            \path[use as bounding box] (0,0) circle (1.5);

            \fill[green!10] (0,0) ellipse (2.156 and 0.6);

            \foreach \x/\y [count=\i] in {
                0.000/0.600, 0.585/0.578, 1.165/0.505, 1.731/0.358,
                2.156/-0.000, 1.731/-0.358, 1.165/-0.505, 0.584/-0.578,
                -0.000/-0.600, -0.585/-0.577, -1.166/-0.505, -1.732/-0.358,
                -2.156/0.001, -1.730/0.358, -1.164/0.505, -0.584/0.578
            } {
                \coordinate (e\i) at (\x, \y);
            }

            \draw[gray, line width=1.5pt] (e1) 
                \foreach \i in {2,...,16} { -- (e\i) } 
                -- cycle;

            \foreach \i in {1,...,16} {
                \node[circle, fill=gray!80, inner sep=2pt, draw=black] at (e\i) {};
            }
            
            \draw[<->, >=stealth, very thick, red] (e1) -- (e8) 
                node[midway, fill=green!10, text=red, inner sep=1.5pt, font=\bfseries\small] {$d_{ellipse}$};
        \end{tikzpicture}
        \caption{Elliptical Ring}
        \label{fig:ellipse}
    \end{subfigure}
    
    \caption{\textbf{Impact of cycle shape on $H_1$ death $(n=16)$.} In an elongated elliptical macrocycle, the death time $d_{ellipse}$ is significantly reduced compared to $d_{circle}.$ The death times were obtained numerically using \texttt{torch\_topological}.}
    \label{fig:ellipse_shortcut}
\end{figure}
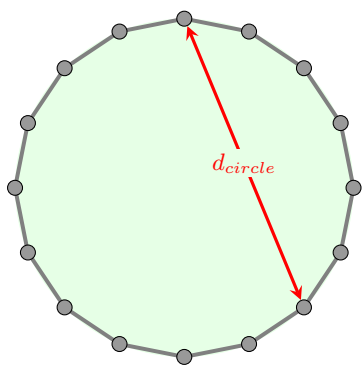
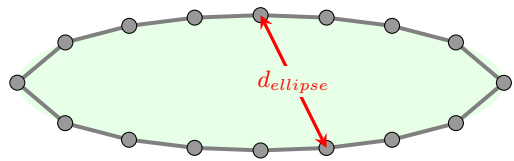

\section{Details on $F^{H_0}_{\mathrm{death}}$: Gradient Masking for Stability}
\label{ap:h0}

\begin{figure}[t]
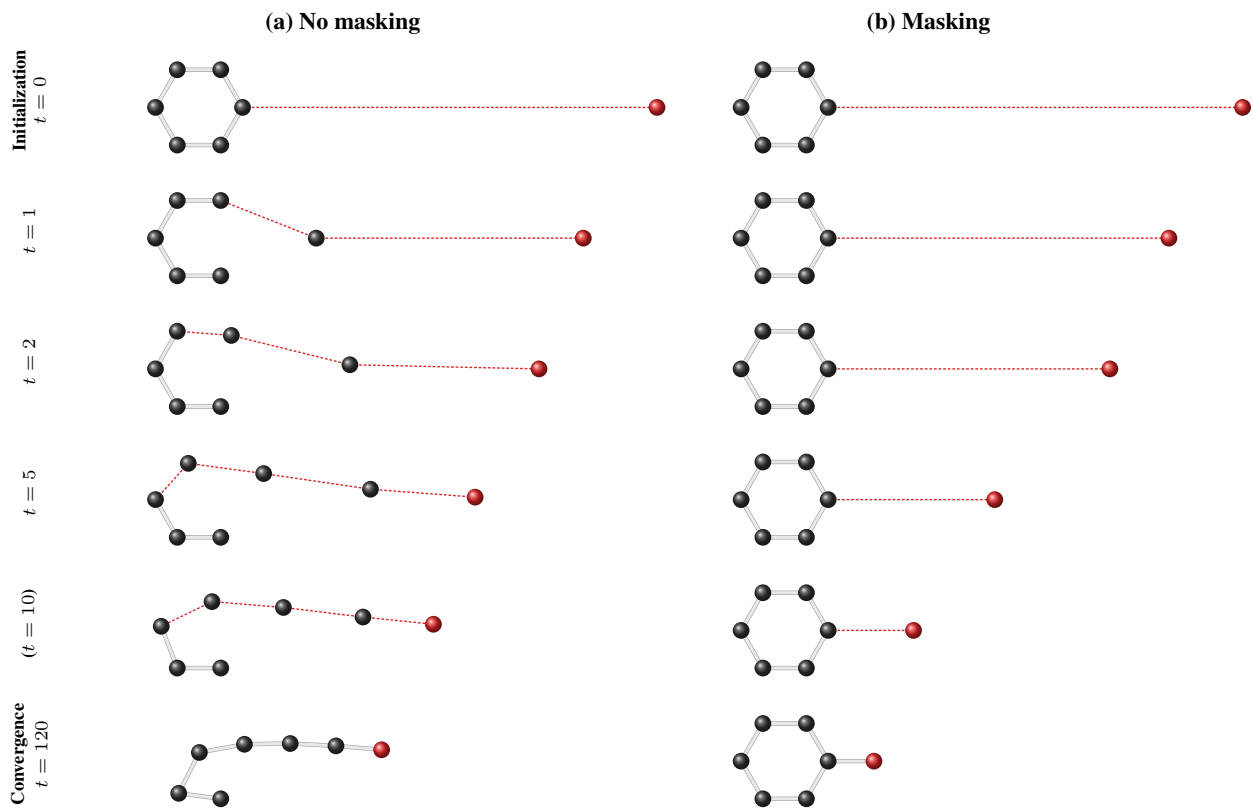

    \centering
    \setlength{\tabcolsep}{1pt} 
    \renewcommand{\arraystretch}{0.1}

    \providecommand{\commonboxTight}{\path[use as bounding box] (-5.5,-2.5) rectangle (16.5,2.5);}

    \begin{tabular}{ c c c }
         & \multicolumn{1}{c}{\small \textbf{(a) No masking}} & \multicolumn{1}{c}{\small \textbf{(b) Masking}} \\[1ex]

        \makecell[c]{\rotatebox{90}{\shortstack[c]{\scriptsize \textbf{Initialization} \\ \scriptsize $t=0$}}} & 
        \adjustbox{valign=m, max width=0.45\linewidth}{
            \begin{tikzpicture} 
                \commonboxTight 
\definecolor{chemcarbon}{HTML}{333333}
\definecolor{chemoxygen}{HTML}{E02222}
\definecolor{chembondfill}{HTML}{E8E8E8}
\definecolor{chembondoutline}{HTML}{999999}
\definecolor{chemactive}{HTML}{D62728}

\draw[chembondoutline, line width=0.6pt, double=chembondfill, double distance=3.5pt, line cap=round] (1.663,0.000) -- (0.831,1.440);
\draw[chembondoutline, line width=0.6pt, double=chembondfill, double distance=3.5pt, line cap=round] (0.831,1.440) -- (-0.831,1.440);
\draw[chembondoutline, line width=0.6pt, double=chembondfill, double distance=3.5pt, line cap=round] (-1.663,0.000) -- (-0.831,-1.440);
\draw[chemactive, line width=1.5pt, dashed, dash pattern=on 3pt off 2pt] (1.663,0.000) -- (17.500,0.000);
\draw[chembondoutline, line width=0.6pt, double=chembondfill, double distance=3.5pt, line cap=round] (-0.831,1.440) -- (-1.663,0.000);
\draw[chembondoutline, line width=0.6pt, double=chembondfill, double distance=3.5pt, line cap=round] (-0.831,-1.440) -- (0.831,-1.440);
\draw[chembondoutline, line width=0.6pt, double=chembondfill, double distance=3.5pt, line cap=round] (1.663,0.000) -- (0.831,-1.440);

\shade[shading=ball, ball color=chemcarbon] (1.663,0.000) circle (0.32cm);
\shade[shading=ball, ball color=chemcarbon] (0.831,1.440) circle (0.32cm);
\shade[shading=ball, ball color=chemcarbon] (-0.831,1.440) circle (0.32cm);
\shade[shading=ball, ball color=chemcarbon] (-1.663,0.000) circle (0.32cm);
\shade[shading=ball, ball color=chemcarbon] (-0.831,-1.440) circle (0.32cm);
\shade[shading=ball, ball color=chemcarbon] (0.831,-1.440) circle (0.32cm);
\shade[shading=ball, ball color=chemoxygen] (17.500,0.000) circle (0.32cm); 
            \end{tikzpicture}
        } & 
        \adjustbox{valign=m, max width=0.45\linewidth}{
            \begin{tikzpicture} 
                \commonboxTight 
\definecolor{chemcarbon}{HTML}{333333}
\definecolor{chemoxygen}{HTML}{E02222}
\definecolor{chembondfill}{HTML}{E8E8E8}
\definecolor{chembondoutline}{HTML}{999999}
\definecolor{chemactive}{HTML}{D62728}

\draw[chembondoutline, line width=0.6pt, double=chembondfill, double distance=3.5pt, line cap=round] (1.663,0.000) -- (0.831,1.440);
\draw[chembondoutline, line width=0.6pt, double=chembondfill, double distance=3.5pt, line cap=round] (0.831,1.440) -- (-0.831,1.440);
\draw[chembondoutline, line width=0.6pt, double=chembondfill, double distance=3.5pt, line cap=round] (-1.663,0.000) -- (-0.831,-1.440);
\draw[chemactive, line width=1.5pt, dashed, dash pattern=on 3pt off 2pt] (1.663,0.000) -- (17.500,0.000);
\draw[chembondoutline, line width=0.6pt, double=chembondfill, double distance=3.5pt, line cap=round] (-0.831,1.440) -- (-1.663,0.000);
\draw[chembondoutline, line width=0.6pt, double=chembondfill, double distance=3.5pt, line cap=round] (-0.831,-1.440) -- (0.831,-1.440);
\draw[chembondoutline, line width=0.6pt, double=chembondfill, double distance=3.5pt, line cap=round] (1.663,0.000) -- (0.831,-1.440);

\shade[shading=ball, ball color=chemcarbon] (1.663,0.000) circle (0.32cm);
\shade[shading=ball, ball color=chemcarbon] (0.831,1.440) circle (0.32cm);
\shade[shading=ball, ball color=chemcarbon] (-0.831,1.440) circle (0.32cm);
\shade[shading=ball, ball color=chemcarbon] (-1.663,0.000) circle (0.32cm);
\shade[shading=ball, ball color=chemcarbon] (-0.831,-1.440) circle (0.32cm);
\shade[shading=ball, ball color=chemcarbon] (0.831,-1.440) circle (0.32cm);
\shade[shading=ball, ball color=chemoxygen] (17.500,0.000) circle (0.32cm); 
            \end{tikzpicture}
        } \\[1ex]

        \makecell[c]{\rotatebox{90}{\scriptsize \textbf{$t=1$}}} & 
        \adjustbox{valign=m, max width=0.45\linewidth}{
            \begin{tikzpicture} 
                \commonboxTight 
                \input{h0_figs/fig_baseline_step001.tex} 
            \end{tikzpicture}
        } & 
        \adjustbox{valign=m, max width=0.45\linewidth}{
            \begin{tikzpicture} 
                \commonboxTight 
\definecolor{chemcarbon}{HTML}{333333}
\definecolor{chemoxygen}{HTML}{E02222}
\definecolor{chembondfill}{HTML}{E8E8E8}
\definecolor{chembondoutline}{HTML}{999999}
\definecolor{chemactive}{HTML}{D62728}

\draw[chembondoutline, line width=0.6pt, double=chembondfill, double distance=3.5pt, line cap=round] (1.663,0.000) -- (0.831,1.440);
\draw[chembondoutline, line width=0.6pt, double=chembondfill, double distance=3.5pt, line cap=round] (0.831,1.440) -- (-0.831,1.440);
\draw[chembondoutline, line width=0.6pt, double=chembondfill, double distance=3.5pt, line cap=round] (-1.663,0.000) -- (-0.831,-1.440);
\draw[chemactive, line width=1.5pt, dashed, dash pattern=on 3pt off 2pt] (1.663,0.000) -- (14.683,0.000);
\draw[chembondoutline, line width=0.6pt, double=chembondfill, double distance=3.5pt, line cap=round] (-0.831,1.440) -- (-1.663,0.000);
\draw[chembondoutline, line width=0.6pt, double=chembondfill, double distance=3.5pt, line cap=round] (-0.831,-1.440) -- (0.831,-1.440);
\draw[chembondoutline, line width=0.6pt, double=chembondfill, double distance=3.5pt, line cap=round] (1.663,0.000) -- (0.831,-1.440);

\shade[shading=ball, ball color=chemcarbon] (1.663,0.000) circle (0.32cm);
\shade[shading=ball, ball color=chemcarbon] (0.831,1.440) circle (0.32cm);
\shade[shading=ball, ball color=chemcarbon] (-0.831,1.440) circle (0.32cm);
\shade[shading=ball, ball color=chemcarbon] (-1.663,0.000) circle (0.32cm);
\shade[shading=ball, ball color=chemcarbon] (-0.831,-1.440) circle (0.32cm);
\shade[shading=ball, ball color=chemcarbon] (0.831,-1.440) circle (0.32cm);
\shade[shading=ball, ball color=chemoxygen] (14.683,0.000) circle (0.32cm); 
            \end{tikzpicture}
        } \\[1ex]

        \makecell[c]{\rotatebox{90}{\scriptsize \textbf{$t=2$}}} & 
        \adjustbox{valign=m, max width=0.45\linewidth}{
            \begin{tikzpicture} 
                \commonboxTight 
                \input{h0_figs/fig_baseline_step002.tex} 
            \end{tikzpicture}
        } & 
        \adjustbox{valign=m, max width=0.45\linewidth}{
            \begin{tikzpicture} 
                \commonboxTight 
\definecolor{chemcarbon}{HTML}{333333}
\definecolor{chemoxygen}{HTML}{E02222}
\definecolor{chembondfill}{HTML}{E8E8E8}
\definecolor{chembondoutline}{HTML}{999999}
\definecolor{chemactive}{HTML}{D62728}

\draw[chembondoutline, line width=0.6pt, double=chembondfill, double distance=3.5pt, line cap=round] (1.663,0.000) -- (0.831,1.440);
\draw[chembondoutline, line width=0.6pt, double=chembondfill, double distance=3.5pt, line cap=round] (0.831,1.440) -- (-0.831,1.440);
\draw[chembondoutline, line width=0.6pt, double=chembondfill, double distance=3.5pt, line cap=round] (-1.663,0.000) -- (-0.831,-1.440);
\draw[chemactive, line width=1.5pt, dashed, dash pattern=on 3pt off 2pt] (1.663,0.000) -- (12.428,0.000);
\draw[chembondoutline, line width=0.6pt, double=chembondfill, double distance=3.5pt, line cap=round] (-0.831,1.440) -- (-1.663,0.000);
\draw[chembondoutline, line width=0.6pt, double=chembondfill, double distance=3.5pt, line cap=round] (-0.831,-1.440) -- (0.831,-1.440);
\draw[chembondoutline, line width=0.6pt, double=chembondfill, double distance=3.5pt, line cap=round] (1.663,0.000) -- (0.831,-1.440);

\shade[shading=ball, ball color=chemcarbon] (1.663,0.000) circle (0.32cm);
\shade[shading=ball, ball color=chemcarbon] (0.831,1.440) circle (0.32cm);
\shade[shading=ball, ball color=chemcarbon] (-0.831,1.440) circle (0.32cm);
\shade[shading=ball, ball color=chemcarbon] (-1.663,0.000) circle (0.32cm);
\shade[shading=ball, ball color=chemcarbon] (-0.831,-1.440) circle (0.32cm);
\shade[shading=ball, ball color=chemcarbon] (0.831,-1.440) circle (0.32cm);
\shade[shading=ball, ball color=chemoxygen] (12.428,0.000) circle (0.32cm); 
            \end{tikzpicture}
        } \\[1ex]

        \makecell[c]{\rotatebox{90}{\scriptsize \textbf{$t=5$}}} & 
        \adjustbox{valign=m, max width=0.45\linewidth}{
            \begin{tikzpicture} 
                \commonboxTight 
                \input{h0_figs/fig_baseline_step005.tex} 
            \end{tikzpicture}
        } & 
        \adjustbox{valign=m, max width=0.45\linewidth}{
            \begin{tikzpicture} 
                \commonboxTight 
\definecolor{chemcarbon}{HTML}{333333}
\definecolor{chemoxygen}{HTML}{E02222}
\definecolor{chembondfill}{HTML}{E8E8E8}
\definecolor{chembondoutline}{HTML}{999999}
\definecolor{chemactive}{HTML}{D62728}

\draw[chembondoutline, line width=0.6pt, double=chembondfill, double distance=3.5pt, line cap=round] (1.663,0.000) -- (0.831,1.440);
\draw[chembondoutline, line width=0.6pt, double=chembondfill, double distance=3.5pt, line cap=round] (0.831,1.440) -- (-0.831,1.440);
\draw[chembondoutline, line width=0.6pt, double=chembondfill, double distance=3.5pt, line cap=round] (-1.663,0.000) -- (-0.831,-1.440);
\draw[chemactive, line width=1.5pt, dashed, dash pattern=on 3pt off 2pt] (1.663,0.000) -- (8.029,0.000);
\draw[chembondoutline, line width=0.6pt, double=chembondfill, double distance=3.5pt, line cap=round] (-0.831,1.440) -- (-1.663,0.000);
\draw[chembondoutline, line width=0.6pt, double=chembondfill, double distance=3.5pt, line cap=round] (-0.831,-1.440) -- (0.831,-1.440);
\draw[chembondoutline, line width=0.6pt, double=chembondfill, double distance=3.5pt, line cap=round] (1.663,0.000) -- (0.831,-1.440);

\shade[shading=ball, ball color=chemcarbon] (1.663,0.000) circle (0.32cm);
\shade[shading=ball, ball color=chemcarbon] (0.831,1.440) circle (0.32cm);
\shade[shading=ball, ball color=chemcarbon] (-0.831,1.440) circle (0.32cm);
\shade[shading=ball, ball color=chemcarbon] (-1.663,0.000) circle (0.32cm);
\shade[shading=ball, ball color=chemcarbon] (-0.831,-1.440) circle (0.32cm);
\shade[shading=ball, ball color=chemcarbon] (0.831,-1.440) circle (0.32cm);
\shade[shading=ball, ball color=chemoxygen] (8.029,0.000) circle (0.32cm); 
            \end{tikzpicture}
        } \\[1ex]

        \makecell[c]{\rotatebox{90}{\scriptsize \textbf{($t=10$)}}} & 
        \adjustbox{valign=m, max width=0.45\linewidth}{
            \begin{tikzpicture} 
                \commonboxTight 
                \input{h0_figs/fig_baseline_step010.tex} 
            \end{tikzpicture}
        } & 
        \adjustbox{valign=m, max width=0.45\linewidth}{
            \begin{tikzpicture} 
                \commonboxTight 
\definecolor{chemcarbon}{HTML}{333333}
\definecolor{chemoxygen}{HTML}{E02222}
\definecolor{chembondfill}{HTML}{E8E8E8}
\definecolor{chembondoutline}{HTML}{999999}
\definecolor{chemactive}{HTML}{D62728}

\draw[chembondoutline, line width=0.6pt, double=chembondfill, double distance=3.5pt, line cap=round] (1.663,0.000) -- (0.831,1.440);
\draw[chembondoutline, line width=0.6pt, double=chembondfill, double distance=3.5pt, line cap=round] (0.831,1.440) -- (-0.831,1.440);
\draw[chembondoutline, line width=0.6pt, double=chembondfill, double distance=3.5pt, line cap=round] (-1.663,0.000) -- (-0.831,-1.440);
\draw[chemactive, line width=1.5pt, dashed, dash pattern=on 3pt off 2pt] (1.663,0.000) -- (4.925,0.000);
\draw[chembondoutline, line width=0.6pt, double=chembondfill, double distance=3.5pt, line cap=round] (-0.831,1.440) -- (-1.663,0.000);
\draw[chembondoutline, line width=0.6pt, double=chembondfill, double distance=3.5pt, line cap=round] (-0.831,-1.440) -- (0.831,-1.440);
\draw[chembondoutline, line width=0.6pt, double=chembondfill, double distance=3.5pt, line cap=round] (1.663,0.000) -- (0.831,-1.440);

\shade[shading=ball, ball color=chemcarbon] (1.663,0.000) circle (0.32cm);
\shade[shading=ball, ball color=chemcarbon] (0.831,1.440) circle (0.32cm);
\shade[shading=ball, ball color=chemcarbon] (-0.831,1.440) circle (0.32cm);
\shade[shading=ball, ball color=chemcarbon] (-1.663,0.000) circle (0.32cm);
\shade[shading=ball, ball color=chemcarbon] (-0.831,-1.440) circle (0.32cm);
\shade[shading=ball, ball color=chemcarbon] (0.831,-1.440) circle (0.32cm);
\shade[shading=ball, ball color=chemoxygen] (4.925,0.000) circle (0.32cm); 
            \end{tikzpicture}
        } \\[1ex]

       \makecell[c]{\rotatebox{90}{\shortstack[c]{\scriptsize \textbf{Convergence} \\ \scriptsize $t=120$}}} & 
        \adjustbox{valign=m, max width=0.45\linewidth}{
            \begin{tikzpicture} 
                \commonboxTight 
\definecolor{chemcarbon}{HTML}{333333}
\definecolor{chemoxygen}{HTML}{E02222}
\definecolor{chembondfill}{HTML}{E8E8E8}
\definecolor{chembondoutline}{HTML}{999999}
\definecolor{chemactive}{HTML}{D62728}

\draw[chembondoutline, line width=0.6pt, double=chembondfill, double distance=3.5pt, line cap=round] (5.233,0.584) -- (3.485,0.674);
\draw[chembondoutline, line width=0.6pt, double=chembondfill, double distance=3.5pt, line cap=round] (3.485,0.674) -- (1.735,0.645);
\draw[chembondoutline, line width=0.6pt, double=chembondfill, double distance=3.5pt, line cap=round] (0.013,0.332) -- (-0.775,-1.230);
\draw[chembondoutline, line width=0.6pt, double=chembondfill, double distance=3.5pt, line cap=round] (5.233,0.584) -- (6.977,0.435);
\draw[chembondoutline, line width=0.6pt, double=chembondfill, double distance=3.5pt, line cap=round] (1.735,0.645) -- (0.013,0.332);
\draw[chembondoutline, line width=0.6pt, double=chembondfill, double distance=3.5pt, line cap=round] (-0.775,-1.230) -- (0.831,-1.440);

\shade[shading=ball, ball color=chemcarbon] (5.233,0.584) circle (0.32cm);
\shade[shading=ball, ball color=chemcarbon] (3.485,0.674) circle (0.32cm);
\shade[shading=ball, ball color=chemcarbon] (1.735,0.645) circle (0.32cm);
\shade[shading=ball, ball color=chemcarbon] (0.013,0.332) circle (0.32cm);
\shade[shading=ball, ball color=chemcarbon] (-0.775,-1.230) circle (0.32cm);
\shade[shading=ball, ball color=chemcarbon] (0.831,-1.440) circle (0.32cm);
\shade[shading=ball, ball color=chemoxygen] (6.977,0.435) circle (0.32cm); 
            \end{tikzpicture}
        } & 
        \adjustbox{valign=m, max width=0.45\linewidth}{
            \begin{tikzpicture} 
                \commonboxTight 
\definecolor{chemcarbon}{HTML}{333333}
\definecolor{chemoxygen}{HTML}{E02222}
\definecolor{chembondfill}{HTML}{E8E8E8}
\definecolor{chembondoutline}{HTML}{999999}
\definecolor{chemactive}{HTML}{D62728}

\draw[chembondoutline, line width=0.6pt, double=chembondfill, double distance=3.5pt, line cap=round] (1.663,0.000) -- (0.831,1.440);
\draw[chembondoutline, line width=0.6pt, double=chembondfill, double distance=3.5pt, line cap=round] (0.831,1.440) -- (-0.831,1.440);
\draw[chembondoutline, line width=0.6pt, double=chembondfill, double distance=3.5pt, line cap=round] (-1.663,0.000) -- (-0.831,-1.440);
\draw[chembondoutline, line width=0.6pt, double=chembondfill, double distance=3.5pt, line cap=round] (1.663,0.000) -- (3.413,0.000);
\draw[chembondoutline, line width=0.6pt, double=chembondfill, double distance=3.5pt, line cap=round] (-0.831,1.440) -- (-1.663,0.000);
\draw[chembondoutline, line width=0.6pt, double=chembondfill, double distance=3.5pt, line cap=round] (-0.831,-1.440) -- (0.831,-1.440);
\draw[chembondoutline, line width=0.6pt, double=chembondfill, double distance=3.5pt, line cap=round] (1.663,0.000) -- (0.831,-1.440);

\shade[shading=ball, ball color=chemcarbon] (1.663,0.000) circle (0.32cm);
\shade[shading=ball, ball color=chemcarbon] (0.831,1.440) circle (0.32cm);
\shade[shading=ball, ball color=chemcarbon] (-0.831,1.440) circle (0.32cm);
\shade[shading=ball, ball color=chemcarbon] (-1.663,0.000) circle (0.32cm);
\shade[shading=ball, ball color=chemcarbon] (-0.831,-1.440) circle (0.32cm);
\shade[shading=ball, ball color=chemcarbon] (0.831,-1.440) circle (0.32cm);
\shade[shading=ball, ball color=chemoxygen] (3.413,0.000) circle (0.32cm); 
            \end{tikzpicture}
        } \\

    \end{tabular}

    \caption{\textbf{Failure mode analysis.} Comparison of optimization dynamics under \Cref{eq:h0ap}. (a) Symmetric gradients cause "unzipping." (b) Masked gradients preserve structure.}
    \label{fig:comparison_dynamics}
\end{figure}

\paragraph{Method.} As presented in \Cref{sec:methodoverview}, connectivity is enforced by minimizing the $H_0$ persistence death times that exceed a threshold $\ell^\star$. In the Vietoris-Rips filtration, a finite $H_0$ death time corresponds to the scale at which two connected components merge. Formally, the connectivity loss is defined over the set of finite death times $\mathcal{D}_0(X)$ in the 0-dimensional persistence diagram:
\begin{equation}
\label{eq:h0ap}
    F^{H_0}_{\mathrm{death}}(X) = \sum\limits_{d \in \mathcal{D}_0(X)} \Big(\mathrm{ReLU} \big(d - \ell^\star\big)\Big)^2
\end{equation}
To optimize this objective, we introduce a \textbf{masked gradient} scheme. Each death value $d > \ell^\star$ is determined by the distance between two specific atoms $(u, v)$ that trigger the merging of two components (i.e., $\|x_u - x_v\| = d$). We define the centroid of the system as $\mathbf{c} = \frac{1}{N} \sum_{i=1}^N x_i$. The gradient applied to each endpoint of this critical edge goes through by a binary mask $m$:
\begin{equation*}
    g_u = m_u \cdot \nabla_{x_u} F^{H_0}_{\mathrm{death}}, \quad m_u = \mathbbm{1}\left(\|x_u - \mathbf{c}\| \geq \|x_v - \mathbf{c}\|\right)
\end{equation*}
In other words, for every pair of atoms responsible for the merging of two connected components (that are separated by more than $\ell^\star$), we block the gradient flow for the atom closer to the centroid ($m=0$) and allow it only for the atom further away ($m=1$).

\paragraph{Instability of symmetric optimization.}
The necessity of this masking becomes apparent when analyzing the case of a stable connected cluster $\mathcal{C}$ (stable in the sense that all the $H_0$ death times are less than $\ell^\star$) interacting with a distant outlier component $o$. Let's assume that the merging of the outlier with the cluster determines a death time $d = \|x_u - x_o\| > \ell^\star$, for some $u \in \mathcal{C}$ (i.e. the outlier is at a distance $d$ from the cluster).
Without masking, the gradient of \Cref{eq:h0ap} is applied symmetrically. The internal atom $u$ moves as
\begin{equation*}
    \Delta x_u \propto -\nabla_{x_u} F^{H_0}_{\mathrm{death}} \propto (x_o - x_u)
\end{equation*}
Since $o$ is an outlier, this vector points away from the cluster centroid. Consequently, $u$ is pulled out of its local equilibrium within $\mathcal{C}$. This displacement stretches the distances between $u$ and its neighbors. If the force exerted by the outlier is sufficiently strong (i.e., if it is sufficiently far), these internal distances cross the threshold $\ell^\star$, creating new high-persistence $H_0$ features that trigger their own penalty terms in Eq.~\ref{eq:h0ap}. This creates a cascading failure where the cluster is sequentially pulled apart, resulting in the \textit{linearization} or "unzipping" of the cycle structure, as demonstrated in Figure~\ref{fig:comparison_dynamics}a. By clamping the inner atom $u$ (via $m_u=0$), our method decouples the outlier's transport from the cluster's internal dynamics, ensuring structural preservation (Figure~\ref{fig:comparison_dynamics}b).

\paragraph{Non-conservative field.} Note that this masking renders the resulting guidance vector field non-conservative, effectively acting as a stabilized approximation of the exact score function described in \Cref{eq:tda_cond_score}. Nevertheless, the theoretical intuition provided in \Cref{sec:posteriorsampling} remains relevant, as the masked update still locally maximizes the conditional likelihood $p_t(y|x)$ to drive the system toward the desired topology, differing only by the suppression of specific unstable gradient~directions.

\section{Topological Features and Their Gradients Throughout Denoising}\label{app:tracking}

To better understand the success of \name, we analyze the evolution of the topological features and their associated gradients throughout the 1000 steps of the reverse diffusion process.

\paragraph{Feature convergence and constraint satisfaction.}
As shown in Figure~\ref{fig:birth_death}, the guidance successfully steers the molecule toward the target topology. All optimized features converge rapidly to their respective targets: the largest cycle size ($H_1$ death) reaches the target interval within the first few hundred steps, while cycle connectivity ($H_1$ birth) and molecular connectivity ($H_0$ death) are effectively kept below the threshold $\ell^\star$. Once satisfied, these constraints remain stable throughout the remainder of the denoising process, ensuring that the global macrocyclic structure is established early and preserved.

\paragraph{Gradient analysis.}
Figure~\ref{fig:gradient} illustrates that, on average, the topological gradients are of the same magnitude as the denoising gradients. We observe that $H_1$ death gradients are dominant at the very beginning of the process, driving the initial ring expansion. The $H_1$ birth gradients, while sparse throughout the trajectory, are constant in norm by design, providing sharp corrections when edges exceed the threshold $\ell^\star$. Crucially, in the last 100 steps (where fine-grained chemical details such as bond orders and angles are finalized) the gradients are dominated almost exclusively by the denoising term. This explains the success of our method: the guidance effectively pushes the molecule into a stable macrocyclic state early on, allowing the denoiser to perform optimal local validity optimization in the final stages without interference.

\paragraph{Comparison with squared $H_1$ birth.}
We compare this behavior with the one achieved by squaring the formula of $F^{H_1}_{\mathrm{birth}}$ as is done for the other guidance terms. As seen in Figure~\ref{fig:birth_death_square}, the $H_1$ birth feature in the squared setting stays very close to the target $\ell^\star$ but tends to oscillate around it rather than remaining strictly below. This behavior is confirmed by the gradient analysis in Figure~\ref{fig:gradient_square}, which shows that $H_1$ birth gradients remain very significant and exceeds the denoising gradients, even during the last 100 steps. Following the reasoning above, this persistent perturbation interferes with the denoiser's validity optimization in the critical final phase, explaining the lower performance observed with the squared $H_1$ birth term.

\begin{figure}[h!]
    \centering
    \begin{minipage}{0.48\linewidth}
        \centering
        \includegraphics[width=\linewidth]{figures/icml_birth_death.pdf}
    \end{minipage}\hfill
    \begin{minipage}{0.48\linewidth}
        \centering
        \includegraphics[width=\linewidth]{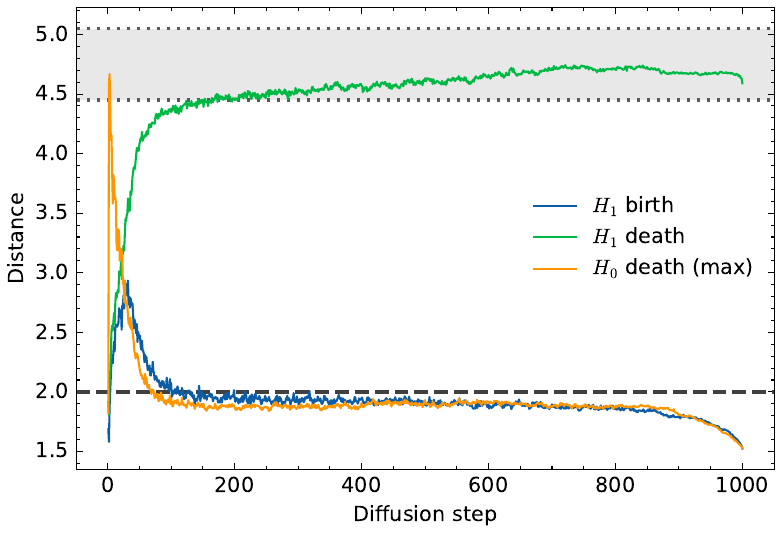}
    \end{minipage}
    \caption{Topological features (no square on $F^{H_1}_{\mathrm{birth}}$) for individual runs (left) and averaged runs (right).}
    \label{fig:birth_death}
\end{figure}

\begin{figure}
    \centering
    \includegraphics[width=0.5\linewidth]{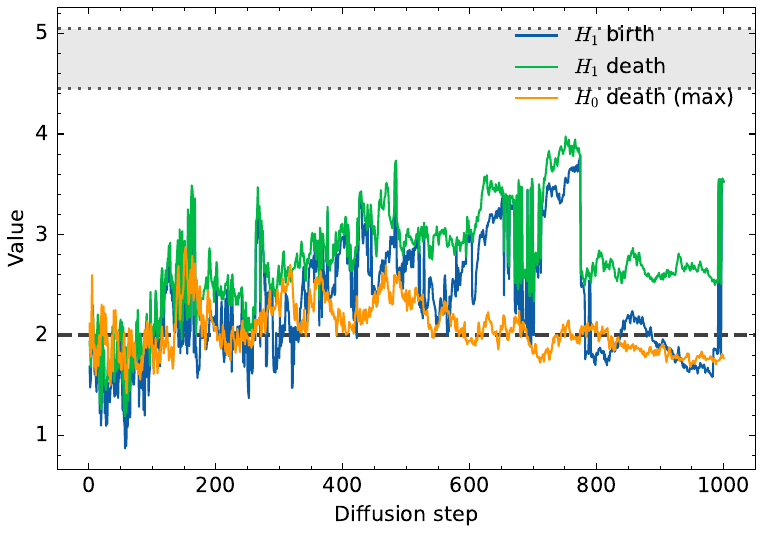}
    \caption{Topological features without guidance}
\end{figure}

\begin{figure}
    \centering
    \begin{minipage}{0.48\linewidth}
        \centering
        \includegraphics[width=\linewidth]{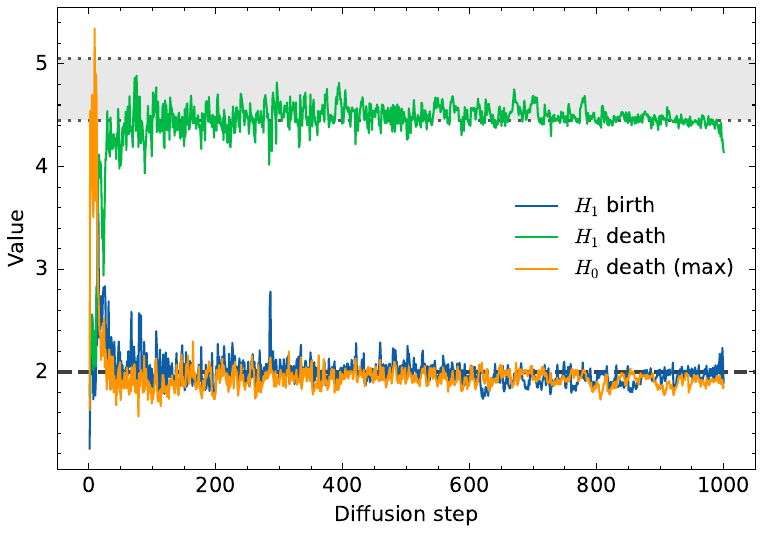}
    \end{minipage}\hfill
    \begin{minipage}{0.48\linewidth}
        \centering
        \includegraphics[width=\linewidth]{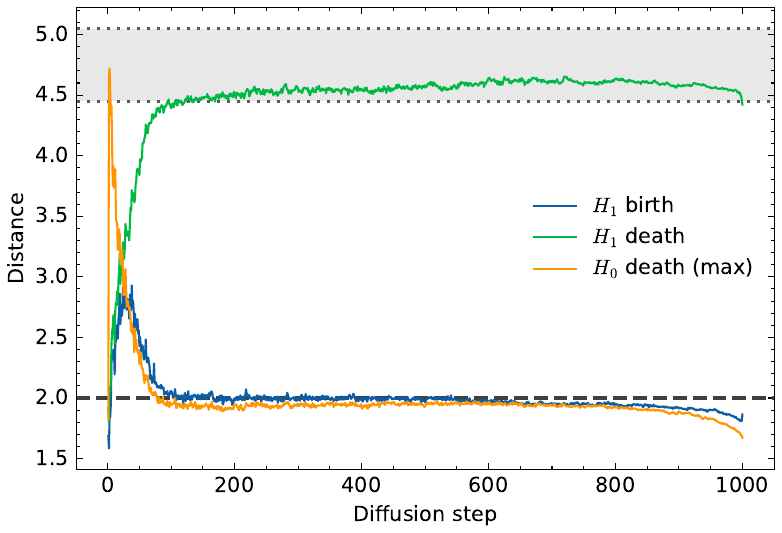}
    \end{minipage}
    \caption{Topological features \textbf{with squared $F^{H_1}_{\mathrm{birth}}$} for individual runs (left) and averaged runs (right).}
    \label{fig:birth_death_square}
\end{figure}

\begin{figure}
    \centering
    \begin{minipage}{0.48\linewidth}
        \centering
        \includegraphics[width=\linewidth]{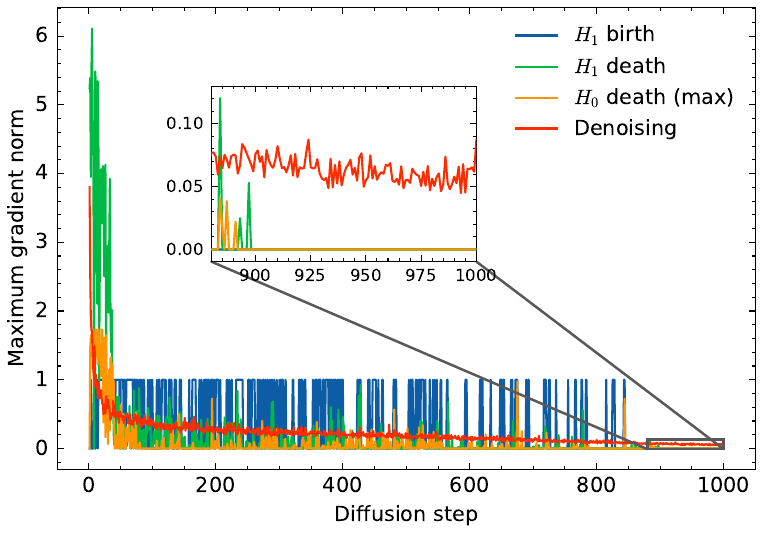}
    \end{minipage}\hfill
    \begin{minipage}{0.48\linewidth}
        \centering
        \includegraphics[width=\linewidth]{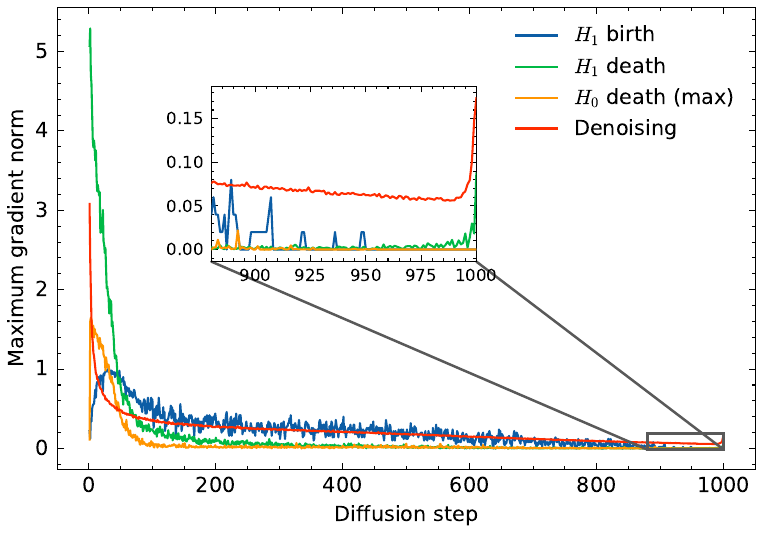}
    \end{minipage}
    \caption{Gradient statistics (no square on $F^{H_1}_{\mathrm{birth}}$) for individual runs (left) and averaged runs (right).}
    \label{fig:gradient}
\end{figure}

\begin{figure}
    \centering
    \begin{minipage}{0.48\linewidth}
        \centering
        \includegraphics[width=\linewidth]{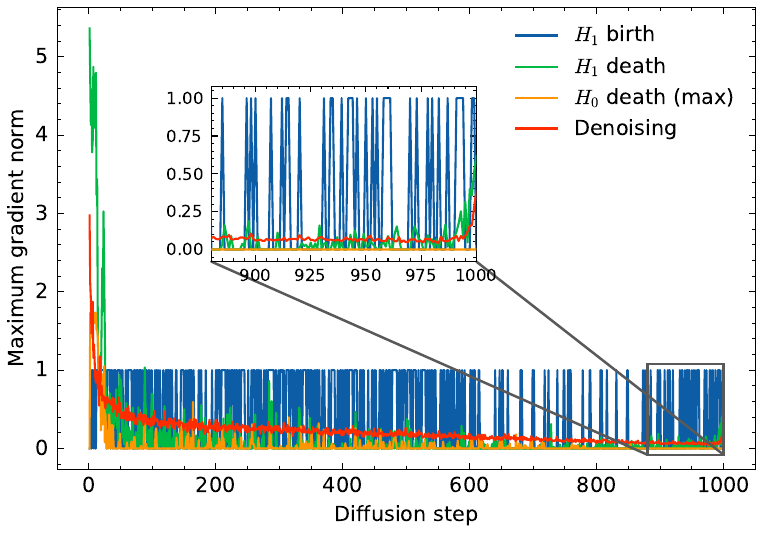}
    \end{minipage}\hfill
    \begin{minipage}{0.48\linewidth}
        \centering
        \includegraphics[width=\linewidth]{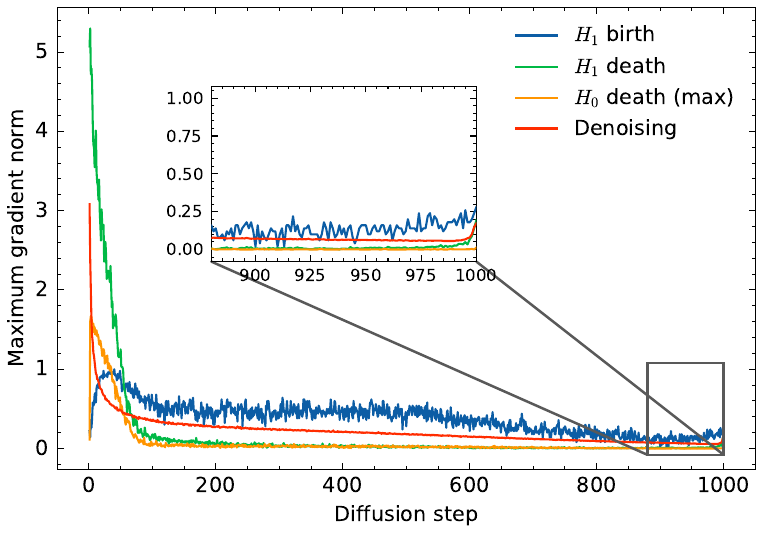}
    \end{minipage}
    \caption{Gradient statistics \textbf{with squared $F^{H_1}_{\mathrm{birth}}$} for individual runs (left) and averaged runs (right).}
    \label{fig:gradient_square}
\end{figure}
  
\newpage

\section{Detailed Algorithms}\label{app:algo}

We present the detailed algorithms used by \name. Algorithm~\ref{alg:denoise_with_topo_guidance_batch} shows how \name fits into the overall scheme of denoising. Algorithm~\ref{alg:macroguide} describes the high-level structure of \name, whereas algorithms~\ref{alg:h1_guidance} and \ref{alg:h0_guidance} describe its components: $H_1$ and $H_0$ guidances, respectively.

\paragraph{Relationship to continuous score matching.}
We note a specific detail regarding the implementation of topological guidance in Algorithm~\ref{alg:denoise_with_topo_guidance_batch} compared to the theoretical formulation in \Cref{eq:diffusion}. While \Cref{eq:diffusion} defines guidance as a modification to the instantaneous score function $\nabla_{x_t} \log p_t(y|x_t)$, our algorithm implements this as a sequential predictor-corrector step. Specifically, we first apply the base denoising update to obtain an intermediate state $x_{t-1}$, and then calculate the topological gradient $\Delta$ on this cleaner state. In this discrete setting, the scalar weight $\lambda$ acts as an effective step size that implicitly absorbs the SDE integration constants (e.g., $\sigma_t^2$) required to map the score-space gradient to a coordinate-space displacement.

\begin{algorithm}[h]
\caption{\textsc{DenoiseWith\name}}
\label{alg:denoise_with_topo_guidance_batch}
\begin{algorithmic}
\STATE {\bfseries Input:} denoiser $\epsilon_\theta$; noise schedule $\{\sigma_t\}_{t=1}^T$; initial sample $\mathbf{x}_T$; guidance parameters $(\lambda, \ell^*, d_{\min}, d_{\max})$, guidance schedule $\mathcal{T}_g \subset \{1,\dots,T\}$

\FOR{$t=T$ {\bfseries down to} $1$}
  \STATE $\hat{\epsilon} \leftarrow \epsilon_\theta(\mathbf{x}_t,\sigma_t)$
  \STATE $\mathbf{x}_{t-1} \leftarrow \textsc{DenoiseUpdate}(\mathbf{x}_t,\hat{\epsilon},\sigma_t)$
  \IF{ $t \in \mathcal{T}_g$}

    \STATE $\Delta \leftarrow \textsc{\name}(\mathbf{x}_{t-1}; \ell^*, d_{\min}, d_{\max})$ \COMMENT{Algorithm~\ref{alg:macroguide}}
    \STATE $\mathbf{x}_{t-1} \leftarrow \mathbf{x}_{t-1} + \lambda \cdot \Delta$ 
  \ENDIF
\ENDFOR
\STATE {\bfseries return} $\mathbf{x}_0$
\end{algorithmic}
\end{algorithm}

\begin{algorithm}[h]
   \caption{\name}
   \label{alg:macroguide}
\begin{algorithmic}
   \STATE {\bfseries Input:} Coordinates $\mathbf{X}$, $\ell^*, d_{\min}, d_{\max}$
   
   \STATE $\Pi \leftarrow \text{VietorisRips}(\mathbf{X})$
   
   \STATE $\mathbf{g}_{H_1} \leftarrow \text{\textsc{$H_1$Guidance}}(\mathbf{X}, \Pi, \ell^*, d_{\min}, d_{\max})$ \COMMENT{Algorithm \ref{alg:h1_guidance}}

   \STATE $\mathbf{g}_{H_0} \leftarrow {\textsc{$H_0$Guidance}}(\mathbf{X}, \Pi, \ell^*)$ \COMMENT{Algorithm \ref{alg:h0_guidance}}
   
   \STATE $\Delta \leftarrow -(\mathbf{g}_{H_1} + \mathbf{g}_{H_0})$
   \STATE {\bfseries Return} $\Delta$
\end{algorithmic}
\end{algorithm}

\begin{algorithm}[t!]
   \caption{\textsc{$H_1$Guidance}}
   \label{alg:h1_guidance}
\begin{algorithmic}
   \STATE {\bfseries Input:} coordinates $X$, persistence info $\Pi$, thresholds $\ell^\star, d_{\min}, d_{\max}$
   \STATE Initialize objectives $F^{H_1}_{\mathrm{birth}} \leftarrow 0, \quad F^{H_1}_{\mathrm{death}} \leftarrow 0$
   \FOR{each sample $k$ in batch}
       \STATE Extract $H_1$ diagram pairs $P$ from $\Pi_k$
       \IF{$P$ is empty}
           \STATE {\bfseries continue}
       \ENDIF
       \STATE \COMMENT{Identify the $H_1$ component (ring) that dies last}
       \STATE Find index $i^\star = \arg\max_i \text{death}(P_i)$
       \STATE Retrieve critical vertices for feature $i^\star$:
       \STATE \quad $(u_b, v_b)$ defining birth time $b^{(1)}_{i^\star}$
       \STATE \quad $(u_d, v_d)$ defining death time $d^{(1)}_{i^\star}$
       
       \STATE {\bfseries 1. Cycle Connectivity ($H_1$ birth):}
       \STATE $b^{(1)}_{i^\star} \leftarrow \|x_{u_b} - x_{v_b}\|$
       \STATE $F^{H_1}_{\mathrm{birth}} \leftarrow F^{H_1}_{\mathrm{birth}} + \mathrm{ReLU}(b^{(1)}_{i^\star} - \ell^\star)$
       
       \STATE {\bfseries 2. Cycle Size ($H_1$ death):}
       \STATE $d^{(1)}_{i^\star} \leftarrow \|x_{u_d} - x_{v_d}\|$
       \STATE \COMMENT{Penalize if outside target interval $[d_{\min}, d_{\max}]$}
       \IF{$d^{(1)}_{i^\star} < d_{\min}$}
           \STATE $F^{H_1}_{\mathrm{death}} \leftarrow F^{H_1}_{\mathrm{death}} + (d_{\min} - d^{(1)}_{i^\star})^2$
       \ELSIF{$d^{(1)}_{i^\star} > d_{\max}$}
           \STATE $F^{H_1}_{\mathrm{death}} \leftarrow F^{H_1}_{\mathrm{death}} + (d^{(1)}_{i^\star} - d_{\max})^2$
       \ENDIF
   \ENDFOR
   \STATE $\mathcal{F}_{\mathrm{TDA}}^{(1)} \leftarrow F^{H_1}_{\mathrm{birth}} + F^{H_1}_{\mathrm{death}}$
   \STATE Compute gradients $\mathbf{g}_{H_1} \leftarrow \nabla_{X} \mathcal{F}_{\mathrm{TDA}}^{(1)}$
   \STATE {\bfseries Return} $\mathbf{g}_{H_1}$
\end{algorithmic}
\end{algorithm}

\begin{algorithm}[h!]
   \caption{\textsc{$H_0$Guidance}}
   \label{alg:h0_guidance}
\begin{algorithmic}
   \STATE {\bfseries Input:} coordinates $X$, persistence info $\Pi$, threshold $\ell^\star$
   \STATE Initialize connectivity objective $F^{H_0}_{\mathrm{death}} \leftarrow 0$
   \FOR{each sample pair $(X^{(k)}, \pi)$ in $(X, \Pi)$}
   \STATE Extract finite death times $\mathcal{D}_0(X^{(k)})$ and corresponding pairs $P$ from $\pi$
   \STATE Identify active edges $E \leftarrow \{ (u,v) \in P \mid \text{death}(u,v) > \ell^\star \}$
   \IF{$E$ is empty}
   \STATE {\bfseries continue}
   \ENDIF
   \STATE $\mathbf{c} \leftarrow \frac{1}{N} \sum_j x_j$ \COMMENT{Calculate centroid of molecule $k$}
   \FOR{each critical edge $(u, v)$ in $E$}
   \STATE \COMMENT{Apply masked gradient: move only the atom furthest from centroid}
   \IF{$\|x_u - \mathbf{c}\| > \|x_v - \mathbf{c}\|$}
   \STATE $u \text{ is active } (m_u=1), \quad v \text{ is masked } (m_v=0)$
   \STATE $x_{active} \leftarrow x_u, \quad x_{masked} \leftarrow \text{stop\_gradient}(x_v)$
   \ELSE
   \STATE $v \text{ is active } (m_v=1), \quad u \text{ is masked } (m_u=0)$
   \STATE $x_{active} \leftarrow x_v, \quad x_{masked} \leftarrow \text{stop\_gradient}(x_u)$
   \ENDIF
   \STATE $d^{(0)}_{j} \leftarrow \|x_{active} - x_{masked}\|$
   \STATE $F^{H_0}_{\mathrm{death}} \leftarrow F^{H_0}_{\mathrm{death}} + \big(\mathrm{ReLU}(d^{(0)}_{j} - \ell^\star)\big)^2$
   \ENDFOR
   \ENDFOR
   \STATE $\mathbf{g}_{H_0} \leftarrow \nabla_{X} F^{H_0}_{\mathrm{death}}$
   \STATE {\bfseries Return} $\mathbf{g}_{H_0}$
\end{algorithmic}
\end{algorithm}

\FloatBarrier
\newpage
\section{Experimental Details}
\label{app:exp}

\subsection{MolDiff Training Dataset}
\label{app:geomdrugs}

The training split used for MolDiff contains 231,521 data points from GEOM-Drugs, out of which only 323 molecules had a ring of at least 12 atoms \cite{axelrod2022geom, peng2023moldiff}.

\subsection{Protein Conditioning with MolSnapper}
\label{app:protein}

Following the original setup, we selected a subset of atoms from the reference ligand to define a pharmacophore and used the provided protein pocket corresponding to PDB ID 1H00 \cite{ziv2025molsnapper, beattie2003cyclin}. We included all nitrogen atoms from the reference ligand, as these are the most likely to participate in specific interactions with the protein pocket and thus provide a strong conditioning signal. In total, the pharmacophore consists of five atoms: three nitrogens and two carbons.

\subsection{Computational Resources}
All experiments were conducted on NVIDIA A10 GPUs, except for finetuning which was performed on NVIDIA H100 GPUs. 

\section{Additional Baselines}
\label{app:baselines}
This section provides descriptions of the baselines used to benchmark \name in Tables \ref{tab:all_models_unconditional} and \ref{tab:all_models_conditional}.

\subsection{Finetuning}
\label{app:finetuning}
We finetune MolDiff using the publicly released pretrained checkpoint and continue training on a macrocycle-specific dataset~\cite{peng2023moldiff}. The finetuning procedure follows the original training protocol, and was run for 50{,}000 iterations. Due to the increased size of macrocyclic molecules relative to those in the original training set, which leads to higher memory consumption during training, we reduce the batch size from 256 to 64 to fit within hardware constraints.

The macrocyclic dataset is derived from Macrocycle-DB \cite{jiang2026macrocycle}, which contains 45{,}525 macrocyclic compounds with resolved three-dimensional structures. We follow the data processing pipeline of MolDiff~\cite{peng2023moldiff} and apply the same filtering criteria, in particular removing molecules containing atom types not present in the pretrained model. After filtering, 40{,}496 molecules remain. We split the resulting dataset into training and validation sets using an 85:15 ratio, yielding 34{,}422 training molecules and 6{,}074 validation molecules, which are used to monitor and control the finetuning process.

Because MolSnapper uses similar underlying architecture and checkpoints as MolDiff, it is possible to use the same finetuned denoising network in both models.

\subsection{Naive Topological Guidance}
\label{app:naive}\paragraph{Naive geometric cycle guidance.}
We introduce a simple geometric guidance term that biases generated molecular point clouds toward forming a single macrocyclic structure. Importantly, we do \emph{not} assume that the atoms are ordered along a ring. Instead, given a molecule with at least $n$ atoms, we select an arbitrary subset of $n$ atoms (e.g., the first $n$ atoms in the representation) and impose a virtual cyclic ordering solely for the purpose of defining the guidance objective. This ordering has no chemical meaning and is used only to provide a differentiable proxy for cyclic topology.

Let
\[
X = (x_0, \dots, x_{n-1}) \in \mathbb{R}^{n \times 3}
\]
denote the coordinates of the selected atoms. We define two complementary geometric constraints: a local adjacency constraint encouraging ring closure, and a global separation constraint preventing degenerate, collapsed configurations.

\paragraph{Local cyclic adjacency.}
We enforce approximate uniformity of distances between consecutive atoms in the imposed cyclic order. Defining
\[
\ell_i = \lVert x_i - x_{i+1} \rVert_2, \qquad x_n \equiv x_0,
\]
we penalize distances that fall outside a target interval $[\ell_{\min}, \ell_{\max}]$. This is implemented using a hinge-squared penalty:
\[
\mathcal{L}_{\text{edge}}(X)
=
\sum_{i=0}^{n-1}
\Bigl[
\max(0, \ell_{\min} - \ell_i)^2
+
\max(0, \ell_i - \ell_{\max})^2
\Bigr].
\]
In our experiments, we set $\ell_{\min} = 1.0$ \AA \ and $\ell_{\max} = 2.0$ \AA. This term encourages the selected atoms to form a closed loop with approximately consistent edge lengths, without requiring exact bond formation or atom ordering.

\paragraph{Global anti-collapse constraint.}
To avoid degenerate solutions in which the loop collapses or folds onto itself, we introduce a coarse global constraint based on opposite atoms in the imposed cycle. For even $n$, we define
\[
o_i = \lVert x_i - x_{i + n/2} \rVert_2,
\]
and penalize opposite pairs that are closer than a minimum separation $o_{\min}$:
\[
\mathcal{L}_{\text{opp}}(X)
=
\sum_{i=0}^{n-1}
\max(0, o_{\min} - o_i)^2.
\]
We set $o_{\min} = 3.0$ \AA \ in all experiments. This term encourages a non-degenerate ring geometry with a finite diameter, discouraging self-intersections and collapsed configurations.

\paragraph{Overall guidance objective.}
The guidance is then given by the sum of $\mathcal{L}_{\text{edge}}$ and $\mathcal{L}_{\text{opp}}$.
While this objective does not explicitly compute topological invariants such as persistent homology, it provides a computationally inexpensive and fully differentiable surrogate that promotes cycle-like geometries. In particular, it biases generation toward configurations exhibiting both local loop closure and global ring structure, despite the absence of any predefined atom ordering or explicit topological constraints.

\subsection{Solid-torus Noise Initialization}
\label{app:torus}

We replace standard Gaussian initialization with a prior distribution supported on a solid torus, defined by major radius $R=2.5$\,\AA\ and tube radius $r_{\text{max}}=1.5$\,\AA. The coordinates $\mathbf{x} \in \mathbb{R}^3$ are sampled according to:
\begin{equation*}
\begin{aligned}
    x &= (R + \rho \cos \phi) \cos \theta, \\
    y &= (R + \rho \cos \phi) \sin \theta, \\
    z &= \rho \sin \phi,
\end{aligned}
\end{equation*}
where $\theta, \phi \sim \mathcal{U}[0, 2\pi)$ are uniform angular coordinates. The radial offset $\rho$ determines the deviation from the torus core and is sampled as $\rho = r_{\text{max}} \cdot \xi$, where $\xi \sim \text{Beta}(0.5, 2.0)$. This choice of hyperparameters ($\alpha < 1$) concentrates the probability mass near the central filament ($\rho \to 0$) while maintaining support throughout the tube.

\paragraph{Prior mismatch.} Note that this initialization strategy fundamentally violates the stationary assumptions of the standard diffusion reverse process \cite{song2020score}, which is trained to map from an isotropic Gaussian prior. The negative results of Tables~\ref{tab:all_models_unconditional} and \ref{tab:all_models_conditional} are therefore not surprising.

\clearpage
\section{Additional Visualizations}
\subsection{Unconditional Generation}
\label{app:vis_unc}
\Cref{fig:vis_unc} displays molecules generated by MolDiff+\name, for the parameter values described in Section~\ref{sec:unconditional}.
\begin{figure}[htb]
\centering
\setlength{\tabcolsep}{3pt}
\renewcommand{\arraystretch}{0} 

\newcommand{\cellW}{0.23\textwidth}
\newcommand{\cellH}{0.16\textheight} 

\begin{tabular}{@{} *{4}{>{\centering\arraybackslash}p{\cellW}} @{}}
\adjustbox{max width=\cellW, max height=\cellH, valign=c}{\includegraphics{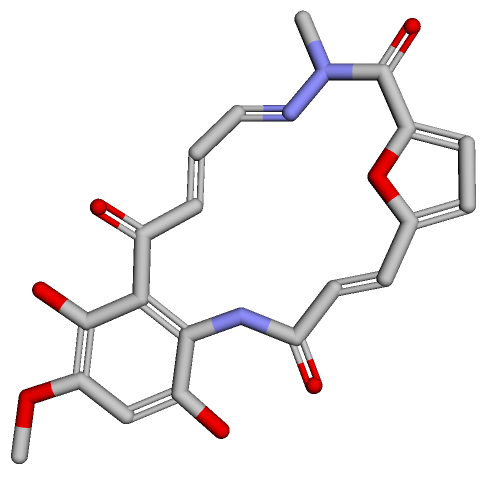}} &
\adjustbox{max width=\cellW, max height=\cellH, valign=c}{\includegraphics{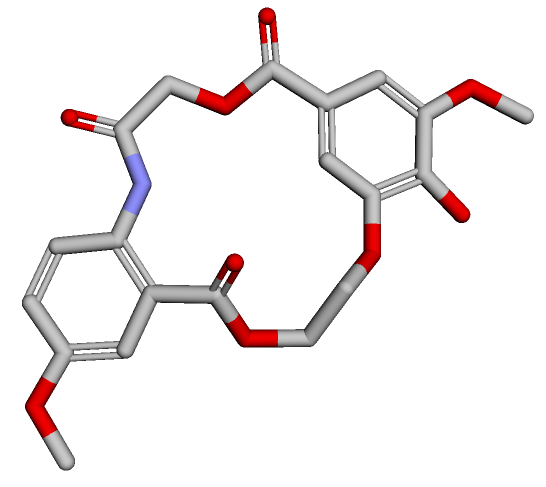}} &
\adjustbox{max width=\cellW, max height=\cellH, valign=c}{\includegraphics{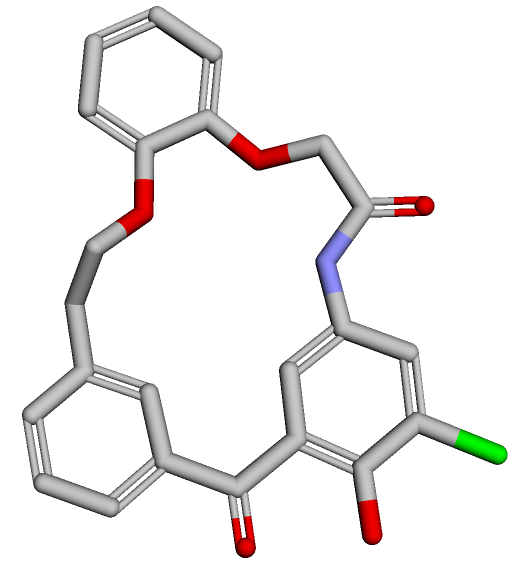}} &
\adjustbox{max width=\cellW, max height=\cellH, valign=c}{\includegraphics{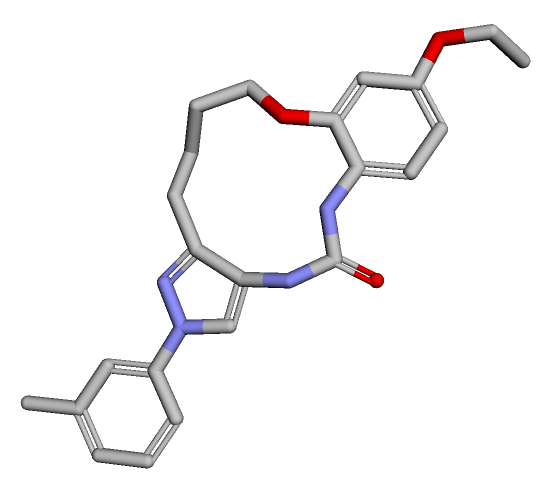}} \\
\adjustbox{max width=\cellW, max height=\cellH, valign=c}{\includegraphics{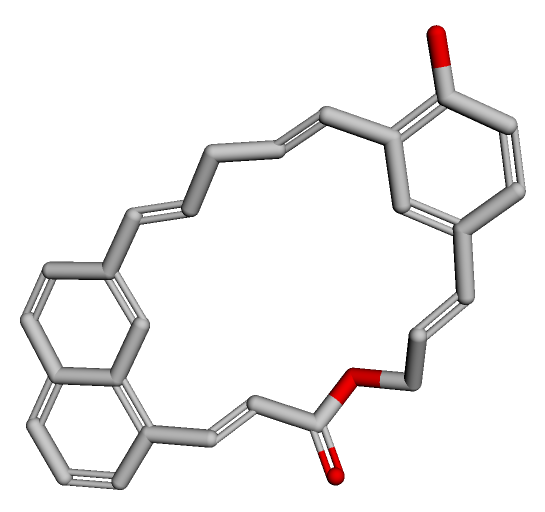}} &
\adjustbox{max width=\cellW, max height=\cellH, valign=c}{\includegraphics{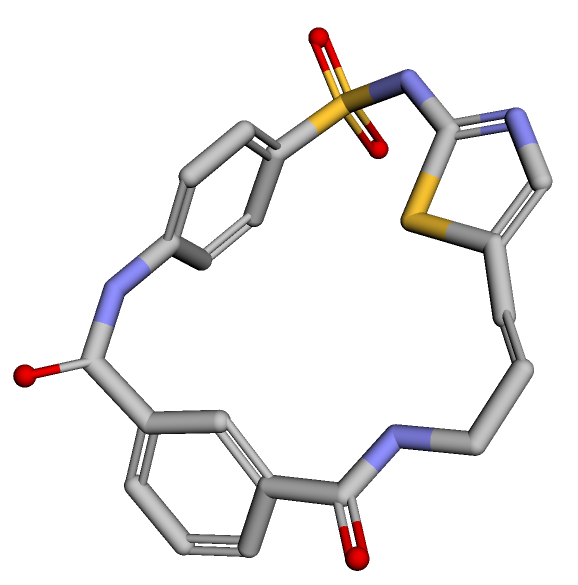}} &
\adjustbox{max width=\cellW, max height=\cellH, valign=c}{\includegraphics{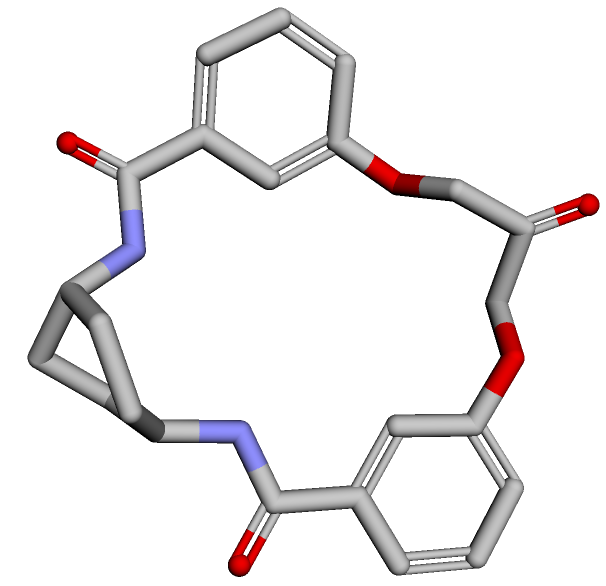}} &
\adjustbox{max width=\cellW, max height=\cellH, valign=c}{\includegraphics{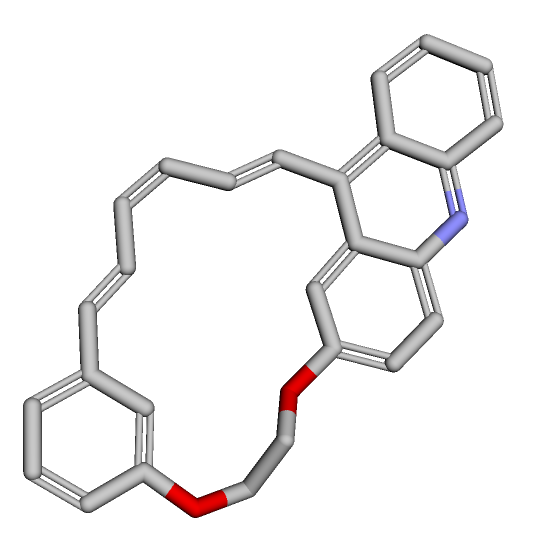}} \\
\adjustbox{max width=\cellW, max height=\cellH, valign=c}{\includegraphics{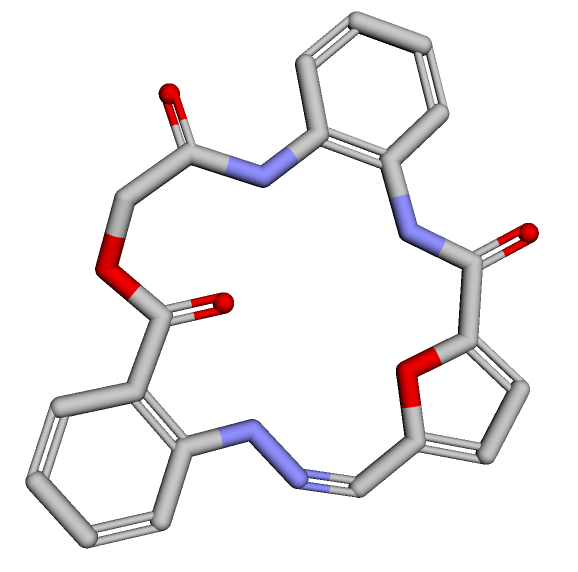}} &
\adjustbox{max width=\cellW, max height=\cellH, valign=c}{\includegraphics{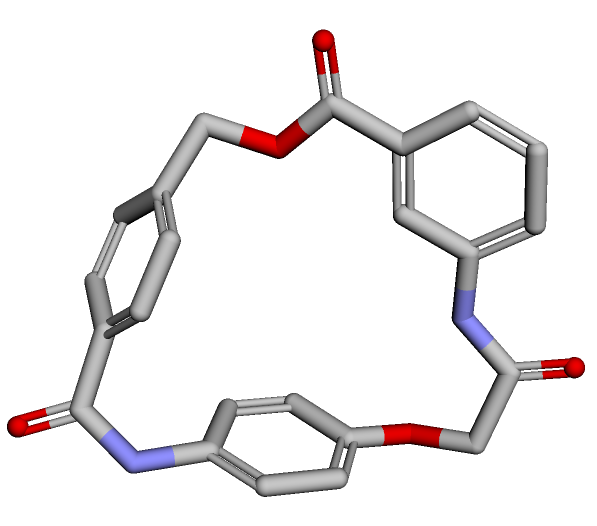}} &
\adjustbox{max width=\cellW, max height=\cellH, valign=c}{\includegraphics{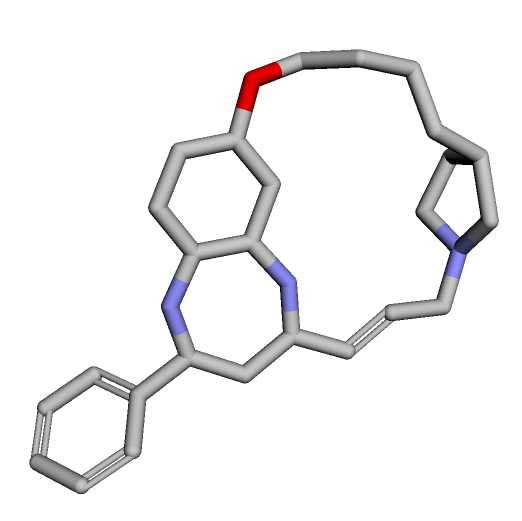}} &
\adjustbox{max width=\cellW, max height=\cellH, valign=c}{\includegraphics{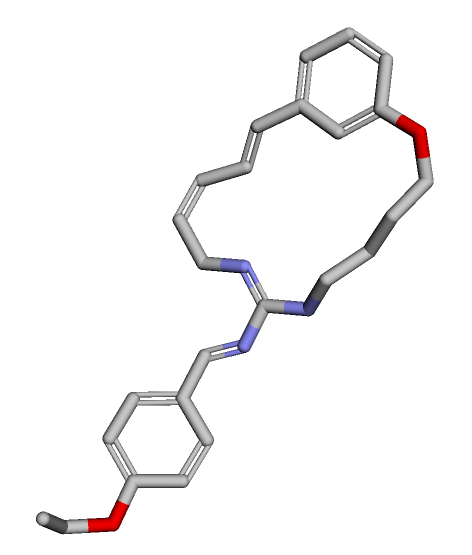}} \\
\adjustbox{max width=\cellW, max height=\cellH, valign=c}{\includegraphics{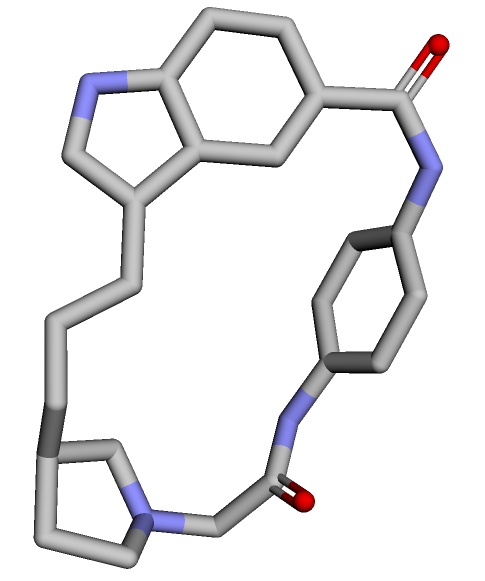}} &
\adjustbox{max width=\cellW, max height=\cellH, valign=c}{\includegraphics{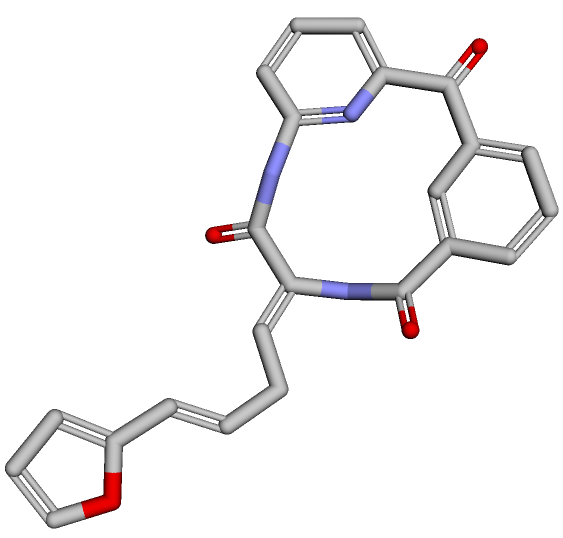}} &
\adjustbox{max width=\cellW, max height=\cellH, valign=c}{\includegraphics{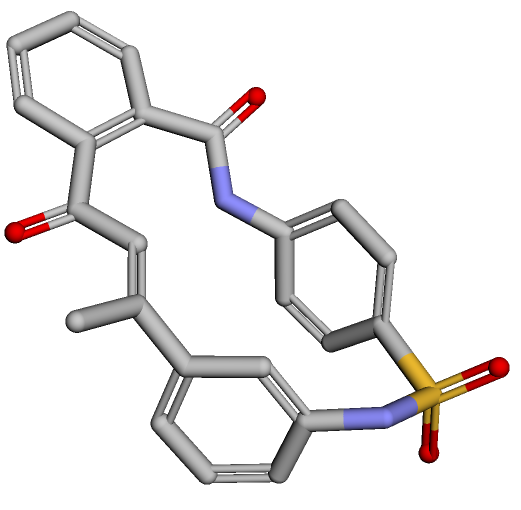}} &
\adjustbox{max width=\cellW, max height=\cellH, valign=c}{\includegraphics{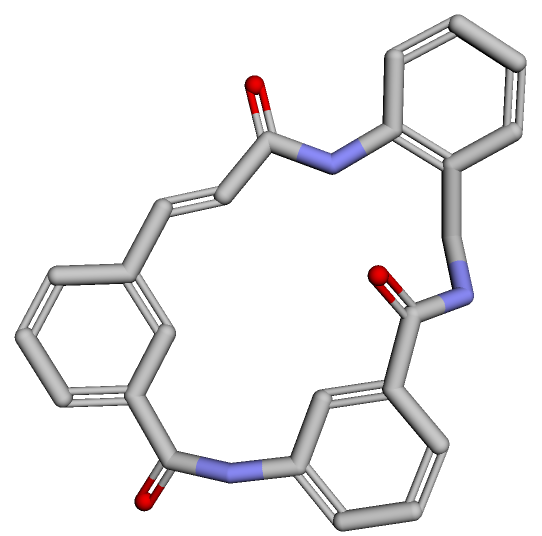}} \\

\end{tabular}

\caption{\textbf{Visualizations of generated molecules in the unconditional setting.}}
\label{fig:vis_unc}
\end{figure}

\clearpage
\subsection{Conditional Generation}
\label{app:vis_cond}
\Cref{fig:vis_cond} displays molecules generated by MolSnapper+\name, for the parameter values described in Section~\ref{sec:conditional}.

\begin{figure}[ht]
    \centering
    
    \def\boxwidth{5.2cm} 
    \def\boxheight{3cm} 
    \def\gapsize{0.4cm}

    \newcommand{\cropimg}[1]{%
        \sbox0{\includegraphics{#1}}%
        \pgfmathsetmacro{\imgRatio}{\wd0/\ht0}%
        \pgfmathsetmacro{\boxRatio}{\boxwidth/\boxheight}%
        \begin{tikzpicture}[baseline={(0,0)}]
            \clip (0,0) rectangle (\boxwidth, \boxheight);
            
            \node[anchor=center, inner sep=0pt, outer sep=0pt] at (0.5*\boxwidth, 0.5*\boxheight) {%
                \ifdim \imgRatio pt > \boxRatio pt
                    \includegraphics[height=\boxheight]{#1}%
                \else
                    \includegraphics[width=\boxwidth]{#1}%
                \fi
            };
            
            \draw[line width=1pt] (0,0) rectangle (\boxwidth, \boxheight);
        \end{tikzpicture}%
    }

    \cropimg{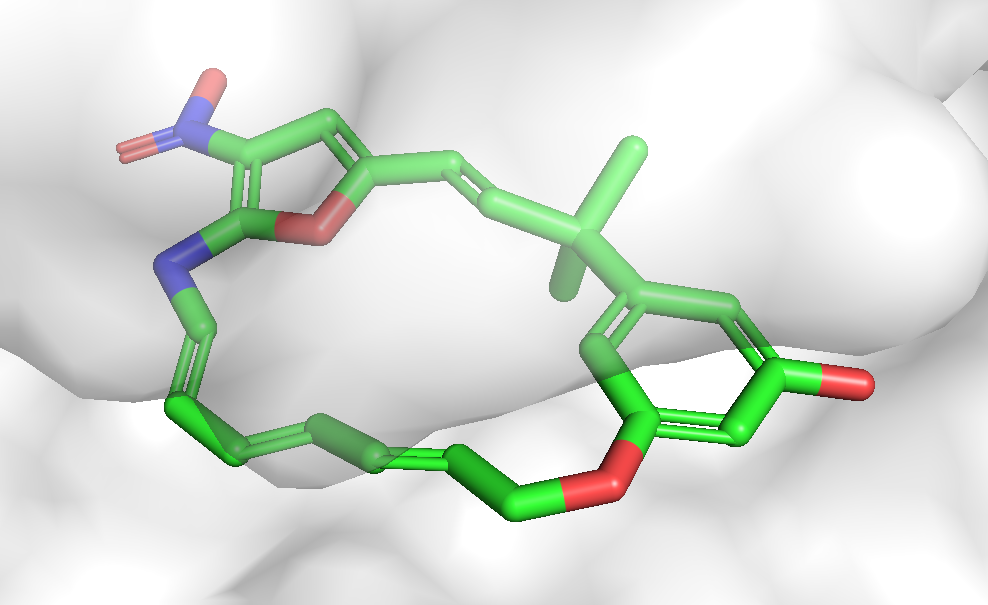}\hspace{\gapsize}%
    \cropimg{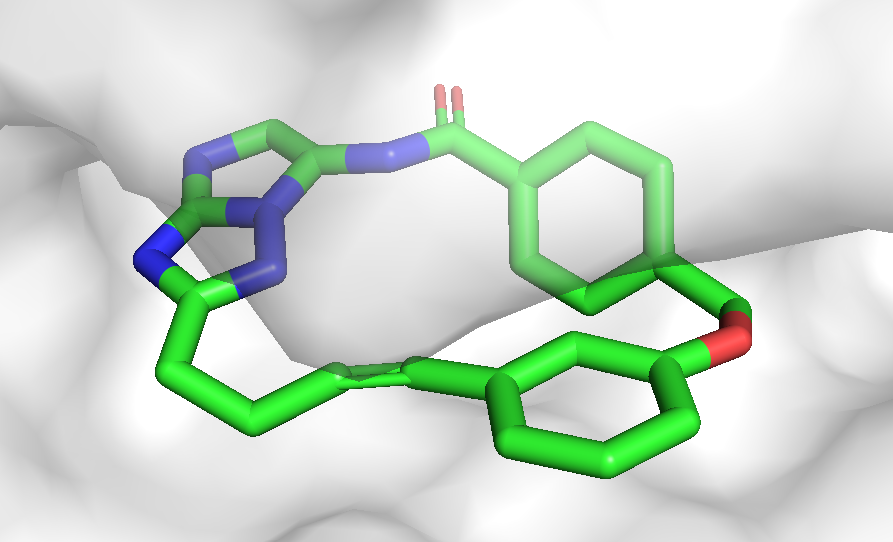}\hspace{\gapsize}%
    \cropimg{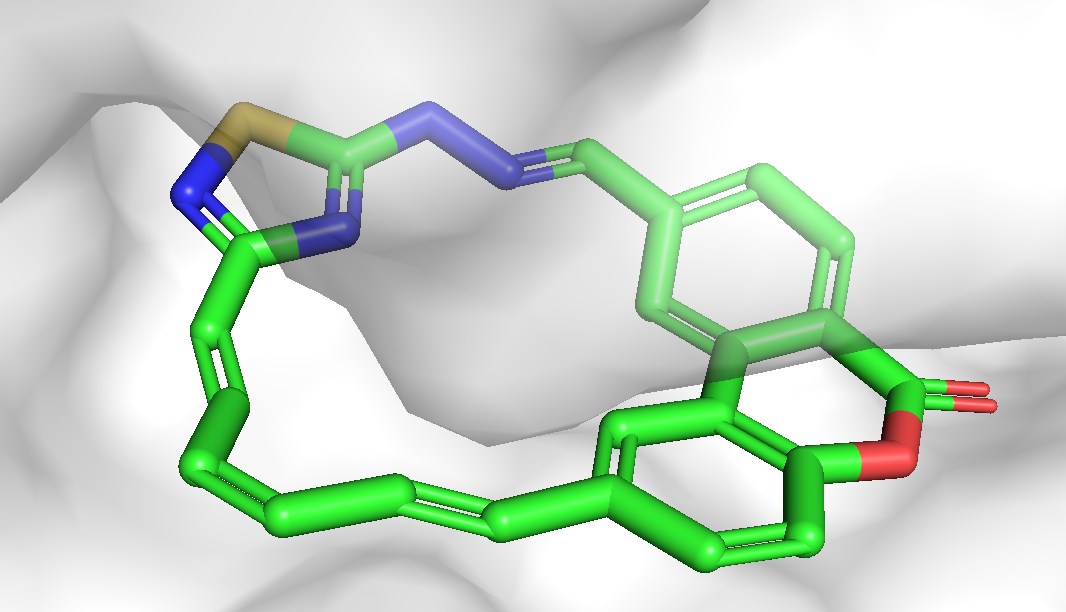}
    
    \par\vspace{\gapsize}

     \cropimg{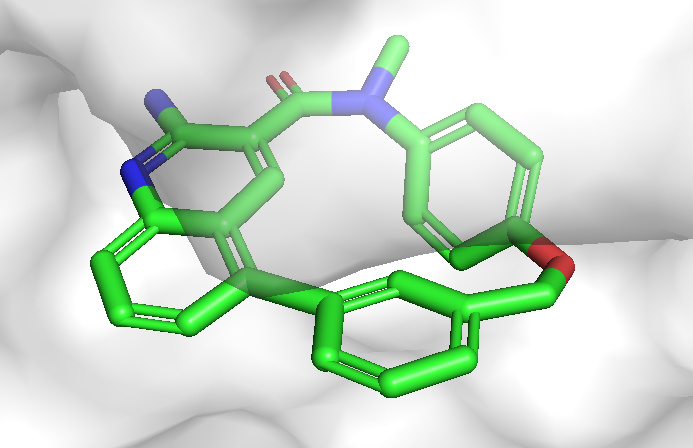}\hspace{\gapsize}%
    \cropimg{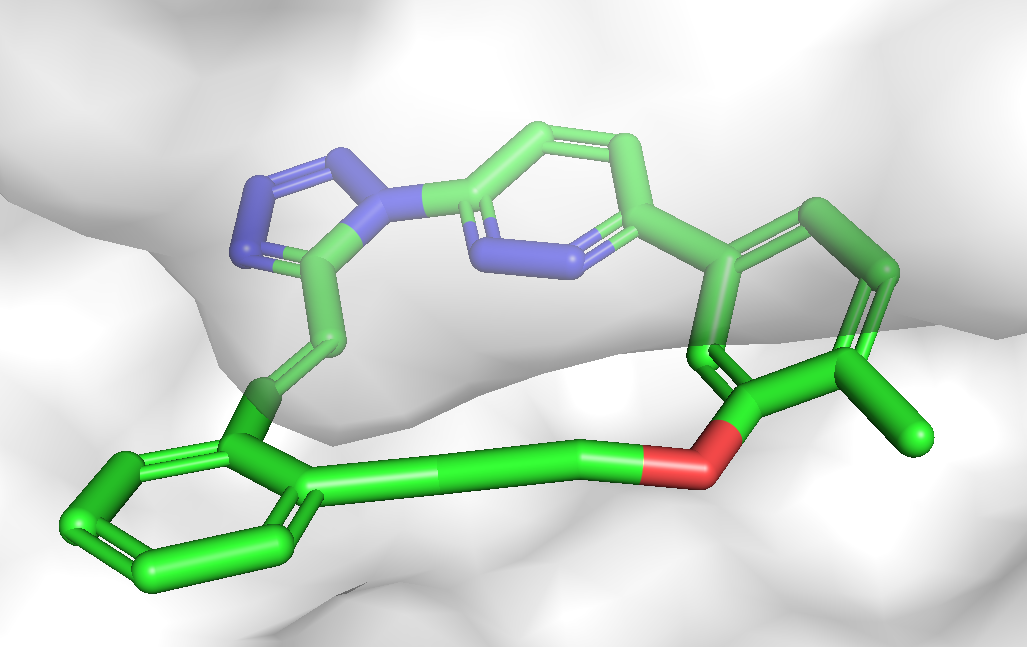}\hspace{\gapsize}%
    \cropimg{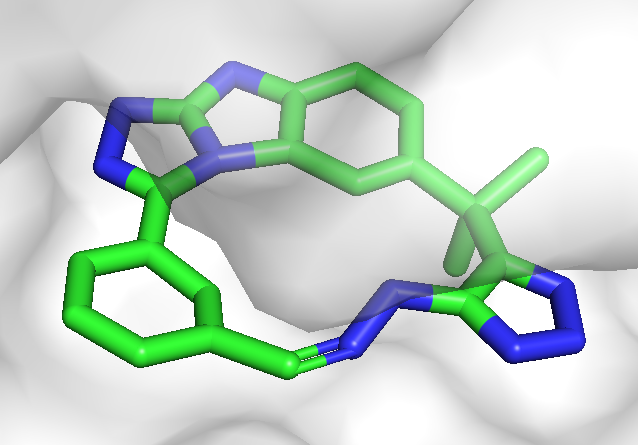}
    
    \par\vspace{\gapsize}
     \cropimg{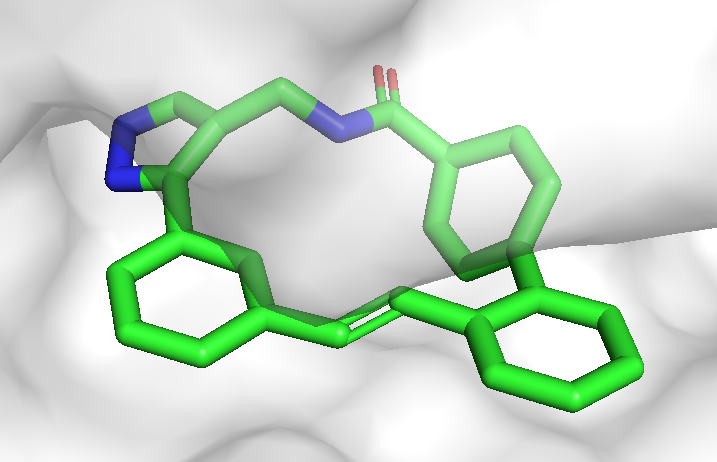}\hspace{\gapsize}%
    \cropimg{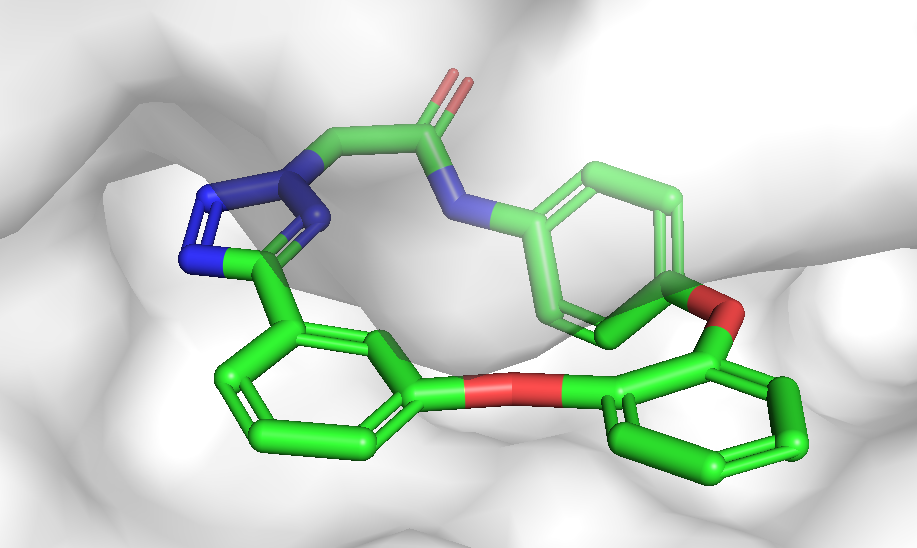}\hspace{\gapsize}%
    \cropimg{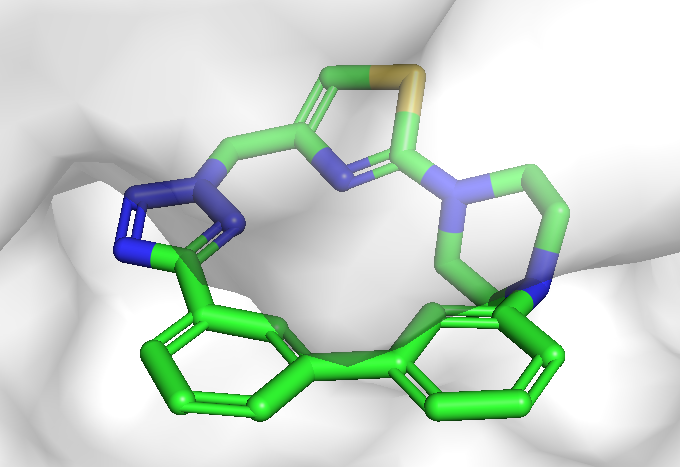}
    
    \par\vspace{\gapsize}
     \cropimg{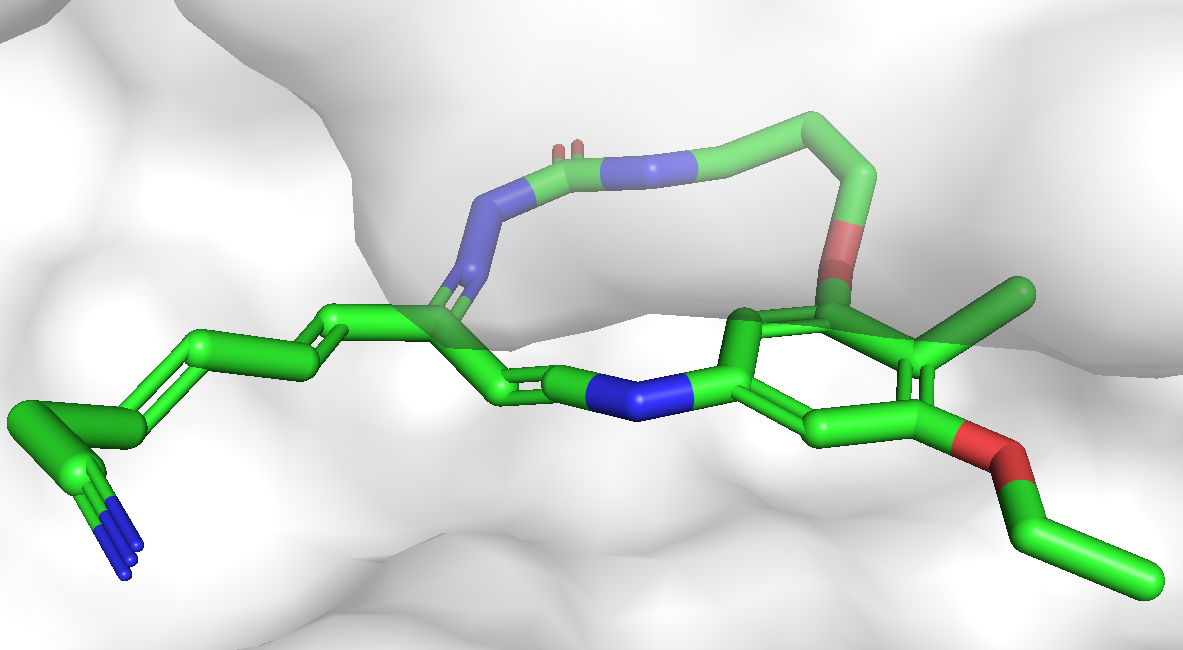}\hspace{\gapsize}%
    \cropimg{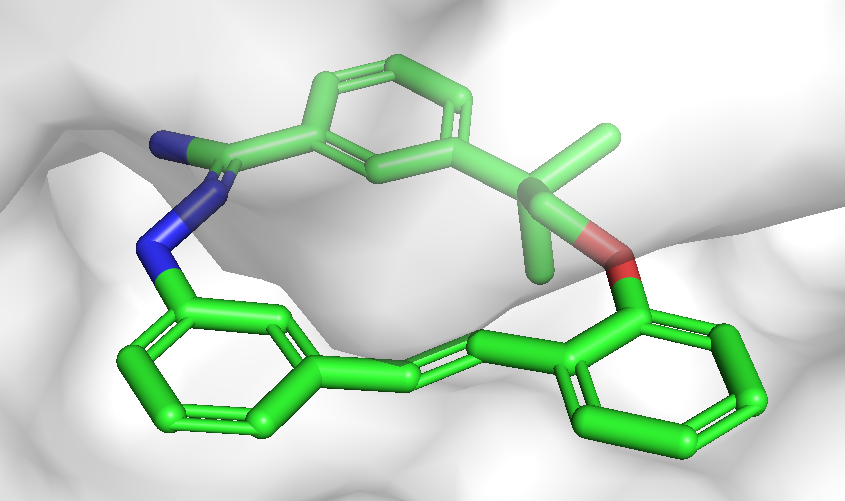}\hspace{\gapsize}%
    \cropimg{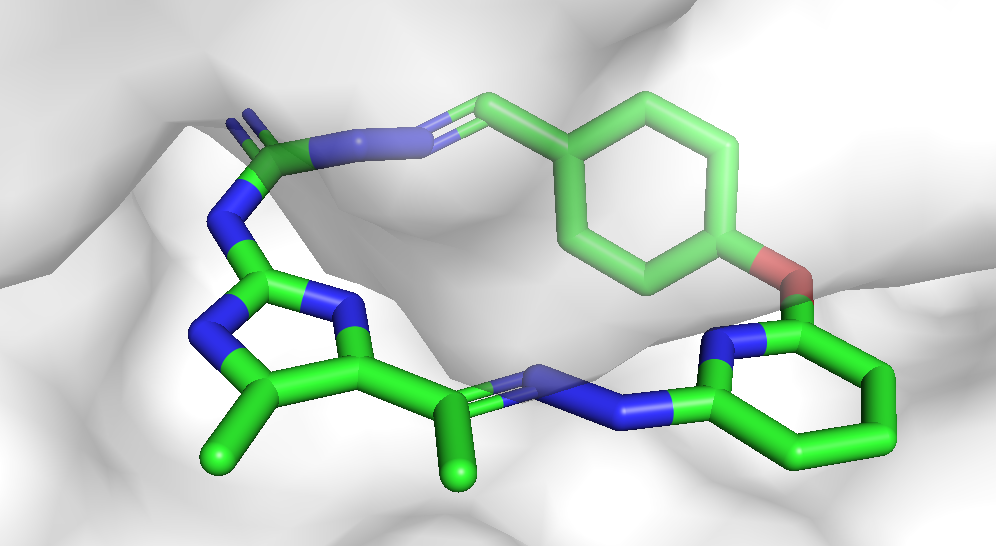}

    \caption{\textbf{Visualizations of generated molecules in the protein-conditioned setting.} The top parts of molecules are less visible because they are obstructed by a semi-transparent protein pocket.}
    \label{fig:vis_cond}

\end{figure}
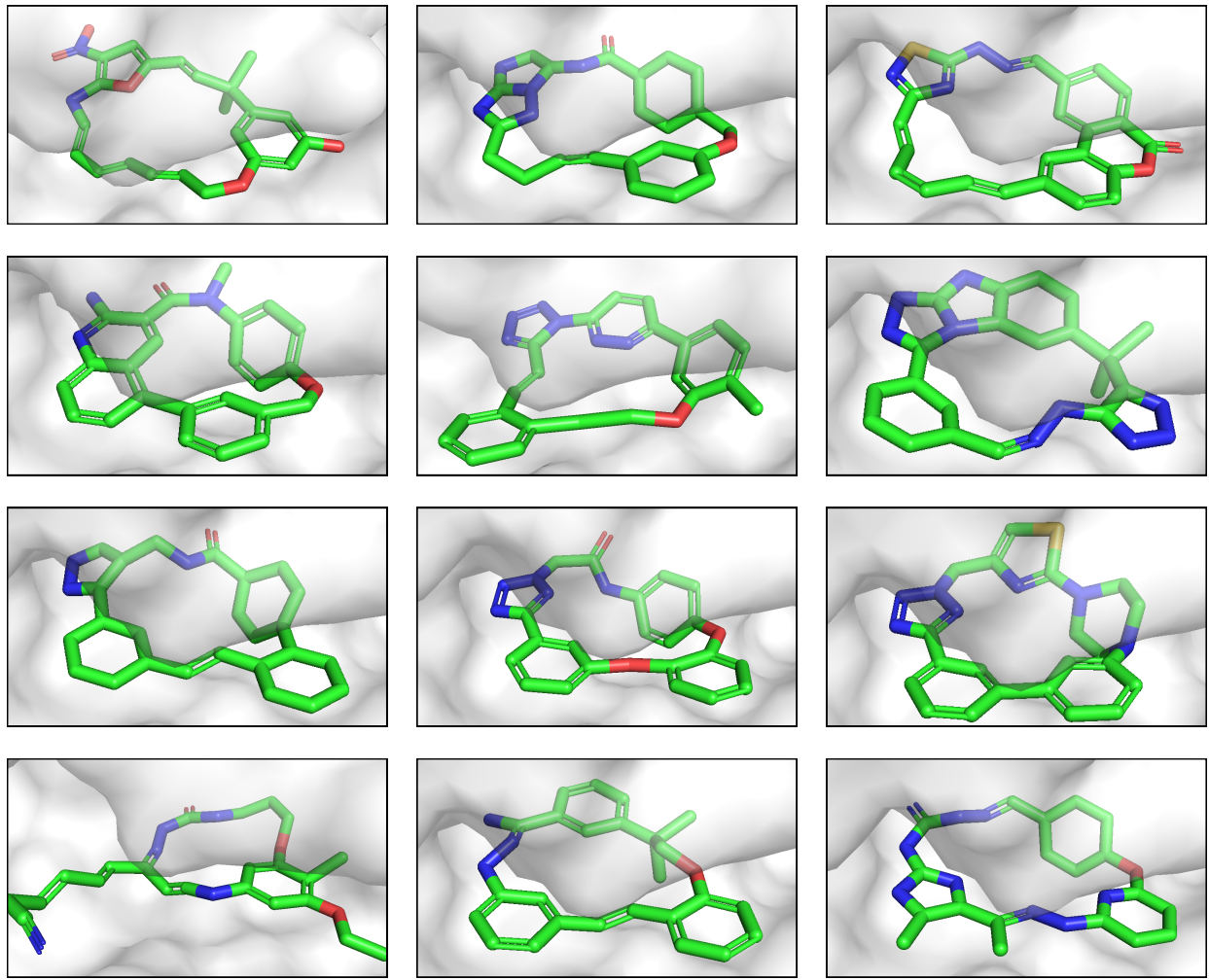
\clearpage

\section{Additional Results}
\label{app:results}

\subsection{Standard Deviations}
\label{app:stds}

In Table~\ref{tab:std_macro} we report standard deviations for our experiments. The metric results were computed on 5 sets of 1000 molecules.

\begin{table}[h!]
\begin{center}
\begin{small}
\caption{\textbf{Standard deviations of unconditional and protein-conditioned macrocycle generation performance.} The last 6 values do not appear for MolDiff as the metrics are defined only in the protein-conditioning setting. Variances for unguided MolDiff and MolSnapper are not reported as generating multiple sets of 1000 macrocycles proved computationally prohibitive due to the very low baseline macrocycle success rates.}
\begin{tabular}{lcc}
\toprule
    Metrics ($\uparrow$; [0-1])  & \makecell[c]{MolDiff\\+\name}& \makecell[c]{MolSnapper\\+\name}\\ \hline
    Validity     &$\pm{.003}$&$\pm{.008}$\\ 
    Connectivity  & $\pm{.001}$ &$\pm{.000}$\\
    Successfulness  & $\pm{.003}$ &$\pm{.008}$\\
\hline
    Out of successful:&   &\\
        \hspace{0.2em}Macrocycles& $\pm{.002}$&$\pm{.003}$\\
\hline
    Out of macrocycles:& & \\
        \hspace{0.2em}Diversity&$\pm{.005}$&$\pm{.001}$\\
        \hspace{0.2em}Novelty& $\pm{.000}$&$\pm{.001}$\\
        \hspace{0.2em}Uniqueness& $\pm{.001}$&$\pm{.001}$\\
        \hspace{0.2em}All PoseBusters tests&  -&$\pm{.015}$\\
            \hspace{1.2em}Ligand PoseBusters& $\pm{.013}$&$\pm{.018}$\\
                \hspace{2em}Bond lengths    & $\pm{.004}$&$\pm{.010}$\\ 
                \hspace{2em}Bond angles    & $\pm{.002}$&$\pm{.012}$\\ 
                \hspace{2em}Internal steric clash    & $\pm{.013}$&$\pm{.017}$\\ 
                \hspace{2em}Aromatic ring flatness& $\pm{.001}$ &$\pm{.002}$\\ 
                \hspace{2em}Non-ar. ring non-flatness & $\pm{.001}$&$\pm{.001}$\\ 
                \hspace{2em}Double bond flatness& $\pm{.004}$&$\pm{.002}$\\ 
                \hspace{2em}Internal energy    & $\pm{.004}$&$\pm{.010}$\\ 
            \hspace{1.2em}Protein PoseBusters&  - &$\pm{.008}$\\
                \hspace{2em}Protein-ligand max. distance &  -&$\pm{.000}$\\
                \hspace{2em}Min. distance to protein& -&$\pm{.008}$\\
                \hspace{2em}Volume overlap with protein&  -&$\pm{.000}$\\
        \hspace{0.2em}Pharmacophore satisfaction &  -&$\pm{.014}$ \\
        \hspace{0.2em}Macrocycle Lipinski & -&$\pm{.017}$\\
\bottomrule

\end{tabular}
\label{tab:std_macro}
\end{small}
\end{center}
\vskip -0.1in
\end{table}
\subsection{Ablations of Guidance Terms}
\label{app:ablations}
We perform ablations to show the influence of $H_1$ birth optimization (closing the cycle) and $H_0$ death optimization (keeping the point cloud connected) on macrocycle generation, both in the unconditional and conditional setting. We keep the $H_1$ death optimization unchanged, as it is the core of the method. We generate sets of 1000 molecules with MolDiff+\name and MolSnapper+\name, and report results in Table~\ref{tab:ablations_macro} and Table~\ref{tab:ablations_prot} respectively.  

We observe a drop in macrocycle generation when the birth of the $H_1$ component is not constrained as this may lead to the creation of non-closed cycles. On the other hand, the lack of $H_0$ death term, while not detrimental to macrocycle generation performance or macrocycle quality, results in a drop in connected molecules and hence a drop in successfulness (Table~\ref{tab:ablations_macro}).

In the protein conditioning setting, the lack of $H_1$ birth term produces a similar drop in performance. However, $H_0$ death optimization can be safely removed, probably because of the confined volume of the protein pocket (Table~\ref{tab:ablations_prot}).

\begin{table}[h!]
\begin{center}
\begin{small}
\caption{\textbf{Ablations of different guidance terms in the unconditional setting.} The metrics most affected by the ablation are in bold. Some metrics are not reported because not enough successful macrocycles were produced. Results were obtained from sets of 1000 molecules.}

\begin{tabular}{lcccc}
\toprule
    Metrics ($\uparrow$; [0-1])  & \makecell[c]{MolDiff\\+\name\\(ours)} &\makecell[c]{\makecell[c]{No $H_1$ birth}}& \makecell[c]{\makecell[c]{No $H_0$}} &\makecell[c]{No $H_1$ birth and no $H_0$}\\ \hline
    Validity     &  0.989&0.990&0.990&0.990\\ 
    Connectivity  &  0.999&0.954&\textbf{0.863}&\textbf{0.022}\\
    Successfulness  &  0.988&0.944&0.854&0.022\\
\hline
    Out of successful:&    && &\\
        \rowcolor[gray]{.9}\hspace{1em}\textbf{Macrocycles}&  0.997&\textbf{0.075}&0.995&\textbf{0.091}\\
\hline
    Out of macrocycles:&  &&  &\\
        \hspace{1em}Diversity& 0.771&0.735&0.774&-\\
        \hspace{1em}Novelty&  1.000&1.000&1.000&-\\
        \hspace{1em}Uniqueness&  1.000&1.000&1.000&-\\
        \rowcolor{gray!20}\hspace{1em}\textbf{All PoseBusters tests}&  0.805&0.732&0.840&-\\ 
            \hspace{2em}Bond lengths    &  0.990&1.000&0.994&-\\ 
            \hspace{2em}Bond angles    &  0.987&0.761&0.985&-\\ 
            \hspace{2em}Internal steric clash    &  0.844&0.972&0.893&-\\ 
            \hspace{2em}Aromatic ring flatness&  0.999&1.000&1.000&-\\ 
            \hspace{2em}Non-ar. ring non-flatness&  0.999&1.000&1.000&-\\ 
            \hspace{2em}Double bond flatness&  0.989&1.000&0.968&-\\ 
            \hspace{2em}Internal energy    &  0.984&1.000&0.991&-\\ 
\bottomrule
    
\end{tabular}
\label{tab:ablations_macro}
\end{small}
\end{center}
\vskip -0.1in
\end{table}

\begin{table}[h!]
\begin{center}
\begin{small}
\caption{\textbf{Ablations of different guidance terms in the conditional setting.} The metrics most affected by the ablation are in bold. Some metrics are not reported because not enough successful macrocycles were produced. Results were obtained from sets of 1000 molecules.}
\begin{tabular}{lcccc}
\toprule
    Metrics ($\uparrow$; [0-1])  & \makecell[c]{MolSnapper\\+\name\\(ours)} & \makecell[c]{\makecell[c]{No $H_1$ birth}}& \makecell[c]{\makecell[c]{No $H_0$}} &\makecell[c]{No $H_1$ birth and no $H_0$}\\ \hline
    Validity     &0.925 & 0.903& 0.921&0.929\\ 
    Connectivity  &1.000 & 1.000& 1.000&1.000\\
    Successfulness  &0.925 & 0.903& 0.921&0.929\\
\hline
    Out of successful:& & & &\\
        \rowcolor[gray]{.9}\hspace{0.2em}\textbf{Macrocycles}&0.995 & \textbf{0.078}& 0.989&\textbf{0.016}\\
\hline
    Out of macrocycles:&  & & &\\
        \hspace{0.2em}Diversity&0.712 & 0.734& 0.711&-\\
        \hspace{0.2em}Novelty&1.000 & 1.000& 1.000&-\\
        \hspace{0.2em}Uniqueness&1.000 & 1.000& 0.999&-\\
        \rowcolor{gray!20}\hspace{0.2em}\textbf{All PoseBusters tests}&0.575 & 0.414& 0.591&-\\
            \hspace{1.2em}Ligand PoseBusters&0.626 & 0.471& 0.654&-\\
                \hspace{2em}Bond lengths    &0.860 & 0.829& 0.923&-\\ 
                \hspace{2em}Bond angles    &0.888 & 0.643& 0.907&-\\ 
                \hspace{2em}Internal steric clash    &0.854 & 0.914& 0.854&-\\ 
                \hspace{2em}Aromatic ring flatness&0.992 & 0.986& 0.993&-\\ 
                \hspace{2em}Non-ar. ring non-flatness &0.999 & 1.000& 1.000&-\\ 
                \hspace{2em}Double bond flatness&0.993 & 0.986& 0.995&-\\ 
                \hspace{2em}Internal energy    &0.921 & 0.929& 0.926&-\\ 
            \hspace{1.2em}Protein PoseBusters&0.911 & 0.843& 0.902&-\\
                \hspace{2em}Protein-ligand max. distance &1.000 & 1.000& 1.000&-\\
                \hspace{2em}Min. distance to protein&0.907 & 0.843& 0.902&-\\
                \hspace{2em}Volume overlap with protein&1.000 & 1.000& 1.000&-\\
        \hspace{0.2em}Pharmacophore satisfaction &0.789  & 0.829& 0.804&-\\
        \hspace{0.2em}Macrocycle Lipinski &0.638  & 0.743& 0.580&-\\
\bottomrule

\end{tabular}
\label{tab:ablations_prot}
\end{small}
\end{center}
\vskip -0.1in
\end{table}

\subsection{Ablations of Squares}
\label{app:ablations_square}
We perform additional experiments with and without squaring the key equations in the topological loss. The results in Table~\ref{tab:ablations_square} and Table~\ref{tab:ablations_square_prot} show that adding a square to $H_1$ birth decreases the rate of macrocycle generation, as well as their PoseBusters performance, both in the unconditional and conditional settings.

Omitting the $H_1$ death square results in a decrease in molecule validity and the PoseBusters checks performance for macrocycles. Finally, the lack of $H_0$ death square seems the least detrimental to model performance, and even improves the performance for some metrics. This term modification can therefore act as a good starting point for adapting \name for specific tasks.

\begin{table}[h!]
\begin{center}
\begin{small}
\caption{\textbf{Influence of the square on each guidance term in the unconditional setting.} The metrics most affected by the ablation are in bold. Results were obtained from sets of 1000 molecules. }

\begin{tabular}{lcccc}
\toprule
    Metrics ($\uparrow$; [0-1])   & \makecell[c]{MolDiff\\+\name\\(ours)} & \makecell[c]{\makecell[c]{+$H_1$ birth square}}& \makecell[c]{No $H_1$ death square} & \makecell[c]{\makecell[c]{No $H_0$ death square}}\\ \hline
    Validity      & 0.989& 0.981&\textbf{0.924}&0.991\\ 
    Connectivity   & 0.999& 0.985&1.000&1.000\\
    Successfulness   & 0.988& 0.966&0.924&0.991\\
\hline
    Out of successful: & &   & &\\
        \rowcolor[gray]{.9}\hspace{1em}\textbf{Macrocycles} & 0.997& \textbf{0.633}&0.998&1.000\\
\hline
    Out of macrocycles: & & &  &\\
        \hspace{1em}Diversity & 0.771&0.789&0.782&0.752\\
        \hspace{1em}Novelty & 1.000& 1.000&1.000&1.000\\
        \hspace{1em}Uniqueness & 1.000& 1.000&1.000&1.000\\
        \rowcolor{gray!20}\hspace{1em}\textbf{All PoseBusters tests} & 0.805& \textbf{0.628}&\textbf{0.711}&0.791\\ 
            \hspace{2em}Bond lengths     & 0.990& 0.964&0.935&0.989\\ 
            \hspace{2em}Bond angles     & 0.987& 0.702&0.938&0.987\\ 
            \hspace{2em}Internal steric clash     & 0.844& 0.913&0.837&0.815\\ 
            \hspace{2em}Aromatic ring flatness & 0.999& 0.998&0.999&0.996\\ 
            \hspace{2em}Non-ar. ring non-flatness & 0.999& 0.998&0.999&0.996\\ 
            \hspace{2em}Double bond flatness & 0.989& 0.987&0.977&0.987\\ 
            \hspace{2em}Internal energy     & 0.984& 0.959&0.953&0.987\\ 
\bottomrule
    
\end{tabular}
\label{tab:ablations_square}
\end{small}
\end{center}
\vskip -0.1in
\end{table}

\begin{table}[h!]
\begin{center}
\begin{small}
\caption{\textbf{Influence of the square on each guidance term in the conditional setting.} The metrics most affected by the ablation are in bold. Results were obtained from sets of 1000 molecules.}
\begin{tabular}{lcccc}
\toprule
    Metrics ($\uparrow$; [0-1])  & \makecell[c]{MolSnapper\\+\name\\(ours)}  & \makecell[c]{\makecell[c]{+$H_1$ birth square}}& \makecell[c]{No $H_1$ death square} &\makecell[c]{\makecell[c]{No $H_0$ death square}}\\ \hline
    Validity     &0.925 & 0.914& \textbf{0.883}&0.915\\ 
    Connectivity  &1.000 & 1.000& 1.000&1.000\\
    Successfulness  &0.925 & 0.914& 0.883&0.915\\
\hline
    Out of successful:& & & &\\
        \rowcolor[gray]{.9}\hspace{0.2em}\textbf{Macrocycles}&0.995 & \textbf{0.521}& 0.998&0.991\\
\hline
    Out of macrocycles:&  & & &\\
        \hspace{0.2em}Diversity&0.712 & 0.711& 0.717&0.703\\
        \hspace{0.2em}Novelty&1.000 & 1.000& 1.000&1.000\\
        \hspace{0.2em}Uniqueness&1.000 & 1.000& 1.000&1.000\\
        \rowcolor{gray!20}\hspace{0.2em}\textbf{All PoseBusters tests}&0.575 & \textbf{0.309}& 0.529&0.572\\
            \hspace{1.2em}Ligand PoseBusters&0.626 & 0.332& 0.580&0.607\\
                \hspace{2em}Bond lengths    &0.860 & 0.800& 0.815&0.869\\ 
                \hspace{2em}Bond angles    &0.888 & 0.502& 0.879&0.905\\ 
                \hspace{2em}Internal steric clash    &0.854 & 0.851& 0.846&0.847\\ 
                \hspace{2em}Aromatic ring flatness&0.992 & 0.987& 0.993&0.981\\ 
                \hspace{2em}Non-ar. ring non-flatness &0.999 & 1.000& 1.000&1.000\\ 
                \hspace{2em}Double bond flatness&0.993 & 0.989& 0.992&0.999\\ 
                \hspace{2em}Internal energy    &0.921 & 0.905& 0.910&0.899\\ 
            \hspace{1.2em}Protein PoseBusters&0.911 & 0.891& 0.915&0.940\\
                \hspace{2em}Protein-ligand max. distance &1.000 & 1.000& 1.000&1.000\\
                \hspace{2em}Min. distance to protein&0.907 & 0.891& 0.915&0.940\\
                \hspace{2em}Volume overlap with protein&1.000 & 1.000& 0.999&1.000\\
        \hspace{0.2em}Pharmacophore satisfaction &0.789  & 0.824& 0.820&0.744\\
        \hspace{0.2em}Macrocycle Lipinski &0.638  & 0.670& 0.640&0.630\\
\bottomrule

\end{tabular}
\label{tab:ablations_square_prot}
\end{small}
\end{center}
\vskip -0.1in
\end{table}

\subsection{Guidance Strength}
\label{app:strength}

Tables~\ref{tab:results_strength} and \ref{tab:strength_prot} demonstrate that the chosen guidance strength of 1 performs consistently well across both unconditional and conditional settings.

\begin{table}[htb]
\begin{center}
\begin{small}
\caption{\textbf{Influence of guidance strength on performance in the unconditional setting.} The original results from Table~\ref{tab:results_unconditional_macro} correspond to the strength of 1. Results were obtained from 1000 samples.}
\begin{tabular}{lccccc}
\toprule
    Metrics ($\uparrow$; [0-1])  &    0.1&0.5&\makecell[c]{MolDiff\\+\name\\1.0 (ours)}  & 1.5&2.0\\ \hline
    Validity     &   0.994&0.995&0.989& 0.991&0.992\\ 
    Connectivity  &   0.999&0.996&0.999& 0.996&0.988\\
    Successfulness  &   0.993&0.991&0.988& 0.987&0.980\\
\hline
    Out of successful:&   && & &\\
        \rowcolor[gray]{.9}\hspace{1em}\textbf{Macrocycles}&   0.769&0.967&0.997& 0.990&0.754\\
\hline
    Out of macrocycles:&    && & &\\
        \hspace{1em}Diversity&   0.752&0.763&0.771& 0.780&0.796\\
        \hspace{1em}Novelty&   1.000&1.000&1.000& 1.000&1.000\\
        \hspace{1em}Uniqueness&   1.000&0.999&1.000& 0.992&1.000\\
        \rowcolor{gray!20}\hspace{1em}\textbf{All PoseBusters tests}&   0.662&0.754&0.805& 0.870&0.548\\ 
            \hspace{2em}Bond lengths    &   0.986&0.996&0.990& 0.997&0.961\\ 
            \hspace{2em}Bond angles    &   0.808&0.947&0.987& 0.939&0.608\\ 
            \hspace{2em}Internal steric clash    &   0.823&0.813&0.844& 0.960&0.963\\ 
            \hspace{2em}Aromatic ring flatness&   0.992&0.997&0.999& 1.000&1.000\\ 
            \hspace{2em}Non-ar. ring non-flatness&   0.995&0.997&0.999& 0.999&1.000\\ 
            \hspace{2em}Double bond flatness&   0.986&0.990&0.989& 0.990&0.991\\ 
            \hspace{2em}Internal energy    &   0.976&0.977&0.984& 0.975&0.951\\ 
\bottomrule
    
\end{tabular}
\label{tab:results_strength}
\end{small}
\end{center}
\vskip -0.1in
\end{table}
\begin{table}[h!]
\begin{center}
\begin{small}
\caption{\textbf{Influence of guidance strength on performance in the conditional setting.} The original results from Table~\ref{tab:results_protein_macro} correspond to the strength of 1. Results were obtained from 1000 samples.}
\begin{tabular}{lccccc}
\toprule
    Metrics ($\uparrow$; [0-1])  &   0.1& 0.5&\makecell[c]{MolSnapper\\+\name\\1.0 (ours)} & 1.5& 2.0\\ \hline
    Validity     &  0.924&0.933 &0.925 & 0.876& 0.891\\ 
    Connectivity  &  1.000&1.000 &1.000 & 1.000& 1.000\\
    Successfulness  &  0.924&0.933 &0.925 & 0.876& 0.891\\
\hline
    Out of successful:&   & && & \\
        \rowcolor[gray]{.9}\hspace{0.2em}\textbf{Macrocycles}&  0.411&0.867 &0.995 & 0.995& 0.845\\
\hline
    Out of macrocycles:&    & && & \\
        \hspace{0.2em}Diversity&  0.655&0.695 &0.712 & 0.724& 0.724\\
        \hspace{0.2em}Novelty&  1.000&1.000 &1.000 & 1.000& 1.000\\
        \hspace{0.2em}Uniqueness&  1.000&1.000 &1.000 & 1.000& 1.000\\
        \rowcolor{gray!20}\hspace{0.2em}\textbf{All PoseBusters tests}&  0.384&0.507 &0.575 & 0.495& 0.154\\
            \hspace{1.2em}Ligand PoseBusters&  0.426&0.561 &0.626 & 0.561& 0.206\\
                \hspace{2em}Bond lengths    &  0.876&0.890 &0.860 & 0.854& 0.789\\ 
                \hspace{2em}Bond angles    &  0.658&0.791 &0.888 & 0.802& 0.311\\ 
                \hspace{2em}Internal steric clash    &  0.824&0.847 &0.854 & 0.874& 0.870\\ 
                \hspace{2em}Aromatic ring flatness&  0.995&0.994 &0.992 & 0.987& 0.991\\ 
                \hspace{2em}Non-ar. ring non-flatness &  1.000&0.999 &0.999 & 0.999& 0.999\\ 
                \hspace{2em}Double bond flatness&  0.982&0.984 &0.993 & 0.993& 0.985\\ 
                \hspace{2em}Internal energy    &  0.876&0.927 &0.921 & 0.911& 0.902\\ 
            \hspace{1.2em}Protein PoseBusters&  0.895&0.886 &0.911 & 0.857& 0.744\\
                \hspace{2em}Protein-ligand max. distance &  1.000&1.000 &1.000 & 1.000& 1.000\\
                \hspace{2em}Min. distance to protein&  0.895&0.886 &0.907 & 0.857& 0.744\\
                \hspace{2em}Volume overlap with protein&  1.000&1.000 &1.000 & 1.000& 0.993\\
        \hspace{0.2em}Pharmacophore satisfaction &  0.800& 0.808&0.789  & 0.828& 0.838\\
        \hspace{0.2em}Macrocycle Lipinski &  0.603&0.642 &0.638  & 0.673& 0.659\\
\bottomrule

\end{tabular}
\label{tab:strength_prot}
\end{small}
\end{center}
\vskip -0.1in
\end{table}

\clearpage
\subsection{Atom Type Distribution}
\label{app:atoms}

Table~\ref{tab:atoms} reports the distribution of atom types in the generated macrocycles. 

\begin{table}[h!]
\begin{center}
\begin{small}
\caption{\textbf{Atom distribution of generated macrocycles.} Represented as a fraction of all atoms.}

\begin{tabular}{lccccccc}
\toprule
Method & C & N & O & F & P & S & Cl \\
\midrule
MolDiff (no guid.)& 0.7978& 0.0798& 0.1089&  0.0011& 0.0000&  0.0113&  
    0.0011\\
+\name (ours) 
    & 0.8785& 0.0477& 0.0722&  0.0001& 0.0000&  0.0014&  
    0.0001\\
    \midrule
MolSnapper (no guid.) 
    &  0.7508&  0.1749&  0.0714&  0.0002&  0.0000&  0.0024&  
    0.0003\\
+\name (ours) 
    &  0.7540&  0.1921&  0.0483&  0.0000&  0.0000&  0.0051&  0.0005\\
\bottomrule
\end{tabular}

\label{tab:atoms}
\end{small}
\end{center}
\vskip -0.1in
\end{table}

\subsection{Ring Size Distribution}
\label{app:rings}

Table~\ref{tab:rings} reports the number of small rings present in the generated molecules. Generated macrocycles typically contain a few additional rings, indicating that they are not merely simple large cycles but exhibit nontrivial structural complexity. Moreover, \name reduces the number of 3- and 4-membered rings, which are the most strained and therefore undesirable, leading to chemically more favourable structures.

\begin{table}[h!]
\begin{center}
\begin{small}
\caption{\textbf{Ring size distribution in generated macrocycles.} Represented as the average number of cycles of a certain size per molecule.}

\begin{tabular}{lccccccc}
\toprule
Method & 3-cycles ($\downarrow$)& 4-cycles ($\downarrow$)& 5-cycles& 6-cycles& 7-cycles& 8-cycles& 9-cycles\\
\midrule
MolDiff (no guid.)& 0.054& 0.022& 0.482&  2.552& 0.154&  0.022&  
    1.885\\
+\name (ours) 
    & 0.065& 0.022& 0.449&  1.933& 0.033&  0.003&  
    1.229\\
    \midrule
MolSnapper (no guid.) 
    &  0.258&  0.082&  1.005&  2.919&  0.133&  0.009&  
    2.216\\
+\name (ours) 
    &  0.183&  0.077&  1.123&  2.127&  0.039&  0.001&  1.341\\
\bottomrule
\end{tabular}

\label{tab:rings}
\end{small}
\end{center}
\vskip -0.1in
\end{table}

\subsection{Runtime Comparison}
\label{app:complexity}

\subsubsection{Applying Guidance Every Few Steps}
\label{app:subsampling}

To reduce computational cost, the topological guidance can be applied only every few denoising iterations. Table~\ref{tab:results_protein_macro_every_bold} shows that applying \name once every up to 5 steps results in only a slight degradation of the performances, while substantially decreasing the overall computational overhead.

\begin{table}[h!]
\begin{center}
\begin{small}
\caption{\textbf{Performance of protein-conditioned macrocycle generation with guidance applied every few steps.} The first two columns correspond to Table~\ref{tab:results_protein_macro}. Values outperforming or matching the baseline MolSnapper metrics are in bold.}
\begin{tabular}{lcccccc}
\toprule
    Metrics ($\uparrow$; [0-1])  &   \makecell[c]{\makecell[c]{MolSnapper\\(no guid.)}}&Every 1&Every 2& Every 3& Every 4&Every 5\\ \hline
    Validity     &   0.858 &\textbf{0.925}&\textbf{0.930}&\textbf{0.942}& \textbf{0.914}&\textbf{0.940}\\ 
    Connectivity  &   1.000 &\textbf{1.000}&\textbf{1.000}&\textbf{1.000}& \textbf{1.000}&\textbf{1.000}\\
    Successfulness  &   0.858 &\textbf{0.925}&\textbf{0.930}&\textbf{0.942}& \textbf{0.914}&\textbf{0.940}\\
\hline
    Out of successful:&     &&& & &\\
        \rowcolor[gray]{.9}\hspace{0.2em}\textbf{Macrocycles}&   0.003 &\textbf{0.995}&\textbf{0.943}&\textbf{0.911}& \textbf{0.864}&\textbf{0.852}\\
\hline
    Out of macrocycles:&   &&&  & &\\
        \hspace{0.2em}Diversity&  0.626&\textbf{0.712}&\textbf{0.709}& \textbf{0.713}& \textbf{0.711}&\textbf{0.723}\\
        \hspace{0.2em}Novelty&   1.000&\textbf{1.000}&\textbf{1.000}&\textbf{1.000}& \textbf{1.000}&\textbf{1.000}\\
        \hspace{0.2em}Uniqueness&   1.000&\textbf{1.000}&\textbf{1.000}&\textbf{1.000}& \textbf{1.000}&\textbf{1.000}\\
        \rowcolor{gray!20}\hspace{0.2em}\textbf{All PoseBusters tests}&    0.440&\textbf{0.575}&\textbf{0.590}&\textbf{0.569}& \textbf{0.535}&\textbf{0.503}\\
            \hspace{1.2em}Ligand PoseBusters&   0.539&\textbf{0.626}&\textbf{0.658}&\textbf{0.661}& \textbf{0.637}&\textbf{0.624}\\
                \hspace{2em}Bond lengths    &   0.844&\textbf{0.860}&\textbf{0.900}&\textbf{0.927}& \textbf{0.925}&\textbf{0.928}\\ 
                \hspace{2em}Bond angles    &   0.913&0.888&\textbf{0.918}&\textbf{0.916}& 0.890&0.850\\ 
                \hspace{2em}Internal steric clash    &   0.862&0.854&0.856&0.851& 0.839&\textbf{0.869}\\ 
                \hspace{2em}Aromatic ring flatness&   0.996 &0.992&0.995&0.994& \textbf{0.997}&0.990\\ 
                \hspace{2em}Non-ar. ring non-flatness &   0.993&\textbf{0.999}&\textbf{1.000}&\textbf{0.999}& \textbf{0.999}&\textbf{1.000}\\ 
                \hspace{2em}Double bond flatness&   0.980&\textbf{0.993}&\textbf{0.992}&\textbf{0.986}& \textbf{0.985}&\textbf{0.983}\\ 
                \hspace{2em}Internal energy    &   0.818&\textbf{0.921}&\textbf{0.928}&\textbf{0.920}& \textbf{0.916}&\textbf{0.920}\\ 
            \hspace{1.2em}Protein PoseBusters&    0.806&\textbf{0.911}&\textbf{0.895}&\textbf{0.875}& \textbf{0.838}&0.792\\
                \hspace{2em}Protein-ligand max. distance &    1.000&\textbf{1.000}&\textbf{1.000}&\textbf{1.000}& \textbf{1.000}&\textbf{1.000}\\
                \hspace{2em}Min. distance to protein&    0.806&\textbf{0.907}&\textbf{0.895}&\textbf{0.875}& \textbf{0.838}&0.792\\
                \hspace{2em}Volume overlap with protein&    1.000&\textbf{1.000}&\textbf{1.000}&\textbf{1.000}& \textbf{1.000}&0.999\\
        \hspace{0.2em}Pharmacophore satisfaction &    0.769&\textbf{0.789} &\textbf{0.791}&\textbf{0.800}& \textbf{0.820}&\textbf{0.823}\\
        \hspace{0.2em}Macrocycle Lipinski &    0.551&\textbf{0.638} &\textbf{0.692}&\textbf{0.698}& \textbf{0.682}&\textbf{0.693}\\
\bottomrule

\end{tabular}
\label{tab:results_protein_macro_every_bold}
\end{small}
\end{center}
\vskip -0.1in
\end{table}

\subsubsection{Starting Topological Guidance Late}
\label{app:starting_late}
Table~\ref{tab:late} analyzes the effect of delaying the onset of topological guidance. The guidance can be introduced halfway through the denoising process with minimal impact on molecular validity, macrocycle-related metrics, and PoseBusters performance.

\begin{table}[h!]
\begin{center}
\begin{small}
\caption{\textbf{Performance of protein-conditioned macrocycle generation with guidance starting late.} The number denotes the number of remaining denoising steps when the guidance is applied, out of 1000 steps total. The first two columns correspond to Table~\ref{tab:results_protein_macro}. Values outperforming or matching the baseline MolSnapper metrics are in bold. Non-baseline results were obtained from sets of 200 molecules.}
\setlength{\tabcolsep}{5pt}
\begin{tabular}{l@{}ccccccccccc}
\toprule
    Metrics ($\uparrow$; [0-1])  &   \makecell[c]{\makecell[c]{MolSnapper\\(no guid.)}}&\makecell[c]{\makecell[c]{1000\\steps}}&\makecell[c]{\makecell[c]{900\\steps}}& \makecell[c]{\makecell[c]{800\\steps}}& \makecell[c]{\makecell[c]{700\\steps}}&\makecell[c]{\makecell[c]{600\\steps}}& \makecell[c]{\makecell[c]{500\\steps}}&\makecell[c]{\makecell[c]{400\\steps}} &\makecell[c]{\makecell[c]{300\\steps}} &\makecell[c]{\makecell[c]{200\\steps}} &\makecell[c]{\makecell[c]{100\\steps}}\\ \hline
    Validity     &   0.858 &\textbf{0.925}&\textbf{0.925}&\textbf{0.925}& \textbf{0.910}&\textbf{0.920}& \textbf{0.940}& \textbf{0.905}& \textbf{0.910}& \textbf{0.860}&0.840\\ 
    Connectivity  &   1.000 &\textbf{1.000}&\textbf{1.000}&\textbf{1.000}& \textbf{1.000}&\textbf{1.000}& \textbf{1.000}& \textbf{1.000}& \textbf{1.000}& \textbf{1.000}&\textbf{1.000}\\
    Successfulness  &   0.858 &\textbf{0.925}&\textbf{0.925}&\textbf{0.925}&\textbf{0.910}&\textbf{0.920}& \textbf{0.940}& \textbf{0.905}& \textbf{0.910}&\textbf{ 0.860}&0.840\\
\hline
    Out of successful:&     &&& & & & & & & &\\
        \rowcolor[gray]{.9}\hspace{0.2em}\textbf{Macrocycles}&   0.003 &\textbf{0.995}&\textbf{0.995}&\textbf{0.989}& \textbf{0.984}&\textbf{0.995}& \textbf{0.995}& \textbf{1.000}& \textbf{0.995}& \textbf{0.965}&\textbf{0.935}\\
\hline
    Out of macrocycles:&   &&&  & & & & & & &\\
        \hspace{0.2em}Diversity&  0.626&\textbf{0.712}&\textbf{0.711}& \textbf{0.698}& \textbf{0.695}&\textbf{0.672}& \textbf{0.686}& \textbf{0.676}& \textbf{0.670}& \textbf{0.672}&\textbf{0.704}\\
        \hspace{0.2em}Novelty&   1.000&\textbf{1.000}&\textbf{1.000}&\textbf{1.000}& \textbf{1.000}&\textbf{1.000}& \textbf{1.000}& \textbf{1.000}& \textbf{1.000}& \textbf{1.000}&\textbf{1.000}\\
        \hspace{0.2em}Uniqueness&   1.000&\textbf{1.000}&\textbf{1.000}&\textbf{1.000}& \textbf{1.000}&\textbf{1.000}& \textbf{1.000}& \textbf{1.000}& \textbf{1.000}& \textbf{1.000}&\textbf{1.000}\\
        \rowcolor{gray!20}\hspace{0.2em}\textbf{All PoseBusters tests}&    0.440&\textbf{0.575}&\textbf{0.598}&\textbf{0.552}& \textbf{0.547}&\textbf{0.503}& \textbf{0.551}&\textbf{ 0.486}& 0.376& 0.319&0.178\\
            \hspace{1.2em}Ligand PoseBusters&   0.539&\textbf{0.626}&\textbf{0.658}&\textbf{0.623}& \textbf{0.654}&\textbf{0.601}& \textbf{0.684}& \textbf{0.652}& \textbf{0.547}& 0.446&0.306\\
                \hspace{2em}Bond lengths    &   0.844&\textbf{0.860}&\textbf{0.880}&\textbf{0.869}& \textbf{0.922}&\textbf{0.863}& \textbf{0.914}& \textbf{0.884}& \textbf{0.917}& \textbf{0.886}&0.834\\ 
                \hspace{2em}Bond angles    &   0.913&0.888&0.902&0.852& \textbf{0.927}&0.902& 0.888& \textbf{0.923}& 0.834& 0.789&0.561\\ 
                \hspace{2em}Internal steric clash    &   0.862&0.854&\textbf{0.864}&\textbf{0.896}& 0.855&0.814& \textbf{0.877}& \textbf{0.873}& \textbf{0.890}& 0.813&0.841\\ 
                \hspace{2em}Aromatic ring flatness&   0.996 &0.992&\textbf{1.000}&\textbf{1.000}& 0.994&\textbf{1.000}& 0.995& 0.994& 0.994& \textbf{1.000}&\textbf{1.000}\\ 
                \hspace{2em}Non-ar. ring non-flatness &   0.993&\textbf{0.999}&1.000&\textbf{1.000}& \textbf{1.000}&\textbf{1.000}& 0.995& \textbf{1.000}& \textbf{1.000}& \textbf{1.000}&\textbf{1.000}\\ 
                \hspace{2em}Double bond flatness&   0.980&\textbf{0.993}&\textbf{0.995}&\textbf{0.989}& \textbf{0.994}&\textbf{0.995}& \textbf{0.989}& \textbf{0.989}& 0.961& 0.910&0.930\\ 
                \hspace{2em}Internal energy    &   0.818&\textbf{0.921}&\textbf{0.929}&\textbf{0.951}& \textbf{0.894}&\textbf{0.891}& \textbf{0.925}& \textbf{0.923}& 0.796& \textbf{0.867}&\textbf{0.822}\\ 
            \hspace{1.2em}Protein PoseBusters&    0.806&\textbf{0.911}&\textbf{0.913}&\textbf{0.902}& \textbf{0.844}&\textbf{0.852}& 0.786& 0.790& 0.740& 0.765&0.561\\
                \hspace{2em}Protein-ligand max. distance &    1.000&\textbf{1.000}&\textbf{1.000}&\textbf{1.000}& \textbf{1.000}&\textbf{1.000}&\textbf{ 1.000}& \textbf{1.000}& \textbf{1.000}& \textbf{1.000}&\textbf{1.000}\\
                \hspace{2em}Min. distance to protein&    0.806&\textbf{0.907}&\textbf{0.913}&\textbf{0.902}& \textbf{0.844}&\textbf{0.852}& 0.786& 0.790& 0.740& 0.765&0.561\\
                \hspace{2em}Volume overlap with protein&    1.000&\textbf{1.000}&\textbf{1.000}&\textbf{1.000}& \textbf{1.000}&\textbf{1.000}& \textbf{1.000}& \textbf{1.000}& \textbf{1.000}& \textbf{1.000}&\textbf{1.000}\\
        \hspace{0.2em}Pharmacophore satisfaction &    0.769&\textbf{0.789} &\textbf{0.804}&\textbf{0.809}&\textbf{ 0.816}&\textbf{0.820}& \textbf{0.893}& \textbf{0.823}& \textbf{0.818}& \textbf{0.819}&\textbf{0.815}\\
        \hspace{0.2em}Macrocycle Lipinski &    0.551&\textbf{0.638} &\textbf{0.598}&\textbf{0.639}& \textbf{0.598}&\textbf{0.557}& \textbf{0.620}& \textbf{0.602}& \textbf{0.713}& \textbf{0.669}&\textbf{0.732}\\
\bottomrule

\end{tabular}
\label{tab:late}
\end{small}
\end{center}
\vskip -0.1in
\end{table}

\clearpage
\section{Integrating Bond-level Information: Challenges}\label{app:bond}
The following strategies have been tested in order to integrate bond-level information in our macrocycle guidance function.
\begin{enumerate}
\item Maximizing the $H_1$ component \textbf{death edge probability} in our method.
\item Computing a \textbf{representative cycle} of our $H_1$ component. However this task is both computationally expensive and ambiguous, as a representative cycle is not unique.
\item Pre-selecting the \textbf{ordered set of atoms} to form the cycle (before sampling), and maximizing edge probabilities while minimizing chord probabilities.
\item Pre-selecting the \textbf{unordered set of atoms} to form the cycle (before sampling), and applying $H_1$ guidance on this subset during sampling based on distances computed from probabilities.
\end{enumerate}

None of the attempts mentioned yield convincing results, which we explain with the following three reasons. First, \name already achieves a $99\%$ macrocycle rate without integrating bond information. Consequently one can hardly hope for significant improvements. Second, any minimization of a chord probability is challenged by the fact that minimizing the probability of a bond can be done both by pushing the atoms apart and by bringing them very close, without a clear way to choose which one. Third, any backpropagation of the bond-predictor model is much more computationally expensive than our current guidance, as well as relying on the quality and robustness of this model. As a consequence, we believe that trying further to integrate bond-level information is not a priority.

\end{document}